\documentclass{article}

\usepackage{iclr2023_conference,times}
\usepackage[T1]{fontenc}    % use 8-bit T1 fonts
\usepackage{graphicx}
\usepackage{subcaption}

\usepackage[colorlinks,citecolor=blue,urlcolor=blue,linkcolor=blue,linktocpage=true]{hyperref}

\usepackage{amsthm}
\usepackage{amsmath}
\usepackage{amssymb}
\usepackage{thmtools}

\usepackage{empheq}
\usepackage{tikz}
\usepackage{wrapfig}

\usepackage{enumitem}
\setlist[enumerate]{itemsep=0mm}
\setlist[itemize]{itemsep=0mm}

\newcommand{\yws}[1]{} % remove note
\newcommand{\rmm}[1]{}

\def\1{\bm{1}}

\DeclareMathAlphabet{\mathsfit}{\encodingdefault}{\sfdefault}{m}{sl}
\SetMathAlphabet{\mathsfit}{bold}{\encodingdefault}{\sfdefault}{bx}{n}

\def\gA{{\mathcal{A}}}
\def\gB{{\mathcal{B}}}
\def\gC{{\mathcal{C}}}
\def\gD{{\mathcal{D}}}

\def\gF{{\mathcal{F}}}
\def\gG{{\mathcal{G}}}

\def\gM{{\mathcal{M}}}
\def\gN{{\mathcal{N}}}

\def\gP{{\mathcal{P}}}

\def\E{\mathbb{E}}

\def\Var{\mathrm{Var}}

\def\tr{\mathrm{tr}}

\def\R{\mathbb{R}}

\def\cF{\mathcal{F}}

\def\MSE{{\mathrm{MSE}}}

\newcommand{\Z}{\mathbb{Z}}
\newcommand{\N}{\mathbb{N}}
\newcommand{\conv}{\textrm{conv}}

\DeclareMathOperator*{\argmin}{arg\,min}

\newtheorem{theorem}{Theorem}
\newtheorem{lemma}[theorem]{Lemma}
\newtheorem{proposition}[theorem]{Proposition}
\newtheorem{corollary}[theorem]{Corollary}
\newtheorem{claim}[theorem]{Claim}
\newtheorem{definition}{Definition}
\newtheorem{remark}{Remark}
\newcommand{\act}{\sigma}

\newcommand{\mat}[1]{\mathbf{#1}}
\newcommand{\vect}[1]{\boldsymbol{#1}}

\newenvironment{eqal}{
\begin{equation}
\begin{aligned}
}
{
\end{aligned}    
\end{equation}}

\newenvironment{eqal*}{
\begin{equation*}
\begin{aligned}
}
{
\end{aligned}    
\end{equation*}}

\def\equationautorefname~#1\null{%
  (#1)\null
}
\allowdisplaybreaks
\iclrfinalcopy

\newcommand{\modif}[1]{#1}

\title{Deep Learning meets Nonparametric Regression: Are Weight Decayed DNNs Locally Adaptive?}

\author{%
   Kaiqi Zhang \\
   Department of Electrical and Computer Engineering \\
   University of California, Santa Barbara\\
   \textit{kzhang70@ucsb.edu}
   \And
   Yu-Xiang Wang \\
   Department of Computer Science\\
   University of California, Santa Barbara\\
   \textit{yuxiangw@cs.ucsb.edu}
}

\begin{document}

\maketitle

\begin{abstract}
% A great amount of research has been trying to understand why neural networks can achieve superior performance on many machine learning tasks, such as understanding neural networks' adaptivity to the H\"{o}lders, Sobolev and Besov spaces. 
% However, many previous studies require tuning the architecture based on the function space and sample size to achieve the optimal error rate.
% In this work, we study a deep ReLU neural network's adaptivity achieved by a parallel neural network without tuning the architecture. 
% We show that with only weight decay, a parallel neural network is equivalent to an $\ell_p$-penalized regression problem where $0 < p < 1$.
% %additive model with $p$ norm constraint where $0 < p < 1$. 
% Using such equivalent model, we study the error rate when the target function is in the Besov or bounded variation (BV) space. 
% We prove that by tuning the weight decay, a deep parallel neural network can achieve close to minimax error rate. 
% Additionally, deep neural models achieve closer to the minimax error rate than shallow models when the sample size is large. 

We study the theory of neural network (NN) from the lens of classical nonparametric regression problems with a focus on NN's ability to \emph{adaptively} estimate functions with \emph{heterogeneous smoothness} --- a property of functions in Besov or Bounded Variation (BV) classes. 
Existing work on this problem requires tuning the NN architecture based on the function spaces and sample size. 
We consider a ``Parallel NN'' variant of deep ReLU networks and show that the standard $\ell_2$ regularization is equivalent to promoting the $\ell_p$-sparsity ($0<p<1$) in the coefficient vector of an end-to-end learned function bases, i.e., a dictionary.
Using this equivalence, we further establish that by tuning only the regularization factor, such parallel NN achieves an estimation error arbitrarily close to the minimax rates for both the Besov and BV classes. 
Notably, it gets exponentially closer to minimax optimal as the NN gets deeper. Our research sheds new lights on why depth matters and how NNs are more powerful than kernel methods.

\end{abstract}

\section{Introduction}
\label{sec:intro}
%Deep learning models \emph{work}.  From speech recognition~\citep{hinton2012deep} to image understanding~\citep{krizhevsky2012imagenet}, from natural languages processing~\citep{vaswani2017attention} to strategic decision making~\citep{silver2017mastering}, these models comprised of stacked blocks of complex computation with a large number of parameters have revolutionized artificial intelligence research and  resulted in applications that have reached almost every aspects of our daily life.

%What are mysterious: why NN is better than others
%But 

\emph{Why} do deep neural networks (DNNs) \emph{work} better? 
%This is a question many have attempted to answer. 
They are universal function approximators~\citep{cybenko1989approximation}, but so are splines and kernels.  They learn data-driven representations, but so are the shallower and linear counterparts such as matrix factorization. 
% There is surprisingly little theoretical understanding on why DNNs are superior to these classical alternatives. 
The theoretical understanding on why DNNs are superior to these classical alternatives is surprisingly limited.

In this paper, we study DNNs in nonparametric regression problems --- a classical branch of statistical theory and methods with more than half a century of associated literatures \citep{nadaraya1964estimating,de1978practical,wahba1990spline,donoho1998minimax,mallat1999wavelet,scholkopf2001learning,rasmussen2006gaussian}.   Nonparametric regression addresses the fundamental problem:
%\begin{itemize}
\begin{list}{$\bullet$}{\topsep=0.0ex \leftmargin=0.2in \rightmargin=0.in \itemsep =-0.0in}
	\item Let $y_i = f(x_i) + \text{Noise}$ for $i=1,...,n$. 
   How can we estimate a  function $f$ using data points $(x_1,y_1),...,(x_n,y_n)$  in conjunction with the knowledge that $f$  belongs to a function class $\gF$?
\end{list}
%\end{itemize}
%of estimating a smooth function using noisy observations of the function values without making assumptions about the form or shape of the function. 
Function class $\gF$ typically imposes only weak regularity assumptions such as 
% boundedness and 
smoothness, which makes nonparametric regression widely applicable to real-life applications under weak assumptions. % especially those with physical processes. 

\noindent\textbf{Local adaptivity.} 
We say a nonparametric regression technique is \emph{locally adaptive} if it can cater to local differences in smoothness, hence allowing more accurate estimation of functions with varying smoothness and abrupt changes. 
A subset of nonparametric regression techniques were shown to have the property of \emph{local adaptivity} \citep{mammen1997locally} in both theory and practice. These include wavelet smoothing \citep{donoho1998minimax}, locally adaptive regression splines \citep[LARS,][]{mammen1997locally}, trend filtering \citep{tibshirani2014adaptive,wang2014falling} and adaptive local polynomials \citep{baby2019online,baby2020higherTV}.  
In light of such a distinction, it is natural to consider the following question: %we focus on a subset of the question 
% \begin{center}
% 	\vspace{-.1cm}
% 	\textsf{
\emph{Are NNs \emph{locally adaptive}, i.e., optimal in learning functions with heterogeneous smoothness?}

% For example, a locally adaptive method will be able to estimate a function $f$ whose \emph{$m$-th order derivative} $f^{(m)}$ has bounded total variation 
% (i.e., when $\gF$ is a $m$-th order bounded variation class) 
% with an optimal mean square error (MSE) of $O(n^{-(2m+2)/(2m+3)})$, while \emph{linear estimators} such as kernel smoothing and smoothing splines have an MSE of $O(n^{-(2m+1)/(2m+2)})$.

\begin{wrapfigure}[12]{R}{0.45\textwidth}
   \centering
   \vspace{-5mm}
   \includegraphics[width=0.45\textwidth]{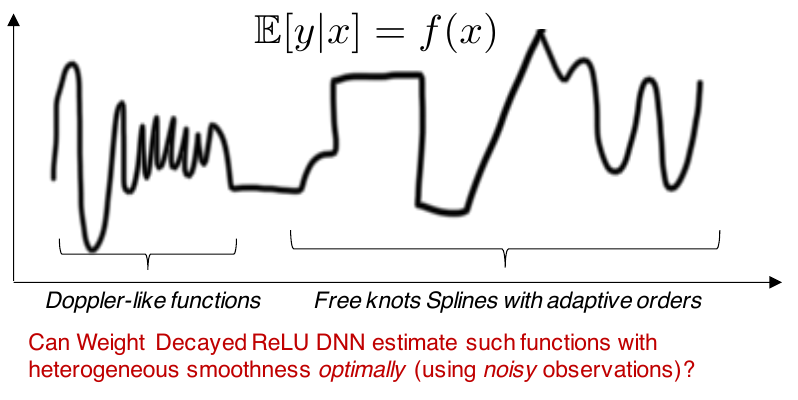}
   \vspace{-2em}
   \caption{Illustration of a function with heterogeneous smoothness and the problem of locally adaptive nonparametric regression. }\label{fig:illus}
\end{wrapfigure}

%for functions whose $k$th order derivative is in a total variation class, a locally adaptive method improves over classical techniques from $O(n^{-(2k+1)/(2k+2)})$ to the optimal $O(n^{-(2k+2)/(2k+3)})$.

%We start by asking the following question about one type of localized structure.

	%    Can NNs efficiently represent and estimate functions with spatially heterogeneous smoothness? }
% 	\vspace{-.1cm}
% \end{center}
%These functions include piecewise linear functions and Doppler-like functions that are very smooth in some regions, but wiggly in other regions. 

This is a timely question to ask, partly because the bulk of recent theory of NN leverages its asymptotic Reproducing Kernel Hilbert Space (RKHS) in the overparameterized regime \citep{jacot2018neural,belkin2018understand,arora2019exact}.
% \footnote{This asymptotic theory on applies to a specific initialization scheme and requires the number of hidden nodes $p \gg n$~\cite{arora2019exact,chizat2019lazy}.};
RKHS-based approaches, e.g., kernel ridge regression with any fixed kernels are \emph{suboptimal} in estimating functions with heterogeneous smoothness \citep{donoho1990minimax}. 
Therefore, existing deep learning theory based on RKHS does not satisfactorily explain the advantages of neural networks over kernel methods. 

We build upon the recent work of \citet{suzuki2018adaptivity} and \citet{parhi2021banach} who provided encouraging first answers to the question above.
%  about the local adaptivity of NNs. 
Specifically, \citet[Theorem 8]{parhi2021banach} showed that a two-layer \emph{truncated power function} activated neural network with a non-standard regularization is equivalent to the LARS.
This connection implies that such NNs achieve the minimax rate for the (high order) bounded variation (BV) classes. 
A detailed discussion is provided in \autoref{sec:warmup}.
\citet{suzuki2018adaptivity} showed that multilayer ReLU DNNs can achieve minimax rate for the Besov class, but requires the artificially imposed sparsity-level of the DNN weights to be calibrated according to parameters of the Besov class, thus is quite difficult to implement in practice.
% different from how DNNs are typically trained 
\modif{
\citep{oono2019approximation, liu2021besov} replaced the sparse neural network with Resnet-style CNN and achieved the same rate, but they similarly require carefully choosing the number of parameters for \emph{each} nonparametric class. We show that $\ell_2$ regularization suffices for \emph{mildly overparameterized} DNNs to achieve the optimal ``local adaptive'' rates for \emph{many} nonparametric classes at the same time. }

% In this paper, we aim at addressing the same \emph{locally adaptivity} question 
% by replacing the requirement on the data-dependent network architecture 
% with $\ell_2$ regularization.
% for a more commonly used neural network with standard $\ell_2$ regularized training.  

\textbf{Parallel neural networks.} We restrict our attention on a special network architecture called \emph{parallel neural network} \citep{haeffele2017global,ergen2021path} which learns an ensemble of subnetworks --- each being a multilayer ReLU DNNs. Parallel NNs have been shown to be more well-behaved both theoretically  \citep{haeffele2017global,zhang2019deep,ergen2021global, ergen2021path, ergen2021revealing} and empirically \citep{zagoruyko2016wide,veit2016residual}. On the other hand, many successful NN architectures such as SqueezeNet, ResNext and Inception (see \citep{ergen2021path} and the references therein) use the idea similar to a parallel NN. 

\modif{
\textbf{Weight decay,}
also known as square \textbf{$\ell_2$ regularization}, is one of the most popular regularization techniques for preventing overfitting in DNNs.
It is called ``weight decay'' because each iteration of the gradient descent (or SGD) shrinks the parameter towards 0 multiplicatively. 
%It is equivalent to penalizing the loss with the summation of squared Frobenius norm of the parameters. 
%Empirically, weight decay can help reduce overfitting. 
Many tricks in deep learning, including early stopping \citep{yao2007early}, quantization \citep{hubara2016binarized}, and dropout \citep{wager2013dropout} behaves like $\ell_2$ regularization. Thus even though we focus on the exact minimizer of the regularized objective, it may explain the behavior of SGD in practice. }
% In this paper, we make no assumption on the training method thus there is no (implicit) regularizers apart from $\ell_2$ regularization. 

%The same limitation of RKHS-based theory is discussed by the recent work of \citep{suzuki2018adaptivity} in which NNs were shown to achieve the optimal rates for 1D Besov class of functions, which separates NN from RKHS methods. Our proposed research in Task 1a further strengthens the results of \citep{suzuki2018adaptivity} in a number of dimensions.

%As neural networks have achieved great success in many machine learning tasks, relevant theoretical analysis has drawn much attention. 
\begin{table}[t]
   \vspace{-5mm}
   \caption{\modif{Comparison with the results in the literature}}
   \label{tab:comp}
   \resizebox{\columnwidth}{!}{
   \begin{tabular}{p{2.3cm}cp{2cm}p{2cm}p{2cm}p{4.2cm}}
      \hline
      & \# layers & Activation   & Function space & Minimax rate & Remark\\
      \hline
      \citet{parhi2021banach, parhi2021near} & 2 & truncated power & $\textrm{BV}^{m}$  & Yes & Non-standard activation and regularization (when $m>1$).\\
      \citet{schmidt2020nonparametric} & $\geq 3$ & ReLU & H\"older & Up to a log factor & With sparsity constraint.\\
      % \citet{parhi2021kinds} & $\geq 3$ & ReLU & compositions of Banach spaces & N/A & two linear layers each activation layer\\
      \citet{suzuki2018adaptivity} & $\geq 3$ & ReLU  & Besov \& m-Besov & Up to a log factor & With sparsity constraint.\\
      \textbf{Ours} & $\geq 3$ & ReLU & Besov \& BV & Up to $n^{o(1)}$ factor & Requires only $\ell_2$ regularization. \\
      \hline
   \end{tabular}
   }
   \vspace{-5mm}
\end{table}

\noindent\textbf{Summary of results.} Our main contributions are:
\begin{enumerate}
	\itemsep0em
	% \item We show that a two-layer truncated power activated neural network with $O(n)$ neurons and a higher-order weight decay regularization from \citet{parhi2021banach} is equivalent to the locally adaptive regression spline \citep{mammen1997locally}, thus achieves the minimax rate for the (higher order) bounded variation (BV) class. Here $n$ is the number of data points.  
	\item  
   %Going beyond 1D functions and two-layer NNs, 
   We prove that the (standard) $\ell_2$ regularization in training an $L$-layer \emph{parallel} ReLU-activated neural network is equivalent to a sparse $\ell_p$ penalty term (where $p = 2/L$) on the linear coefficients of a learned representation (\autoref{prop:eqmodel}).
	%applying $p$-norm regularizer to the coefficients in an additive model,
	%where $p = 2/L$ which induces sparsity in a particular transformed space.
   % \item We extend the generic statistical learning bound for the estimation error. 
   % We decompose the function space to a finite dimension unconstrained space, and an infinite dimension constrained space, and achieve an estimation error bound for this space.
   
   % We show that 
   %neural networks can approximate B-spline basis functions of any order without the need of choosing the order parameter manually. 
   % In other words, 
   % neural networks can adapt to functions of different order of smoothness, and even functions with different smoothness in different regions in their domain.
   \item We show that the estimation error of $\ell_2$ regularized parallel NN 
   % decreases polynomially with the number of samples up to a constant error for estimating functions 
   % % with heterogeneous smoothness 
   % in the both BV and Besov classes,
   % and the exponential term in the error rate
   can be close to the minimax rate for estimating functions in Besov space. 
   Notably, the method can adapt to different smoothness parameter, which is not the case for many other methods. 
	%The key novel technical component is a new approximation-theoretic result showing that one can approximate any Besov class functions with a linear combination of cardinal B-spline wavelets having an $\ell_p$-norm bounded coefficient vector.
	%We study the error rate of training a ReLU neural network with weight decay, and prove that one can achieve close to minimax rate with no need of tuning model architecture.
	\item We find that deeper models achieve closer to the optimal error rate. This result helps explain why deep neural networks can achieve better performance than shallow ones empirically. 
\end{enumerate}

\modif{
Besides, we have the following technical contributions which could be of separate interest:
\begin{itemize}
   \item We provide a way to bound the complexity of an overparameterized neural network. Specifically, we bound the metric entropy of a parallel neural network in \autoref{thm:pnncv}, and the bound does not depend on the number of subnetworks.
   \item We propose a method to handle unconstrained function subspace when bounding the estimation error as in Equation \autoref{eq:errdecom}.
\end{itemize}
}
The above results separate parallel NNs with any linear methods such as kernel ridge regression.
To the best of our knowledge, we are the first to demonstrate that standard techniques ($\ell_2$ regularization and ReLU activation) suffice for DNNs in achieving the optimal rates for estimating BV and Besov functions.
\modif{The comparison with previous works is shown in \autoref{tab:comp}.
More discussion about related works are shown in \autoref{sec:related}. }

\section{Preliminary}
\subsection{Notation and Problem Setup.}
%In the following discussion, 
We denote regular font letters as scalars, bold lower case letters as vectors and bold upper case letters as matrices. 
$a \lesssim b$ means $a \leq C b$ for some constant $C$ that does not depend on $a$ or $b$, and $a \eqsim b$ denotes $a \lesssim b$ and $ b\lesssim a$. See \autoref{tab:symb} for the full list of symbols used.

\begin{table}[b]
   \vspace{-3mm}
   \centering
\caption{Symbols used in this paper}
\label{tab:symb}
\begin{tabular}{cp{5.1cm}|cp{4.7cm}}
   \hline
   symbol & Meaning\\
   \hline
   $a/\vect a/ \mat A$ & scalars / vectors / matrices.
   & $[a, b]$ & $\{x \in \R: a \leq x \leq b\}$\\
   %$a, b,  \dots$ & scalars.\\
   %$\vect a, \vect b, \dots$ & vectors.\\
   %$\mat A, \mat B, \dots$ & matrices.\\
   $B_{p,q}^\alpha$ & Besov space.
   & $[n]$ & $\{x \in \N: 1 \leq x \leq n\}$.\\
   $|\cdot|_{B_{p,q}^\alpha}$ & Besov quasi-norm .
   & $\|\cdot\|_F$ & Frobenius norm.\\
   $\|\cdot\|_{B_{p,q}^\alpha}$ & Besov norm.
   & $\|\cdot\|_p$ & $\ell_p$-norm.\\
   $M_m(\cdot)$ & $m^{th}$ order  Cardinal B-spline bases.
   & $d$ & Dimension of input. \\
   $M_{m,k,\vect s}(\cdot)$ & $m^{th}$ order Cardinal B-spline basis 
   & $M$ & \# subnetworks in a parallel NN.\\
     & function of resolution $k$ at 
   & 
   $L$ & \# layers in a (parallel) NN.\\   
     & position $\vect s$.  
   & $w$ & Width of a subnetwork.\\
   $\act(\cdot)$ & ReLU activation function.
   & $n$ & \# samples. \\
   $\mat W^{(\ell)}_j, \vect b^{(\ell)}_j$ & Weight and bias in the $\ell$-th layer in the $j$-th subnetwork.
   & $\R, \Z, \N$ & Set of real numbers, integers, and nonnegative integers.\\
   \hline 
\end{tabular}
\vspace{-4mm}
\end{table}

% We use $f_0$ to denote the target function to be estimated. 
Let $f_0$ be the target function to be estimated. 
The training dataset is $\gD_n := \{(\vect x_i, y_i), y_i = f_0(\vect x_i) + \epsilon_i, i \in [n]\} $, 
where $x_i$ are fixed and $\epsilon_i$ are zero-mean, independent Gaussian noises with variance $\sigma^2$. In the following discussion, we assume $\vect x_i\in [0, 1]^d, f_0(x_i) \in [-1, 1], \forall i$.

We will be comparing estimators under the mean square error (MSE), defined as
$
\MSE(\hat{f}):= \E_{\gD_n}\frac{1}{n} \sum_{i=1}^n (\hat f(\vect x_i)-f_0(\vect x_i))^2.
$
The optimal worst-case MSE is described by 
$R(\mathcal F) := \min_{\hat{f}}\max_{f_0\in\gF} \MSE(\hat{f})$.
We say that $\hat{f}$ is optimal if $\MSE(\hat{f}) \lesssim R(\mathcal F)$.
The empirical (square error) loss is defined as 
%\begin{eqal*}
 $  \hat L(\hat{f}):= \frac{1}{n} \sum_{i=1}^n (\hat f(\vect x_i)- y_i)^2.$
%\end{eqal*}
The corresponding population loss is $L(\hat{f}) := \E[\frac{1}{n} \sum_{i=1}^n (\hat f(\vect x_i)- y'_i)^2 | \hat{f}]$ where $y'_i$ are new data points. It is clear that $\E[L(\hat{f})] = \MSE[\hat{f}] + \sigma^2$.

% In the following discussion, we assume $\vect x_i\in [0, 1]^d, f_0(x_i) \in [-1, 1], \forall i$.
% To evaluate the performance of an estimator, we compare the mean square error (MSE) against the ground truth:
% \begin{eqal*}
%   L(f) := \E_{\gD_n}\frac{1}{n} \sum_{i=1}^n (\hat f(\vect x_i)-f_0(\vect x_i))^2
% \end{eqal*}

%We use $M_m(\cdot)$ to denote $m$-th order Cardinal B-spline basis functions, and define $B_{m, k, \vect s}(\vect x) := B_m(2^k(\vect x - \vect s))$ as the multi-resolution B-spline basis functions.
% The meanings of symbols used are listed in \autoref{tab:symb}.

\subsection{Besov Spaces and Bound Variation Space}
\textbf{Besov space}, denoted as $B^\alpha_{p, q}$, is a flexible function class parameterized by $\alpha,p,q$ whose definition is deferred to \autoref{sec:besov}.
%plays an important role in nonparametric regression. 
Here $\alpha \geq 0$ determines the smoothness of functions, $1 \leq p \leq \infty$ determines the averaging (quasi-)norm over locations, $1 \leq q \leq \infty$ determines the averaging (quasi-)norm over scale which plays a relatively minor role. Smaller $p$ is more forgiving to inhomogeneity and loosely speaking, when the function domain is bounded, smaller $p$ induces a larger function space. 
On the other hand, it is easy to see from definition that 
$
   B^\alpha_{p, q} \subset B^\alpha_{p, q'}, \textrm{ if } q < q'.
$
Without loss of generalizability, in the following discussion we will only focus on $B^{\alpha}_{p, \infty}$.
%  We will not cover this case in this paper. 
When $p=1$, the Besov space allows higher inhomogeneity, and it is more general than the Sobolev or H\"older space.

\textbf{Bounded variation (BV) space} is a more interpretable class of functions with spatially heterogeneous smoothness \citep{donoho1998minimax}. It is defined through the total variation (TV) of a function. For $(m+1)$th differentiable function $f:[0,1]\rightarrow \R$, the $m$th order total variation is defined as
$
 TV^{(m)}(f) := TV(f^{(m+1)}) = \int_{[0,1]} |f^{(m+1)}(x)| dx,
$
and the corresponding $m$th order Bounded Variation class 
$
      BV(m) := \{f: TV(f^{(m)}) < \infty\}.
$
The more general definition is given in \autoref{sec:tv}.
Bounded variation class is tightly connected to Besov classes. Specifically \citep{devore1993constructive}:
\begin{equation}  
   \label{eq:bvbesov}
   B^{m+1}_{1, 1}\subset BV(m) \subset B^{m+1}_{1, \infty}
\end{equation}
This allows the results derived for the Besov space to be easily applied to BV space.

\textbf{Minimax MSE} It is well known that minimax rate for Besov and 1D BV classes are $O(n^{-\frac{2\alpha}{2\alpha+d}})$ and $O(n^{-(2m+2)/(2m+3)})$ respectively . The minimax rate for \emph{linear estimators} in 1D BV classes is known to be $O(n^{-(2m+1)/(2m+2)})$ \citep{mammen1997locally,donoho1998minimax}.

% See \autoref{sec:besov} for the detail

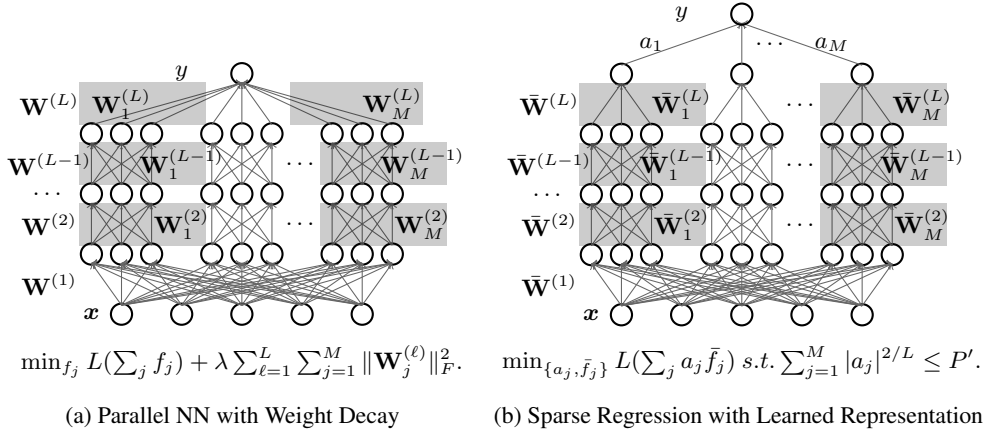
\begin{figure*}
   \centering
   \subcaptionbox{Parallel NN with $\ell_2$ regularization\label{fig:pnn}}{
   \begin{tikzpicture}[
      scale=0.75,
      roundnode/.style={circle, draw=black, thick, minimum size=8, inner sep=0},
      arrow/.style={->, draw=black!60},
      font=\small
   ]
      \foreach \x in {0, ..., 4}{
         \node [roundnode] at (\x,0) {};
      }

      \foreach \y in {1, ..., 3}{  
         \foreach \x in {0, ..., 2}{
            \node [roundnode] at (0.5*\x-0.5,\y) {};
         }
         \foreach \x in {0, ..., 2}{
            \node [roundnode] at (0.5*\x+1.5,\y) {};
         }
         \foreach \x in {0, ..., 2}{
            \node [roundnode] at (0.5*\x+3.5,\y) {};
         }
      }
      \foreach \xa in {0, ..., 4} {
         \foreach \xb in {0, ..., 2} {
            \draw [arrow] (\xa, 0.15) -- (0.5*\xb-0.5, 0.85);
            \draw [arrow] (\xa, 0.15) -- (0.5*\xb+1.5, 0.85);
            \draw [arrow] (\xa, 0.15) -- (0.5*\xb+3.5, 0.85);
         }      
      }
      \foreach \y in {2, ..., 3}{  
         \foreach \xa in {0, ..., 2}{
            \foreach \xb in {0, ..., 2}{
               \draw [arrow] (0.5*\xa-0.5, \y-0.85) -- (0.5*\xb-0.5, \y-0.15);
               \draw [arrow] (0.5*\xa+1.5, \y-0.85) -- (0.5*\xb+1.5, \y-0.15);
               \draw [arrow] (0.5*\xa+3.5, \y-0.85) -- (0.5*\xb+3.5, \y-0.15);
            }  
         }
      }
      \node [roundnode] at (2, 4) {};
      \foreach \x in {0, ..., 2} {
         \draw [arrow] (0.5*\x-0.5, 3.15) -- (2, 3.85);
         \draw [arrow] (0.5*\x+1.5, 3.15) -- (2, 3.85);
         \draw [arrow] (0.5*\x+3.5, 3.15) -- (2, 3.85);
      }
      \node at (-0.5, 0) {$\vect x$};
      \node at (-1.2, 0.5) {$\mat W^{(1)}$};
      \node at (-1.2, 1.5) {$\mat W^{(2)}$};
      \node at (-1.2, 2) {$\dots$};
      \node at (-1.2, 2.5) {$\mat W^{(L-1)}$};
      \node at (-1.2, 3.5) {$\mat W^{(L)}$};

      \draw [fill=black, draw=none, opacity=0.2] (-0.7, 1.15) rectangle (1.4, 1.85);
      \node at (1, 1.5) {$\mat W^{(2)}_1$};
      \draw [fill=black, draw=none, opacity=0.2] (3.3, 1.15) rectangle (5.4, 1.85);
      \node at (5, 1.5) {$\mat W^{(2)}_M$};
      \node at (3, 1.5) {$\dots$};
      \draw [fill=black, draw=none, opacity=0.2] (-0.7, 2.15) rectangle (1.4, 2.85);
      \node at (1, 2.5) {$\mat W^{(L-1)}_1$};
      \draw [fill=black, draw=none, opacity=0.2] (3.3, 2.15) rectangle (5.4, 2.85);
      \node at (5, 2.5) {$\mat W^{(L-1)}_M$};
      \node at (3, 2.5) {$\dots$};
      \draw [fill=black, draw=none, opacity=0.2] (-0.7, 3.15) rectangle (1.4, 3.85);
      \node at (0, 3.5) {$\mat W^{(L)}_1$};
      \draw [fill=black, draw=none, opacity=0.2] (2.8, 3.15) rectangle (5, 3.85);
      \node at (4.5, 3.5) {$\mat W^{(L)}_M$};
      \node at (1, 4) {$y$};
      \node at (2, -0.8) {$\min_{f_j} L(\sum_j f_j) + \lambda \sum_{\ell=1}^{L}\sum_{j=1}^M \|\mat W^{(\ell)}_j\|_F^2.$};
      %------------- block diagonal matrix --------
      \begin{scope}[shift={(-1.5,4.3)}, scale=0.8]
      \foreach \x in {0, ..., 3} {
         \draw [draw=black!40!white] (\x, 0) -- (\x, 3);
         \draw [draw=black!40!white] (0, \x) -- (3, \x);
      }
      \node at (-0.8, 1.5) {$\mat W^{(\ell)}$:};
      \filldraw[draw=none, fill=black, opacity=0.2](0, 2) rectangle (1,3);
      \node at (0.5, 2.5) {$\mat W^{(\ell)}_1$};
      \filldraw[draw=none, fill=black, opacity=0.2](1, 1) rectangle (2,2);
      \node at (1.5, 1.5) {$\mat W^{(\ell)}_2$};
      \filldraw[draw=none, fill=black, opacity=0.2](2, 0) rectangle (3,1);
      \node at (2.5, 0.5) {$\mat W^{(\ell)}_M$};
      \node at (1.5, 2.5) {0};
      \node at (2.5, 2.5) {0};
      \node at (2.5, 1.5) {0};
      \node at (0.5, 0.5) {0};
      \node at (0.5, 1.5) {0};
      \node at (1.5, 0.5) {0};
      \node [anchor=west] at (3, 1) {(c) Block diagonal weights};
      \end{scope}
      \draw [dashed] (-2, 1) rectangle (-0.6, 3);
      \draw [arrow] (-2, 2) to [out=120, in=-120] (-2, 5);
   \end{tikzpicture}
   }
   \subcaptionbox{Sparse Regression with Learned Representation\label{fig:eqnn}}[0.5\textwidth]{
      \begin{tikzpicture}[
        scale=0.75,
         roundnode/.style={circle, draw=black, thick, minimum size=8, inner sep=0},
         arrow/.style={->, draw=black!60},
         font=\small
      ]
         \foreach \x in {0, ..., 4}{
            \node [roundnode] at (\x,0) {};
         }
   
         \foreach \y in {1, ..., 3}{  
            \foreach \x in {0, ..., 2}{
               \node [roundnode] at (0.5*\x-0.5,\y) {};
            }
            \foreach \x in {0, ..., 2}{
               \node [roundnode] at (0.5*\x+1.5,\y) {};
            }
            \foreach \x in {0, ..., 2}{
               \node [roundnode] at (0.5*\x+3.5,\y) {};
            }
         }
         \foreach \xa in {0, ..., 4} {
            \foreach \xb in {0, ..., 2} {
               \draw [arrow] (\xa, 0.15) -- (0.5*\xb-0.5, 0.85);
               \draw [arrow] (\xa, 0.15) -- (0.5*\xb+1.5, 0.85);
               \draw [arrow] (\xa, 0.15) -- (0.5*\xb+3.5, 0.85);
            }      
         }
         \foreach \y in {2, ..., 3}{  
            \foreach \xa in {0, ..., 2}{
               \foreach \xb in {0, ..., 2}{
                  \draw [arrow] (0.5*\xa-0.5, \y-0.85) -- (0.5*\xb-0.5, \y-0.15);
                  \draw [arrow] (0.5*\xa+1.5, \y-0.85) -- (0.5*\xb+1.5, \y-0.15);
                  \draw [arrow] (0.5*\xa+3.5, \y-0.85) -- (0.5*\xb+3.5, \y-0.15);
               }  
            }
         }
         \foreach \x in {0, 2, 4}{
            \node [roundnode] at (\x,4) {};
            \draw [arrow] (\x, 4.15) -- (2, 4.85);
         }
         \node [roundnode] at (2,5) {};
   
         \foreach \x in {0, ..., 2} {
            \draw [arrow] (0.5*\x-0.5, 3.15) -- (0, 3.85);
            \draw [arrow] (0.5*\x+1.5, 3.15) -- (2, 3.85);
            \draw [arrow] (0.5*\x+3.5, 3.15) -- (4, 3.85);
         }
         \node at (-0.5, 0) {$\vect x$};
         \node at (-1.2, 0.5) {$\bar {\mat W}^{(1)}$};
         \node at (-1.2, 1.5) {$\bar {\mat W}^{(2)}$};
         \node at (-1.2, 2) {$\dots$};
         \node at (-1.2, 2.5) {$\bar {\mat W}^{(L-1)}$};
         \node at (-1.2, 3.5) {$\bar {\mat W}^{(L)}$};
   
         \draw [fill=black, draw=none, opacity=0.2] (-0.7, 1.15) rectangle (1.4, 1.85);
         \node at (1, 1.5) {$\bar {\mat W}^{(2)}_1$};
         \draw [fill=black, draw=none, opacity=0.2] (3.3, 1.15) rectangle (5.4, 1.85);
         \node at (5, 1.5) {$\bar {\mat W}^{(2)}_M$};
         \node at (3, 1.5) {$\dots$};
         \draw [fill=black, draw=none, opacity=0.2] (-0.7, 2.15) rectangle (1.4, 2.85);
         \node at (1, 2.5) {$\bar {\mat W}^{(L-1)}_1$};
         \draw [fill=black, draw=none, opacity=0.2] (3.3, 2.15) rectangle (5.4, 2.85);
         \node at (5, 2.5) {$\bar {\mat W}^{(L-1)}_M$};
         \node at (3, 2.5) {$\dots$};
         \draw [fill=black, draw=none, opacity=0.2] (-0.7, 3.15) rectangle (1.4, 3.85);
         \node at (1, 3.5) {$\bar {\mat W}^{(L)}_1$};
         \draw [fill=black, draw=none, opacity=0.2] (3.3, 3.15) rectangle (5.4, 3.85);
         \node at (5, 3.5) {$\bar {\mat W}^{(L)}_M$};
         \node at (3, 3.5) {$\dots$};
         \node at (2.5, 4.5) {$\dots$};
         \node at (0.5, 4.5) {$a_1$};
         \node at (3.5, 4.5) {$a_M$};
         \node at (1, 5) {$y$};
         \node at (2, -0.8) {$\min_{\{a_j, \bar{f}_j\} } L(\sum_j a_j \bar{f}_j ) + \lambda' \sum_{j=1}^M |a_j|^{2/L}.$};
      \end{tikzpicture}
      }
   \caption{Parallel neural network and the equivalent sparse regression model we discovered. \label{fig:pnnall}}
   \vspace{-0.3cm}
\end{figure*}

\section{Main Results: Parallel ReLU DNNs}
Consider a parallel neural network containing $M$ multi layer perceptrons (MLP) with ReLU activation functions called \emph{subnetworks}.
Each subnetwork has width $w$ and depth $L$. 
The input is fed to all the subnetworks, and the output of the parallel NN is the summation of the output of each subnetwork. The architecture of a parallel neural network is shown in \autoref{fig:pnn}. This parallel neural network is equivalent to a vanilla neural network with block diagonal weights in all but the first and the last layers (\autoref{fig:pnnall}(c)).  
%A parallel neural network can be expressed as a fully connected neural network with additional constraint. In a neural network with width $W=Mw$ and depth $L$, 
%Furthermore, in a parallel neural network with $K$ subnetworks, 
Let $\mat W^{(\ell)}_j$ and $\vect b^{(\ell)}_j$ denote the weight and bias in the $\ell$-th layer in the $j$-th subnetwork respectively. 
Training this model with $\ell_2$ regularization returns: 
%the optimization problem:
% A parallel neural network can be expressed as a fully connected neural network by letting 
% \begin{eqal*}
%    \mat W^{(1)} &= \big[{\mat W^{(1)}_1}^T, \dots, {\mat W^{(1)}_K}^T\big]^T,\\
%    \mat W^{(\ell)} &= \diag(\mat W^{(\ell)}_1, \mat W^{(\ell)}_2, \dots \mat W^{(\ell)}_{M}),\ \ell = 2,\dots, L-1, \\
%    \mat W^{(L)} &= [\mat W^{(L)}_1, \dots, \mat W^{(L)}_K]
% \end{eqal*}
% where $\mat W^{(\ell)}$ denote the weight matrix in the $\ell$-th layer, and $\vect b^{(\ell)}$ to denote the bias in it. 
% In this paper, we are targeting a slightly modified neural network. Specifically, we focus on a $L$-(linear-)layer ReLU neural network with width $W$ in all the hidden (activation) layers. Furthermore, apart from the first and the last (linear) layers, the rest layers can be decomposed into $W/w$ number of subnetworks with depth $L$ and width $w$, where $w$ is a constant depending on the dimension of input as well as function class. In other words, the weight matrix in all the layers except the first and last layer are block diagonal with blocksize $w$:
% \begin{equation*}
%     \mat W^{(\ell)} = \diag(W^{(\ell)}_1, W^{(\ell)}_2, \dots W^{(\ell)}_{W/w}), \ell = 2,\dots, L-1
% \end{equation*}
% As for the training method, we apply standard weight decay to this model, which is equivalent to applying squared $\ell_2$ regularization to it:
% \begin{eqal*}
%    \argmin_{\{\mat W^{(\ell)}, \vect b^{(\ell)}\}} &\frac{1}{n}\sum_{i=1}^n \ell(f(\vect x_i), y_i) + \lambda \sum_{\ell=1}^L \|\mat W^{(\ell)}\|_F^2,
% \end{eqal*}
% or equivalently
\begin{eqal}
   % \argmin_{\{\mat W^{(\ell)}_j, \vect b^{(\ell)}_j\}} &\frac{1}{n}\sum_{i=1}^n \ell\Big(\sum_{j=1}^M f_j(\vect x_i), y_i\Big) \\
   % &+ \lambda \sum_{j=1}^M\sum_{\ell=1}^L \|\mat W^{(\ell)}_j\|_F^2,
   \argmin_{\{\mat W^{(\ell)}_j, \vect b^{(\ell)}_j\}} \hat L( f)
   + \lambda \sum_{j=1}^M\sum_{\ell=1}^L \big\|\mat W^{(\ell)}_j\big\|_F^2,
   \label{eq:l2}
\end{eqal}  
where $f(x) = \sum_{j=1}^M f_k(x)$ denotes the parallel neural network,
%$\ell(\cdot, \cdot)$ denotes the loss, 
%$f(\cdot)$ denotes the neural network with parameters $\{\mat W^{(\ell)}, \vect b^{(\ell)}\}$, 
$f_j(\cdot)$ denotes the $j$-th subnetwork,
%  with parameters $\{\mat W^{(\ell)}_j, \vect b^{(\ell)}_j\}$,
and $\lambda > 0$ is a fixed scaling factor. 
We choose not to regularize the bias terms $\vect b^{(\ell)}_j$ to provide a cleaner equivalent model (\autoref{prop:eqmodel}).
If the bias terms are regularized, the result will be similar.
\modif{Besides, we ignore the computation issue and focus on the global optimal solution to this problem. 
In practice, in deep neural network, the solution obtained using gradient descent-style methods are often close to the global optimal solution \citep{choromanska2015loss}.}
%In the latter discussion, 
% We focus the on the mean square error (MSE) loss under a discrete measure defined by the traing set:
% $$
%    % \ell(x, y) =  (x-y)^2.
%    L(\gV f) := \frac{1}{n}\sum_{i=1}^n (f(\vect x_i)-y_i)^2
% $$

\newcommand\thmmain[1][]{
   For any fixed $\alpha - d/p > 1, q \geq 1, L \geq 3$,
   define $m = \lceil \alpha - 1 \rceil$.
   For any $f_0 \in B^{\alpha}_{p, q}$,
   given an $L$-layer parallel neural network satisfying
   \begin{itemize}
   % \itemsep{0em}
      % \item The depth of this neural network is $L $.
      \item The width of each subnetwork is \modif{ \textbf{fixed} satisfying $ w \geq O(md)$.} See \autoref{thm:pnormapp} for the detail.
      \item The number of subnetworks is \modif{\textbf{large enough}: $M \gtrsim n^\frac{1-2/L}{2\alpha/d + 1 - 2/(pL)}$. }
      % $M \geq \bar M$, where $\bar M$ is defined in \autoref{sec:proofthmmain}.
   \end{itemize}
   Under the assumption as in \autoref{lemma:cons2regu},
   with proper choice of the parameter of regularizaton $\lambda$ that depends on  $\gD, \alpha, d, L$, the solution $\hat{f}$ parameterized by \autoref{eq:l2} satisfies
   %the empirical risk minimizer in  \autoref{eq:l2r} satisfy
   % \begin{eqal}
   %    &\MSE(\hat{f})= \tilde O \big(n^{-{\frac{2\alpha/d(1-2/L)}{2\alpha/d+1-2/(pL)}}} \big) + Const.
   %    \ifx\\#1\\
   %       \nonumber
   %    \else
   %       \label{eq:errorrate}
   %    \fi
   % \end{eqal}
   \ifx\\#1\\
      \begin{eqal*}
         &\MSE(\hat{f})= \tilde O \Bigg(\Big(\frac{w^{4-4/L}L^{2-4/L}}{n^{1-2/L}}\Big)^\frac{2\alpha/d}{2\alpha/d+1-2/(pL)}  + e^{-c_6 L}\Bigg)
      \end{eqal*}
      where $\tilde O$ shows the scale up to a logarithmic factor, and $c_6$ is the constant defined in \autoref{thm:pnormapp}. 
   \else
      \begin{eqal}
         \MSE(\hat{f})= C(w, L)\tilde O \big(n^{-{\frac{2\alpha/d(1-2/L)}{2\alpha/d+1-2/(pL)}}} \big) + e^{-c_6 L}.
      \label{eq:errorrate}
      \end{eqal}
      where $\tilde O$ shows the scale up to a logarithmic factor, 
      $c_6>0$ is a numerical constant from Theorem~\ref{thm:pnormapp}, 
      $C(w, L) \eqsim (w^{4-4/L}L^{2-4/L})^\frac{2\alpha/d}{2\alpha/d+1-2/(pL)}$ depends polynomially on $L$.
   \fi
   %$\epsilon$ is a constant that doesn't depend on $n$, and decrease exponentially with $L$.
}

\begin{theorem}
   \label{thm:main}
   \thmmain[label]
\end{theorem}
We explain the proof idea in the next section,but defer the extended form of the theorem and the full proof to \autoref{sec:proofthmmain}.  
Before that, we comment on a few interesting aspects of the result. 

\noindent\textbf{Near optimal rates and the effect of depth.} The first term in the MSE bound is the estimation error and the second term is (part of) the approximation error of this NN.  Recall that the minimax rate of a Besov class is $O(n^{-\frac{2\alpha}{2\alpha+d}})$.
\modif{The gap between the estimation error and the minimax rate is because the minimax rate can be achieved by an $\ell_0$ sparse model, while the parallel NN is equivalent to an $\ell_{p}$ sparse model (will be shown in \autoref{prop:eqmodel}), which is an approximation to $\ell_0$.}
As the depth parameter $L$ increases, $p=2/L$ gets closer to $0$, the MSE can get arbitrarily close to the minimax rate and the trailing constant term in \autoref{eq:errorrate} can be arbitrarily small. Close to the optimal rate can be achieved if we choose $L \gtrsim \log n $:
\begin{corollary}
   Under the conditions of Theorem~\ref{thm:main}, for any $f_0\in B^\alpha_{p,q}$, there is a numerical constant $C$ such that when we choose $C\log n \leq L  \leq 100 C\log n $,
   $$ 
   \mathrm{MSE}(\hat f) = \tilde{O}(n^{-\frac{2\alpha}{2\alpha +d}(1-o(1))} ), 
   $$
   where $\tilde{O}$ hides only logarithmic factors and the $o(1)$ factor in the exponent is $O(1/\log(n))$.
\end{corollary}

% This result says that 
% % with only weight decay, 
% deeper parallel neural networks achieves lower error and gets closer to the statistical limit.

\noindent\textbf{Sparsity and comparison with standard NN.} We also note that the result does not depend on $M$ as long as $M$ is large enough.
% $M > \bar{M}$ where $\bar{M}$ increases sublinearly with $n$.
% . Exact expression of $\bar{M}$ is not shown but it suffices to choose $M \geq  n$. 
This means that the neural network can be arbitrarily overparameterized while not overfitting. The underlying reason is \emph{sparsity}. As it will become clearer in \autoref{sec:eqmodel}, $\ell_2$ regularized training of a parallel $L$-layer ReLU NNs is equivalent to a sparse regression problem with an $\ell_p$ penalty assigned to the coefficient vector of a learned dictionary. Here $p = 2/L$ which promotes even sparser solutions than an $\ell_1$ penalty.
\modif{Such $\ell_p$ sparsity does not exist in standard deep neural networks to the best of our knowledge, which indicates that parallel neural networks may be superior over standard neural networks in local adaptivity.}

\modif{
\noindent\textbf{Adaptivity to function spaces.} 
For any fixed $L, \tilde m$, our result shows the parallel neural network with width $w = O(\tilde m d)$ can achieve close to the minimax rate for any Besov class as long as $\alpha \leq \tilde m$.
In other words, neural networks can adapt to smoothness parameter by tuning only the regularizaton parameter. As will be shown in \autoref{thm:pnncv}, overestimating $\alpha$ with $\tilde{m}$ only changes the logarithmic terms in the MSE bound --- a mild price to pay for a more adaptive method. 
%As will be shown in \autoref{thm:pnncv}, the estimation error depends only logarithmically on $w$ so slightly overestimating $\alpha$ leads to only mild increase in MSE. 
}
% For any fixed $L$, the required architecture of the model does not depend on the dataset or the target function ($n, \alpha$) expect the number of subnetworks $M$, for which the only requirement is being large enough. As a result, one can design a model using a large guess on $M$, and achieve the claimed near-optimal error rate by only tuning the weight decay parameter.

\modif{
\noindent\textbf{Hyperparameter tuning.}
% Although it is not practical to compute $\lambda$ analytically, 
We provide an explicit choice of $\lambda$ in \autoref{lemma:cons2regu} underlying our theoretical result.  %rough guide to choose $\lambda$ in \autoref{lemma:cons2regu}.
Empirically, it can be determined empirically, e.g. using cross validation.
}

\modif{
\noindent\textbf{Fixed design v.s. random design.}
We mainly focus on bounding the error at sample covariates (the \emph{fixed design} problem) to be comparable to classical nonparameteric regression results. For completeness, we also state results for the \emph{random design} version of the problem (bounding $\E_\gD\E_f \mathrm{MSE}(\hat f)$) in \autoref{thm:rand}, to be compatible with the standard statistical learning setting \citep[e.g.,][]{suzuki2018adaptivity}. 
}

More discussion about this result can be found in \autoref{sec:adddis}.

\noindent\textbf{Bounded variation classes.} Thanks to the Besov space embedding of the BV class \autoref{eq:bvbesov}, our theorem also implies the result for the BV class in $1D$.
\begin{corollary}
   If the target function is in bounded variation class 
   $
      f_0 \in BV(m),
   $
   For any fixed $L \geq 3$, for a neural network satisfying the requirements in \autoref{thm:main} with $d=1$ and 
   with proper choice of the regularization factor $\lambda$, the NN $\hat{f}$ parameterized by \autoref{eq:l2r} satisfies
   \begin{eqal*}
    %   &\E_{\gD_n}[\|\hat f - f_0\|^2_{L^2(P_x)}]   \\
     & \MSE(\hat f)
      = C(w, L)\tilde O(n^{-{\frac{(2m+2)(1-2/L)}{2m+3-2/L}}} 
      )
      + O(e^{-c_6 L}),
   \end{eqal*}
   where $C(w, L)$ is the same as in \autoref{eq:errorrate} except replacing $\alpha$ with $m$.
   % $c_6$ is the same constant.
   % $\epsilon$ is a constant that doesn't depend on $n$, and decrease exponentially with $L$.
\end{corollary}
It is known that any linear estimators such as kernel smoothing and smoothing splines cannot have an error lower than $O(n^{-(2m+1)/(2m+2)})$ for $BV(m)$ \citep{donoho1998minimax}. 
\modif{When $L > O(m^2)$, the first term in the MSE of NN decreases with $n$ faster than that of the linear methods.
When $n$ is large enough, there exists $L$ such that the MSE of NN is strictly smaller than that of any linear method.}
This partly explains the advantage of DNNs over kernels.

\section{Proof Overview}
% \label{sec:esperr}
We start by first proving that a parallel neural network trained with $\ell_2$ regularization is equivalent to an $\ell_p$-sparse regression problem with representation learning (\autoref{sec:eqmodel}); which helps decompose its MSE into an estimation error and approxmation error. Then we bound the two terms under an $\ell_p$-sparse constrained problem setting in \autoref{sec:esterr} and \autoref{sec:apperr} respectively. 

Notably, we adapted the generic statistical learning machinery (a self-bounding argument) for studying this constrained ERM problem  \citep[Proposition 4]{suzuki2018adaptivity}  to bound the estimation error.
This adaption is non-trival because there is an \emph{unconstrained} subspace with no bounded metric entropy. 
%We extended a result from \citet[Proposition 4]{suzuki2018adaptivity} to a \emph{fixed design} regression problem with a finite dimensional uncontrained subspace. 
Specifically,  \autoref{prop:msecov} shows that the MSE of the regression problem can be bounded by
% \begin{eqal*}
% \MSE(\hat{f}) = &O\bigg(\underbrace{\inf_{f\in \cF} \MSE(f)}_{\text{approximation error} }\\
% & + \underbrace{\frac{\log \gN(\gF_\parallel,\delta,\|\cdot\|_\infty) + d(\gF_\bot)}{n}+ \delta}_\text{estimation error}\bigg)
% \end{eqal*}
% \begin{align}
%    &\MSE(\hat{f}) = O\bigg({\inf_{f\in \cF} \MSE(f)} \label{eq:apperr}\\
%    &\quad\quad + \frac{\log \gN(\gF_\parallel,\delta,\|\cdot\|_\infty) + \dim(\gF_\bot)}{n}+ \delta, \label{eq:esterr}\bigg)
% \end{align}
\begin{eqal}
   \label{eq:errdecom}
   \MSE(\hat{f}) = &O\bigg(\underbrace{\inf_{f\in \cF} \MSE(f)}_{\text{approximation error} }
    + \underbrace{\frac{\log \gN(\gF_\parallel,\delta,\|\cdot\|_\infty) + d(\gF_\bot)}{n}+ \delta}_\text{estimation error}\bigg)
\end{eqal}

in which $\gF$ decomposes into $\gF_\parallel \times \gF_\bot$, where $\gF_\bot$ is an unconstrained subspace with finite dimension, and $\gF_\parallel$ is a compact set in the orthogonal complement with a $\delta$-covering number of  $\gN(\gF_\parallel,\delta,\|\cdot\|_\infty)$ in $\|\cdot\|_\infty$-norm. 
This decomposes MSE into an approximation error and an estimation error. The  novel analysis of these two represents the major technical contribution of this paper.

%In \autoref{sec:esterr},  $\ell_p$-sparse regression model, and analyse the estimation error by bounding covering number (metric entropy).
%In \autoref{sec:apperr}, we bound the approximation error via decomposition using B-spline basis functions. 
%Finally, by choosing the $\ell_p$-sparsity parameter properly, which is empirically implemented by tuning weight decay, and using self bounding trick, we finish the analysis of MSE.

\subsection{Equivalence to  $\ell_p$ Sparse Regression }
%with a Learned Feature Representation}
\label{sec:eqmodel}
% We decompose the first and the last layer in this neural network to match the block diagonal structure of the rest layers: 
% let $\mat W^{(1)} = [{\mat W^{(1)}}_1^T,\dots, {\mat W^{(1)}}_{W/w}^T]^T$, 
% $\mat W^{(L)} = [{\mat W^{(1)}}_1^T,\dots, {\mat W^{(1)}}_{W/w}^T]^T$. 
% In this way, we decompose this neural network into $W/w$ subnetworks. Each subnetwork has width $w$ and depth $L$, and the output of the large neural network equals the sum of the subnetworks. Furthermore, t
% In a parallel neural network, the constraint in \eqref{eq:l2c} can be decomposed into each subnetwork:
% \begin{equation}
%    \sum_{\ell=1}^L \|\mat W^{(\ell)}\|_F^2 = \sum_{k=1}^{W/w}\sum_{\ell=1}^L \|\mat W^{(\ell)}_k\|_F^2 \leq M.
%    \label{eq:l2eq}
% \end{equation}

% Using Langrange's method, one can easily find \autoref{eq:l2} is equivalent to a constrained optimization problem:
% \begin{eqal}
%    % \argmin_{\{\mat W^{(\ell)}_j, \vect b^{(\ell)}_j\}} &\frac{1}{n}\sum_{i=1}^n \ell\Big(\sum_{j=1}^M f_j(\vect x_i), y_i\Big), \\
%    \argmin_{\{\mat W^{(\ell)}_j, \vect b^{(\ell)}_j\}} &\hat L\Big(\sum_{j=1}^M f_j\Big), &
%    s.t. &\sum_{j=1}^M\sum_{\ell=1}^L \big\|\mat W^{(\ell)}_j\big\|_F^2 \leq P
%    \label{eq:l2c}
% \end{eqal}
% for some constant $P$ that depend on $\lambda$ and the dataset $\gD$.%$\{\vect x_i, y_i\}$.

% On the other hand
It is widely known that ReLU function is 1-homogeneous: 
$
   \sigma(ax) = a\sigma(x), \forall a \geq 0, x \in \R.
$
In any consecutive two layers in a neural network (or a subnetwork), one can multiply the weight and bias in one layer with a positive constant, and divide the weight in another layer with the same constant. The neural network after such transformation is equivalent to the original one:
\begin{eqal}
   &\mat W^{(2)}\sigma(\mat W^{(1)}\vect x+\vect b^{(1)} 
   = \frac{1}{c}\mat W^{(2)}\sigma(c\mat W^{(1)}\vect x+c\vect b^{(1)}),
   \quad \forall c>0, \vect x.
    \label{eq:eqv}
\end{eqal}
%Because of that, it suffices to consider the set of neural networks that satisfy \eqref{eq:suffice}. Under this constraint, denote $\tilde W_i^{(\ell)} = \frac{\tilde W_i^{(\ell)}}{\|\tilde W_i^{(\ell)}\|_F}$ as the normalized weight, and $a_i = \prod_{\ell=1}^L \|\tilde W_i^{(\ell)}\|_F = \|\tilde W_i^{(1)}\|_F^L$, the constrained problem in \eqref{eq:l2c} can be reformulated as 
\modif{This property can be applied to \emph{each subnetwork} (instead of the entire model in a standard NN),} and we can reformulate \autoref{eq:l2} to an $\ell_p$ sparsity-regularized problem:
\newcommand\propeqmodel[1][]{
   % Fix the input dataset $\gD_n$ and a constant $c_1 > 0$.
   {There exists an one-to-one mapping between $\lambda > 0$ and $\lambda' > 0$}
%   and constants $c_\ell>0$ for $\ell\in[L]$. 
%   For any $\lambda > 0$, there exists $P'>0$ 
   such that  \autoref{eq:l2} is equivalent to the following problem:
   \begin{eqal}
      % \argmin_{\{\mat {\bar W}^{(\ell)}_j, \vect {\bar b}^{(\ell)}_j, a_j\}} &\frac{1}{n}\sum_{i=1}^n \ell\Big(\sum_{j=1}^{M} a_j \bar f_j(\vect x_i), y_i\Big),\\
      & \argmin_{\{\mat {\bar W}^{(\ell)}_j, \vect {\bar b}^{(\ell)}_j, a_j\}} 
      \hat L\Big(\sum_{j=1}^M a_j \bar f_j \Big) 
      + \lambda' \|\{a_j\}\|_{2/L}^{2/L}\\
      % & = \frac{1}{n}\sum_i (y_i - \bar{f}_{1:M}(\vect x_i)^T\vect a)^2 \\
      %----------------------
      s.t.\ & \|\mat {\bar W}_j^{(1)}\|_F \leq c_1 \sqrt{d}, \forall j \in [M];
      \ \|\mat {\bar W}_j^{(\ell)}\|_F \leq c_1  \sqrt{w}, \forall j \in [M],  2 \leq \ell \leq L,
   %    %-----------------------
   %    & \|\{a_j\}\|_{2/L}^{2/L} 
   %  %   = \sum_{j=1}^{M} a_j^{2/L} 
   %    \leq P'
      \ifx\\#1\\
         \nonumber
      \else
         \label{eq:l2r}
      \fi
   \end{eqal}
   \noindent where $\bar f_j(\cdot)$ is a subnetwork with parameters $\mat {\bar W}_j^{(\ell)}, \vect {\bar b}_j^{(\ell)}$.
%   $c_\ell$ are  positive constants that can be arbitrarily chosen,
   %and $P'$ is a constant that depends on $\lambda$ in \autoref{eq:l2}.
   %, $\gD$, $M$,  and $\{c_\ell\}$. 
}
\begin{proposition}
   \label{prop:eqmodel}
   \modif{\propeqmodel[label]}
\end{proposition}
This equivalent model is demonstrated in \autoref{fig:eqnn}. The proof, which we defer to \autoref{sec:proofpropeqmodel}, uses AM-GM inequality and the observation that the optimal solution will have norm-equalized weights per layer. 
% In the following discussion, we choose $c_1 \eqsim \sqrt{d}, c_\ell \eqsim \sqrt{w} $ for $\ell > 1$, 
The constraint $ \|\mat {\bar W}_j^{(1)}\|_F \lesssim  \sqrt{d},  \|\mat {\bar W}_j^{(\ell)}\|_F \lesssim  \sqrt{w}, \forall \ell > 1$
is typical in deep learning for better numerical stability.
The equivalent model in \autoref{prop:eqmodel} is also a parallel neural network, but it appends one layer with parameters $\{a_k\}$ at the end of the neural network, and the constraint on the Frobenius norm is converted to the $2/L$ norm on the factors $\{a_k\}$. 
Since $L \gg 2$ in a typical application, $2/L \ll 1$ and this regularizer can enforce a sparser model than that in \autoref{sec:warmup}.
The same technique can also be used to prove that an $\ell_2$ constrained neural network is equivalent to the $\ell_{2/L}$ constrained model as in \autoref{eq:l2c}.  

There are two useful implications of Proposition~\ref{prop:eqmodel}. First, it gives an intuitive explanation on how a regularized Parallel NN works. Specifically, it can be viewed as a sparse linear regression with representation learning. 
Secondly, the conversion into the constrained form allows us to decompose the MSE into two terms as in \autoref{eq:errdecom} and bound them separately.
%$\bar{f}_j$ is described by the $j$th subnetwork, and the full network returns a sparse linear combination of these end-to-end learned basis functions. 

We emphasize that Proposition~\ref{prop:eqmodel} by itself is not new. The same result was previously obtained by \citet[Appendix C]{savarese2019infinite} (see Section~\ref{sec:related} for more details) and the key proof techniques date back to at least \citet{burer2003nonlinear}. Our novel contribution is to leverage this folklore equivalence for proving new learning bounds.

%returns a linear combination of the $M$ end-to-end learned 

\subsection{Estimation Error Analysis}
\label{sec:esterr} 
% The decomposition in \autoref{eq:errdecom} reveals that to bound the estimation error, it suffices to compute the covering number of the constraint set in the sup-norm of the function it represents.

%\subsubsection{Covering Number and Estimation Error}
%\label{sec:esterrcov}

% The estimation error of neural networks have been studied in many literatures by bounding its covering number or entropy \citep{yarotsky2017error, suzuki2018adaptivity}.
% One can apply these methods and get an upper bound of the covering number of the model in \autoref{eq:l2} directly, but such a bound depends on the width of the network or equivalently the number of subnetworks $M$ explicitly, which cannot be used directly in analying an infinite wide neural network. 
Previous results that bound the covering number of neural networks \citep{yarotsky2017error, suzuki2018adaptivity} depends on the width of the neural networks explicitly, which cannot be applied when analysing a potentially infinitely wide neural network. 
In this section, we leverage the $\ell_p$-norm bounded coefficients to avoid the dependence in $M$ in the covering number bound, 
\modif{and focus on a constrained optimization problem:
\begin{eqal}
   \argmin_{\{\mat {\bar W}^{(\ell)}_j, \vect {\bar b}^{(\ell)}_j, a_j\}} 
      \hat L\Big(\sum_{j=1}^M a_j \bar f_j \Big), \quad 
   %----------------------
   s.t.\|\{a_j\}\|_{2/L}^{2/L} \leq P',
   \label{eq:l2c}
\end{eqal}
and $\{\mat {\bar W}^{(\ell)}_j, \vect {\bar b}^{(\ell)}_j\}$ satisfy the same constraint as in \autoref{eq:l2r}.
The connection between the regularized problem and the constrained problem is defered to \autoref{lemma:cons2regu}. }

%sparse model \autoref{eq:l2c} to get rid of the dependency in a parallel neural network.
% First, we bound the covering number of a vanilla neural network, or a subnetwork in a parallel neural network.

% \begin{remark}
%    The condition $\|\mat W^{(1)}\|_F \leq \sqrt{d}, \|\mat W^{(\ell)}\|_F \leq \sqrt{w}$ are typical in deep learning as this preserves the norm of activations among layers thus increase numerical stability.
% \end{remark}

% The proof can be found in \autoref{sec:prooflemmapnorm}.   
% Using \autoref{lemma:nncn1} and \autoref{lemma:pnorm}, and taking $p$ in \autoref{lemma:pnorm} as $2/L$, one can easily get the covering number of parallel neural network:
\newcommand{\thmpnncv}[1][]{
  The covering number of the model defined in \autoref{eq:l2c} apart from the bias in the last layer satisfies %is bounded by 
   \begin{eqal}
      % \log \gN(\gF, \delta) \lesssim P'^{\frac{1}{1-2/L}}\delta^{-\frac{2/L}{1-2/L}} \log (P'/\delta).
      \log \gN(\gF, \delta) \lesssim w^{2+2/(1-2/L)} L^2 \sqrt{d} P'^{\frac{1}{1-2/L}}\delta^{-\frac{2/L}{1-2/L}} \log (wP'/\delta).
      \ifx\\#1\\
         \nonumber
      \else
         \label{eq:pnc}
      \fi
   \end{eqal}
%   where $c_2$ is a constant that depends only on $d, w$ and $L$.
}
\begin{theorem}
   \label{thm:pnncv}
   \thmpnncv[label]
\end{theorem}

This theorem
% does not depend on the number of subnetworks $M$. In other words, it 
provides a bound of estimation error for an arbitrarily wide parallel neural network as long as the total Frobenius norm is bounded.
The proof can be found in \autoref{sec:proofthmpnncv}. It requires the following lemma, whose proof is deferred to \autoref{sec:prooflemmapnorm}:
\def\lemmapnorm{
   Let $\gG\subseteq\{\R^d \rightarrow [-c_3, c_3]\}$ be a set with covering number satisfying 
   % $\gN(\gG, \delta) \lesssim \delta^{-k} \log(1/\delta)$ 
   $\log\gN(\gG, \delta) \lesssim k \log(1/\delta)$ 
   for some finite $c_3$, and for any $g\in \gG, |a|\leq 1$, we have $ag \in \gG$. The covering number of 
   $
      \gF = \left\{\sum_{i=1}^{M} a_i g_i \middle| g_i \in \gG, \|a\|_p^p \leq P, 0<p<1\right\}
   $
   for any $P > 0$ satisfies
   \begin{equation*}
      \log \gN(\gF, \epsilon) \lesssim kP^{\frac{1}{1-p}}(\delta/c_3)^{-\frac{p}{1-p}} \log (c_3P/\delta)
   \end{equation*}
   up to a double logarithmic factor.
   % where $c_3$ is a constant that depends only on $p$.
}
\begin{lemma}
    \label{lemma:pnorm}
   \lemmapnorm
\end{lemma}

%  of \autoref{lemma:pnorm}.

\subsection{Approximation Error Analysis}
\label{sec:apperr}
%The approximation ability of ReLU neural network has been widely studied \citep{liang2016deep, herrmann2022constructive}. 
The approximation error analysis involves two steps. 
% In \autoref{sec:tpf}, 
We first analyse how a subnetwork can approximate a B-spline basis, which is defered to \autoref{sec:prooflemmatpapp}.
Then 
% in \autoref{sec:apperrmul} 
we show that a sparse linear combination of B-spline bases approximates Besov functions.  
Both add up to the total error in approximating Besov functions with a parallel neural network (\autoref{thm:pnormapp}).

\def\propbapp{
   Let $\alpha - d/p > 1, r > 0$.  For any function in Besov space $f_0 \in B^{\alpha}_{p, q}$ and any positive integer $\bar M$, there is an $\bar M$-sparse approximation using B-spline basis of order $m$ satisfying $0 < \alpha < \min(m, m - 1 + 1/p)$:
   $
      \check f_{\bar M} = \sum_{i=1}^{\bar M} a_{k_i,\vect s_i}M_{m, k_i,\vect s_i}
   $
   for any positive integer $\bar M$ such that the approximation error is bounded as 
   $
      \|\check f_{\bar M} - f_0\|_r \lesssim {\bar M}^{-\alpha /d}\|f_0\|_{B^{\alpha}_{p, q}},
   $
   and the coefficients satisfy 
   $$
      \|\{2^{k_i} a_{k_i,\vect s_i}\}_{k_i, \vect s_i}\|_p \lesssim \|f_0\|_{B^{\alpha}_{p, q}}.
   $$
   }
\begin{proposition}
   \propbapp
   \vspace{-0.5cm}
   \label{prop:bapp}
\end{proposition}
The proof as well as the remark can be found in \autoref{sec:proofpropbapp}.

\def\thmpnormapp{
   Under the same condition as \autoref{prop:bapp},  
   for any positive integer $\bar M$,
   any function in Besov space $f_0 \in B^{\alpha}_{p, q}$ can be approximated by a parallel neural network with no more than \modif{$O(\bar M)$} number of subnetworks
   % $
   %    f(x) = \sum_{i=1}^{\bar M} a_j f_j(x)
   % $
   % where $f_j$ is a subnetwork with parameters $\{\mat {\bar W}_j^{(\ell)}, \vect {\bar b}_j^{(\ell)} \}$ 
   satisfying:
   \begin{enumerate}
      \item Each subnetwork has width \modif{$w = O(md)$} and depth $L$.
      %  where $w_{d, m}$ depends only on $d$ and $m$. 
      \item The weights in each layer satisfy $\|\mat {\bar W}_k^{(\ell)}\|_F \leq O(\sqrt{w})$ except the first layer $\|\mat {\bar W}_k^{(1)}\|_F \leq O(\sqrt{d}) $,
      \item The scaling factors have bounded $2/L$-norm:
      $
         \modif{P' :}= \|\{a_j\}\|_{2/L}^{2/L} \lesssim  {\bar M}^{1-2/(pL)}.
      $
      \item The approximation error is bounded by 
      \begin{equation*}
         \|\tilde f - f_0\|_r \leq (c_{4} {\bar M}^{-\alpha/d} + c_{5} e^{-c_{6} L}) \|f\|_{B^\alpha_{p, q}}
      \end{equation*}
      where $c_{4}, c_{5}, c_{6}$ are constants that depend only on $m, d$ and $p$.
   \end{enumerate}
}

\begin{theorem}
   \label{thm:pnormapp}
   \thmpnormapp
\end{theorem}
Here ${\bar M}$ is the number of ``active'' subnetworks, which is not to be confused with the number of subnetworks at initialization. The proof can be found in \autoref{sec:proofthmpnormapp}.

% \begin{remark}
%   When training such a neural network, one do not explicitly specity $\alpha$ or choose $m$. Instead, one can use a large guess on $m$ to determine the model architecture, specifically $m$ and $L$, and the minimum approximation error of this model is no larger than any specific choice of $m$,
%   \begin{eqal*}
%       &\min_{f \in \gF} \|f - f_0\|_r \leq \min_{f \in \gF^{(m)}} \|f - f_0\|_r \\
%       &\quad\leq (c_{4,m} {\bar M}^{-\alpha/d} + c_{5, m} e^{-c_{6,m} L}) \|f\|_{B^\alpha_{p, q}}, \forall m.
%   \end{eqal*}
%   where $\gF$ denotes all the functions that satisfy the constraint in \autoref{sec:eqmodel}, and $\gF^{(m)}$ denotes the functions when each subnetwork approximates a B-spline basis of order $m$. We add subscript $m$ to $c_4, c_5, c_6$ to denote that they depend on $m$. In other word, a neural network can adapt to different degree of smoothness. One can also choose different $m$ for different parts in the domain of $\tilde f$, showing a higher degree of ``local adaptivity''.
% \end{remark}

Using the estimation error in \autoref{thm:pnncv} and approximation error in \autoref{thm:pnormapp}, by choosing $\bar M$ to minimax the total error, we can conclude the sample complexity of parallel neural networks using $\ell_2$ regularization, which is the main result (\autoref{thm:main}) of this paper. See \autoref{sec:proofthmmain} for the detail.

\section{Experiment}
\newcommand{\figdist}{-3mm}
\begin{figure}[t!]
   \vspace{-2mm}
   % \centering
   % \subfigure[]{
   %    \includegraphics[width=0.45\textwidth]{hill_curve.pdf}
   % }
   % \subfigure[]{
   %    \includegraphics[width=0.45\textwidth]{dopler_curve.pdf}
      
   % }\\\vspace{-0.8cm}
   % \centering
   % \subfigure[]{
   %    \includegraphics[width=0.45\textwidth]{hill_err_newlook.pdf}
   % }
   % \subfigure[]{
   %    \includegraphics[width=0.45\textwidth]{dolper_err_newlook.pdf}
   % }\\
   % \centering
   \subcaptionbox{Dopler, DoF=30.}{\includegraphics[width=0.3\textwidth]{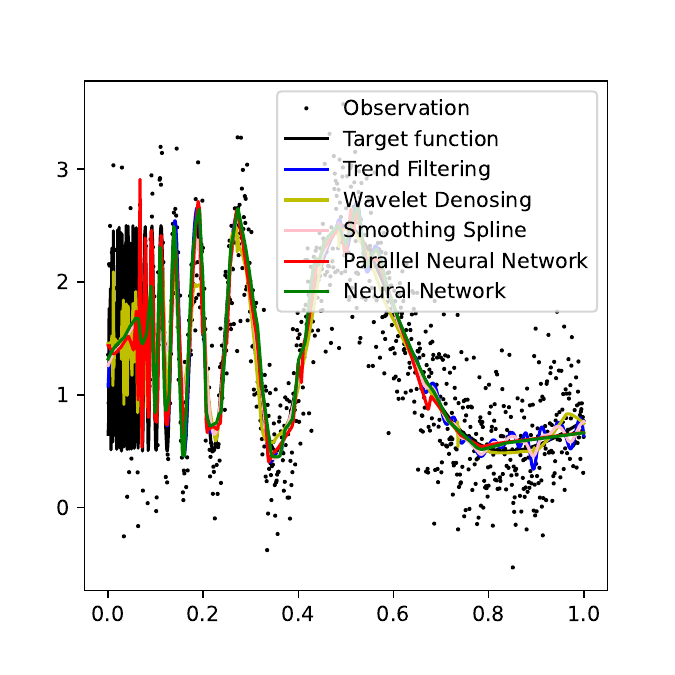}\vspace{\figdist}}
   \subcaptionbox{MSE versus DoF.}{\includegraphics[width=0.3\textwidth]{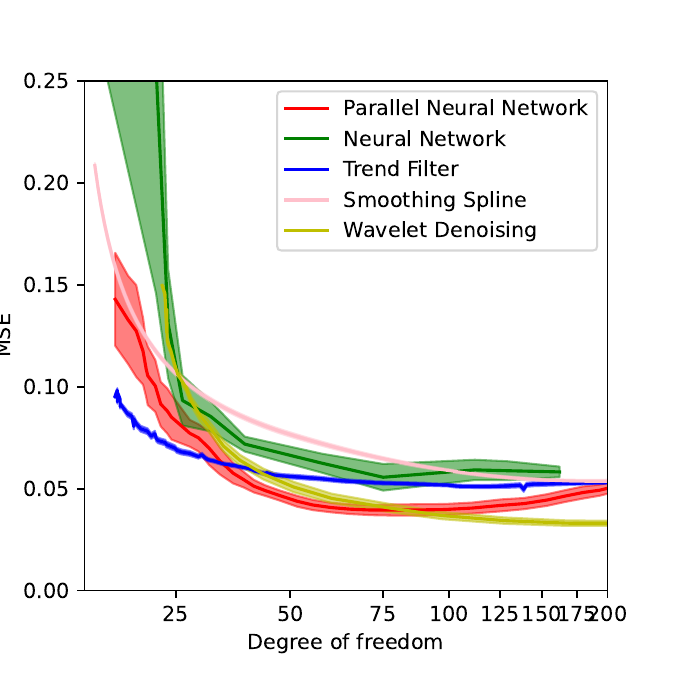}\vspace{\figdist}}
   \subcaptionbox{Active subnetworks.}{\includegraphics[width=0.3\textwidth]{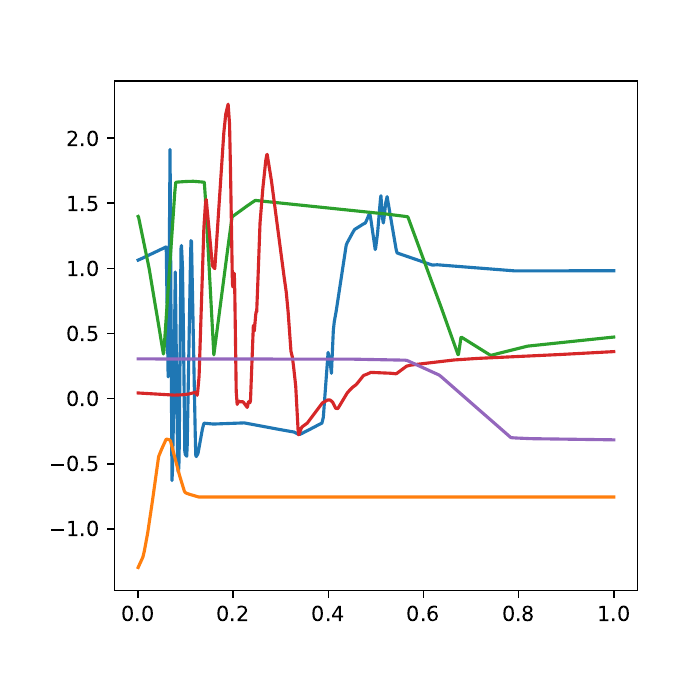}\vspace{\figdist}}\\
   \vspace{0mm}\\
    %----------
   \subcaptionbox{``Vary'', DoF=50.}{\includegraphics[width=0.3\textwidth]{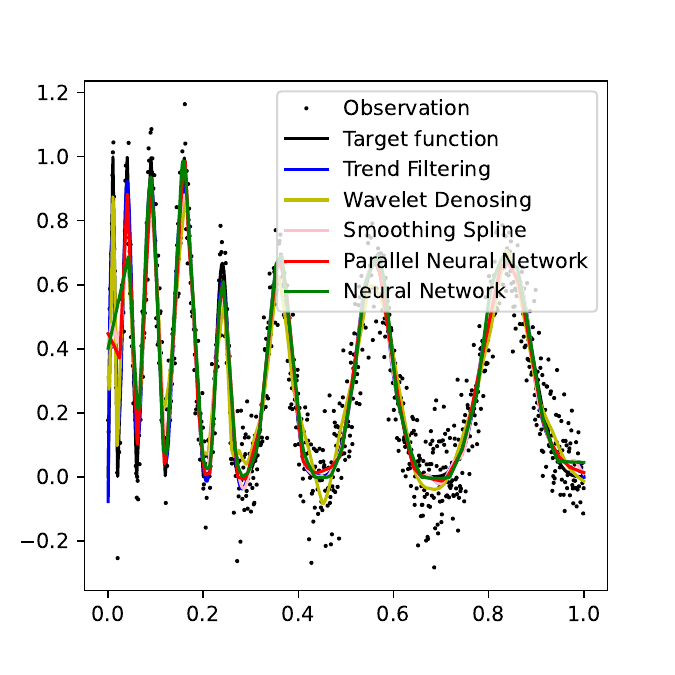}\vspace{\figdist}}
   \subcaptionbox{MSE versus DoF}{\includegraphics[width=0.3\textwidth]{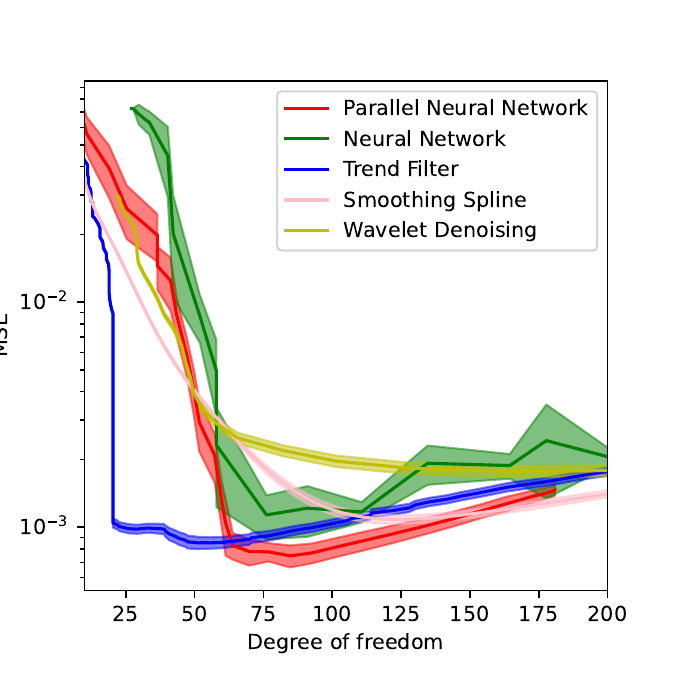}\vspace{\figdist}}
   \subcaptionbox{Active subnetworks.}{\includegraphics[width=0.3\textwidth]{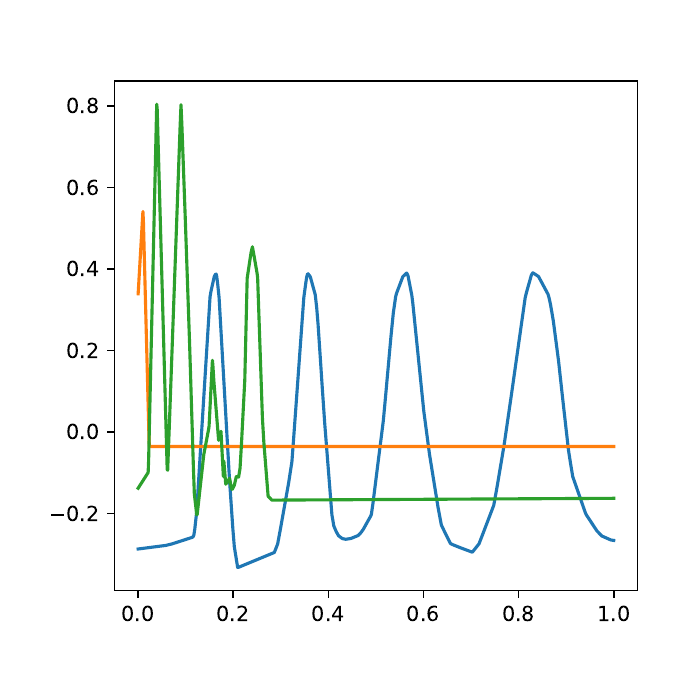}\vspace{\figdist}}\\
   \vspace{0mm}\\
   \subcaptionbox{Zoom in to (a)(d)}{\includegraphics[width=0.14\textwidth]{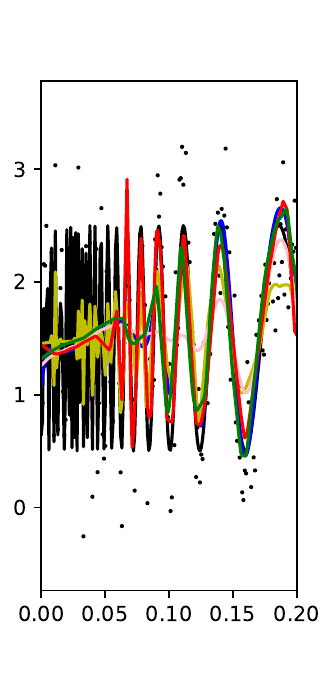}
   \hspace{-1mm}
   \includegraphics[width=0.14\textwidth]{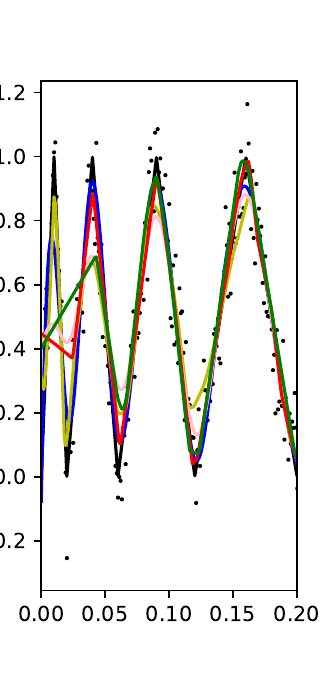}\vspace{\figdist}}
   \subcaptionbox{MSE versis $n$, Dopler}{\includegraphics[width=0.3\textwidth]{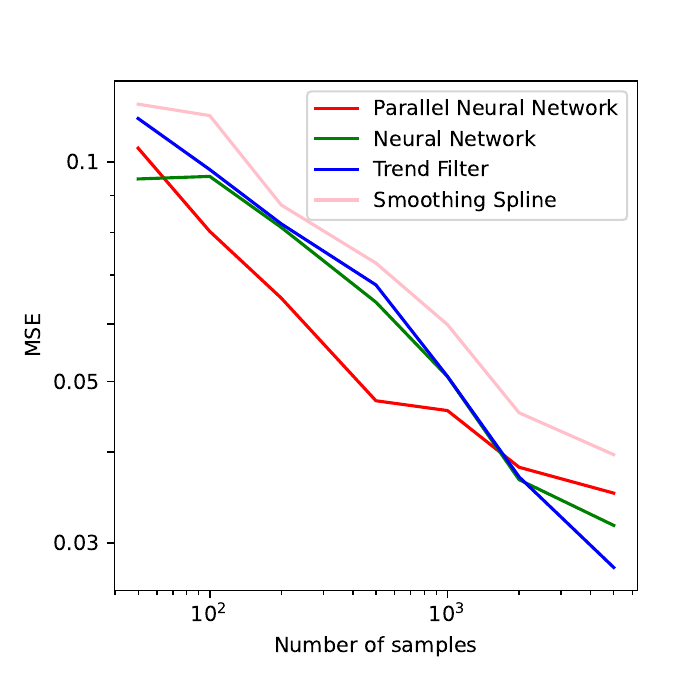}\vspace{\figdist}}
   \subcaptionbox{MSE versis $n$, ``Vary''}{\hspace{2mm}\includegraphics[width=0.3\textwidth]{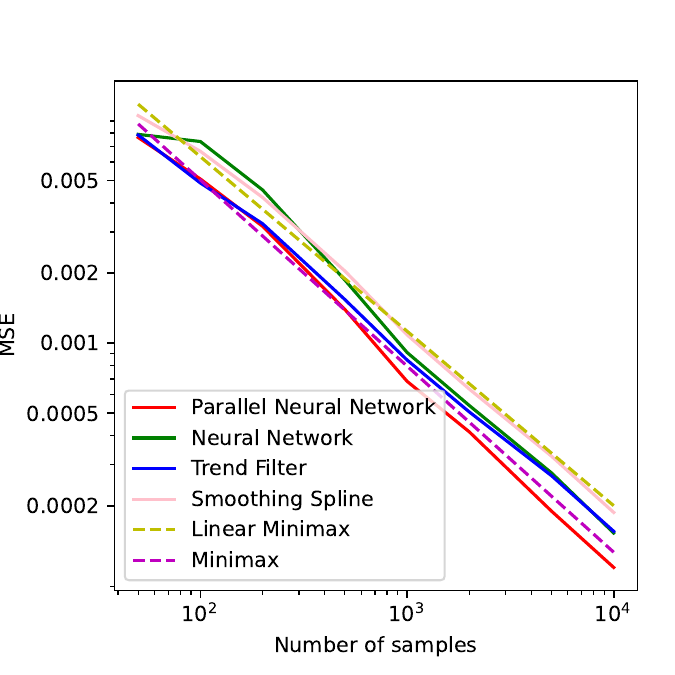}\vspace{\figdist}}\\
   \vspace{-3mm}\\

   \caption{Numerical experiment results of the Doppler function (a-c,h), and ``vary'' function (d-f,g). All the ``active'' subnetworks are plotted in (c)(f). The horizontal axis in (b) is not linear.
   % (a)(d): Estimations of different estimators.
   % % when the degree of freedom is 30(a) or 50(d). 
   % (b)(e) Mean squared error (MSE) with respect to degree of freedom (DoF).
   % % at the input points. 
   % % The shadowed area shows the 95\% confidence interval. 
   % Note that the horizontal axis in (b) is not linear. (c)(f): Output of all the ``active'' subnetwork. (g): Zooming into the left part of (a)(d). (h)(i): MSE with respect to number of samples.
   \vspace{-0.5cm}
   }
   \label{fig:exp}
\end{figure}

We empirically compare a parallel neural network (PNN) and a vanilla ReLU neural network (NN) with smoothing spline, trend filtering (TF) \citep{tibshirani2014adaptive}, and wavelet denoising.
Trend filtering can be viewed as a more efficient discrete spline version of locally adaptive regression spline and enjoys the same optimal rates for the BV classes. 
Wavelet denoising is also known to be minimax-optimal for the BV classes.
The results are shown in \autoref{fig:exp}. 
We use two target functions: a Doppler function whose frequency is decreasing(\autoref{fig:exp}(a)-(c)(h)), and a combination of piecewise linear function and piecewise cubic function, or ``vary'' function (\autoref{fig:exp}(d)-(f)(i)).
% In \autoref{fig:exp}(a)(d)(g), we show the estimation results using different methods when the degree of freedom is 15 (a), 30 (b) or 50(g).  
% In \autoref{fig:exp}(b)(e)(h), we show the mean squared error in predicting the true function at the input points. 
We repeat each experiment 10 times and take the average. The shallow area in \autoref{fig:exp}(b)(e) shows 95\% confidence interval by inverting the Wald's test.
% The degree of freedom of parallel neural network is defined as the number of subnetworks whose weights are nonzero. 
The degree of freedom (DoF) is computed based on \citet{tibshirani2015degrees}.

%As can be seen from the figure \autoref{fig:exp}(a)(b), 
As can be shown in the figure, both TF and wavelet denoising can adapt to the different levels of smoothness in the target function, while smoothing splines tend to be oversmoothed where the target function is less smooth 
(the left side in (a)(d), enlarged in (g)).
The prediction of PNN is similar to TF and wavelet denoising and shows local adaptivity.
% As for the mean squared error in predicting the true function at the input points, neural networks and trend filtering have smaller MSE than smoothing spline especially when the degree of freedom is relatively small, which agrees with the result in \citet{tibshirani2014adaptive}. 
% Neural networks can achieve lower error than trend filtering when the degree of freedom is very small. This is because 
%the degree of freedom are defined differently and 
% under the same degree of freedom, neural networks have more trainable parameters and thus higher flexibility. 
Besides, the MSE of PNN almost follows the same trend as TF and wavelet denoising
% lays between smoothing spline and trend filtering in most of the cases, and is closer to the latter,
which is consistent with our theoretical understanding that the error rate of neural network is closer to locally adaptive methods. Notably PNN, TF and wavelet denoising achieve lower error at a much smaller degree-of-freedom than smoothing splines. 
% In a vanilla NN, $\ell_2$ regularization in the weights is equivalent to $\ell_1$ regularization in any two successive layers, but to the best of our knowledge it does not lead to sparse representation learning unless some specific sparse structure is enforced.
% While our theory does not apply to vanilla neural networks, the results seem to suggest the NN behaves similar to smoothing spline and is \emph{not} locally adaptive. 

There are some mild drops in the best MSE one can achieve with Parallel NN vs TF in both examples. We are surprised that the drop is small because Parallel NN needs to learn the basis functions that TF essentially hard-coded. The additional price to pay for using a more adaptive and more flexible representation learning method seems not high at all. 

In \autoref{fig:exp}(c)(f), we give the output \emph{all} the ``active'' subnetwork, i.e. the subnetworks whose output is not a constant. 
Notice that the number of active subnetworks is much smaller than the initialization. 
This is because $\ell_2$ regularization in weights induces $\ell_p$ sparsity and the weight in most of the subnetworks reduces towards 0 after training.
More details are shown in \autoref{sec:expdet}.

{
In \autoref{fig:exp}(h)(i), we plot the MSE versus the number of training samples for ``Doppler'' and ``Vary'' respectively. It is clear that parallel NN works the best overall. In (i), we further compare the scaling of the MSE against the minimax rate ($n^{-4/5}$) and the minimax linear rate ($n^{-3/4}$), i.e., the best rate kernel methods could achieve.
%we also plotted the minimax rate ($n^{-4/5}$) and the minimax rate of linear estimators ($n^{-3/4}$). 
As is predicted by our theory, when $n$ is large, the MSE of parallel neural networks and trend filtering decreases at almost the same rate as the minimax rate, while smoothing splines, as expected, is converging at the (suboptimal) minimax linear rate. Interestingly, vanilla NN seems to converge at the optimal rate too on this example. It remains an open question whether vanila NN is merely ``lucky'' on this example, or it also  achieves the minimax rate for all functions in BV(m).
%However, it seems that the MSE of neural networks decreases faster than the minimax rate. The reason of this phenomenon remains to be studied.
}

%We also note that the maximum degree of freedom achieved by neural network is limited by optimization issues, as they are likely to get stuck at smooth local minima, rather than completely overfit to the noisy data. 

\section{Conclusion and Discussion}
\label{sec:conclusion}
% In this paper, we draw the connection between deep learning and nonparametric regression.
%we show that a two-layer neural network with truncated power basis function is equivalent to locally adaptive regression spline thus achieves the minimax rate of BV class. 
%e extend these results to multivariate function classes and deep ReLU neural network. 
In this paper, we show that a deep parallel neural network can be locally adaptive with standard $\ell_2$ regularization.
This confirms that neural networks can be nearly optimal in learning functions with heterogeneous smoothness which separates them from kernel methods.

Specifically, we prove that training an $L$ layer parallel neural network with standard $\ell_2$ regularization is equivalent to an $\ell_{2/L}$-penalized regression model with representation learning.
Since in typical application $L \gg 2$, standard regularization promotes a sparse linear combination of the learned bases. %, $0 < p \ll 1$, resulting in a sparse model. 
Using this method, we proved that a parallel neural network can achieve close to the minimax rate in the Besov space and bounded variation (BV) space by tuning the regularization factor. %and  without the need to tune the architecture of the neural network. 

\modif{
Our result reveals that one do not need to specify the smoothness parameter $\alpha$ (or $m$) when training a parallel neural network. 
With only an estimation of the upper bound of $\alpha$ (or $m$), parallel neural networks can adapt to different degree of smoothness, or choose different parameters for different regions of the domain of the target function. 
This property shows the strong adaptivity of deep neural networks.}
% This is a new type of adaptivity not possessed by traditional adaptive nonparametric regression methods like locally adaptive regression spline or trend filtering.

On the other hand, as the depth of neural network $L$ increases, $2/L$ tends to 0 and the error rate moves closer to the minimax rate of Besov and BV space.
This indicates that when the sample size is large enough, deeper models have smaller error than shallower models, and helps explain why empirically deep neural networks has better performance than shallow neural networks. 

\section*{Acknowledgments}
The work is partially supported by NSF Award \#2134214. The authors thank Alden Green for references on Besov-space embedding of BV classes, Dheeraj Baby for helpful discussion on B-splines as well as Ryan Tibshirani on connections to \citep{tibshirani2021equivalences} and for sharing with us a simpler proof of the Theorem 1 of \citep{parhi2021banach} based on Caratheorodory’s Theorem (used in the proof of Theorem~\ref{thm:2layer} on two-layer NNs).

\bibliography{main}
\bibliographystyle{iclr2023_conference}

\newpage
\onecolumn
\appendix

\section{Other related works} 
\label{sec:related}
\textbf{NN and kernel methods. } \citet{jacot2018neural} draws the connection between neural networks and kernel methods. However, it has been found that neural networks often outperform any kernel method, especially when the learning rate is relatively large \citep{lewkowycz2020large}. 
A series of work tried to distinguish NN from kernel methods by providing examples of function spaces that NN provably outperform kernel methods \citep{allen2019can, ghorbani2020neural}. 
However, these papers did not consider the local adaptivity of nerual networks, which provides a more systematic explanation.

\textbf{NN and splines. }
Besides \citet{parhi2021banach} which we discussed earlier, \citet{parhi2021kinds,parhi2021near} also leveraged the connections between NNs and splines. %Neither extension overlaps with what we do.
\citet{parhi2021kinds} focused on characterizing the variational form of multi-layer NN. \citet{parhi2021near} showed that two-layer ReLU activated NN achieves minimax rate for a BV class of order $1$ but did not cover multilayer NNs nor BV class with order $>1$, which is our focus.

\textbf{Weight-decay regularization with sparsity-inducing penalties. }
The connection between weight-decay regularization with sparsity-inducing penalties in two-layer NNs is folklore and used by \citet{neyshabur2014search,savarese2019infinite, ongie2019function,ergen2021convex,ergen2021revealing,parhi2021banach,parhi2021near,pilanci2020neural}. 
The key underlying technique --- an application of the AM-GM inequality (which we used in this paper as well) ---  can be traced back to \citet{srebro2004maximum} (see a recent exposition by \citet{tibshirani2021equivalences}). \citet{tibshirani2021equivalences}  also generalized the result to multi-layered NNs, but with a simple (element-wise) connections.
% {\citet{ergen2021global} generalized the results to a three-layer parallel neural network, and proved its equivalence to an $\ell_1$ sparse model, but this requires a non-standard regularizer. 
Besides, 
\citet{ergen2020implicit} proved that training a two-layer convolution neural network (CNN) with weight decay induces sparsity, and points to a potential extension to these works including our work.

Finally, it was brought to our attention that while \citet{savarese2019infinite} mainly consider two-layer NNs, a set of results about $L$-layer parallel NNs was presented in Appendix C of their paper, which essentially contains same arguments we used for proving the equivalence to an $\ell_{2/L}$ regularized optimization problem in  Proposition~\ref{prop:eqmodel}. The difference is they applied the insight to understand the interpolation regime while we focused on analyzing MSE in the noisy case. 

%showed that a parallel networks of depth $L$ have an inductive bias for the $L_{2/L}$ sparse model, and that an explicit weight decay is equivalent to an $\ell_{2/L}$

Proposition~\ref{prop:eqmodel}, is the \citet{savarese2019infinite} showed that a parallel networks of depth $L$ have an inductive bias for the $L_{2/L}$ sparse model,
and  explicit weight decay causes the solutions of these networks to have a sparse last layer with at most $n$ nonzero weights.

\modif{
\textbf{Resnet-type convolution neural networks.}
A recent series of work \citep{oono2019approximation, liu2021besov} proves that an arbitrary parallel neural network can be approximated by a resnet-type convolution neural networks. These works do not require the model to be sparse, thus are easier to train, yet they still require the architecture (the width and depth of each residual block, the number of residual blocks) to be tuned based on the dataset, and the estimation error analysis is based on the number of parameters.
Besides, the number of residual block need to increase with $n$, making the entire too deep to train in practice.
}

\textbf{Approximation and estimation.}
The approximation-theoretic and estimation-theoretic research for neural network has a long history too \citep{cybenko1989approximation,barron1994approximation,yarotsky2017error,schmidt2020nonparametric,suzuki2018adaptivity}. Most existing work considered the Holder, Sobolev spaces and their extensions, which contain only homogeneously smooth functions and cannot demonstrate the advantage of NNs over kernels. 
\modif{
The exceptions including \citet{suzuki2018adaptivity, oono2019approximation,liu2021besov} which, as we discussed earlier, requires modifications to NN architecture for each class. 
In contrast, we require tuning only the standard weight decay parameter. 
Most importantly, in all previously works, the estimation error of the model (eg. the covering number) depends on the number of nonzero parameters in the model, while our work provides a bound that depends on the norm of the weights instead of the number of subnetworks.}

%-----------------------------------------
% \section{Warm-up: Two-layer Neural Network}
\section{Two-layer Neural Network with Truncated Power Activation Functions}
\label{sec:warmup}
%In this section, we'll focus on a neural network similar to the one proposed in \citep{parhi2021banach} but with standard weight decay. Specifically, let the neural network with one-dimensional input
We start by recapping the result of \citet{parhi2021banach} and formalizing its implication in estimating BV functions.
\citet{parhi2021banach} considered a two layer neural network with truncated power activation function. % and a non-standard regularization.
%is equivalent to a polynomial spline in Radon domain. 
%In the univariate case,
% we can drop Radon transform and 
% the problem reduces to a trend filtering problem. 
Let the neural network be
\begin{eqal}
   f(x) = \sum_{j=1}^M v_j\act^m(w_j x + b_j) + c(x), 
   \label{eq:nnnstd}
\end{eqal}
% One can easily prove that with standard weight decay, the regularized regression problem
% \begin{eqal*}
%    \sum_{i=1}^n (f(x_i)-y_i)^2 + \sum_{k=1}^K (w_i^2 + v_i^2)
% \end{eqal*}
% is equivalent To
% \begin{eqal*}
%    \sum_{i=1}^n (g(x_i)-y_i)^2 + \sum_{k=1}^K a_k^{2/(m+1)}
% \end{eqal*}
% where $g(\cdot)$ has the form
% \begin{eqal*}
%    g(x) = \sum_{k=1}^K a_k\act^m(x + b_k)
% \end{eqal*}
% Here we abuse the notation of $b_k$.
where $w_j, v_j$ denote the weight in the first and second layer respectively, $b_j$ denote the bias in the first layer, $c(x)$ is a polynomial of order up to $m$, $\act^m(x) := \max(x, 0)^m$.
%Let $L$ be a convex loss function, $\gV$ a measurement operator, $\lambda > 0$ a constant, and $TV(\cdot)$ denote the total variation.  
\citet[Theorem 8]{parhi2021banach} showed that when $M$ is large enough, The optimization problem
\begin{eqal}
   \label{eq:l2nn}
   \min_{\vect w, \vect v} \hat L(f) + \frac{\lambda}{2}\sum_{j=1}^M (|v_j|^2 + |w_j|^{2m})
\end{eqal}
is equivalent to the locally adaptive regression spline:
\begin{eqal}
   \min_{f} \hat L(f) + \lambda TV(f^{(m)}(x)),
   \label{eq:tf}
\end{eqal}
which optimizes over arbitrary functions that is $m$-times weakly differentiable. The latter was studied in \citet{mammen1997locally}, which leads to the following MSE:
\def\thmtwolayer{
   Let $M \geq n-m$, and $\hat{f}$ be the function \autoref{eq:nnnstd} parameterized by the minimizer of \eqref{eq:l2nn}, then
   \begin{eqal*}
      \MSE(\hat{f}) = O(n^{-(2m+2)(2m+3)}).
   \end{eqal*}
}
\begin{theorem}
   \label{thm:2layer}
   \thmtwolayer
\end{theorem}

% % \subsection{Proof of \autoref{thm:2layer}}
% \label{sec:proofthm2layer}
% \textbf{\autoref{thm:2layer}.}\textit{\thmtwolayer}
We show a simpler proof in the univariate case due to \citet{tetibshirani2022private}:

\begin{proof}
As is shown in \citet[Theorem 8]{parhi2021banach}, the minimizer of \autoref{eq:l2nn} satisfy 
\begin{eqal*}
   |v_j| = |w_j|^{m}, \forall k
\end{eqal*}
so the TV of the neural network $f_{NN}$ is 
\begin{eqal*}
   TV^{(m)}(f_{NN}) &= TV^{(m)} c(x) + \sum_{j=1}^M |v_j||w_j|^m TV^{(m)}(\act^{(m)}(x))\\
   &= \sum_{j=1}^M |v_j||w_j|^m\\
   &= \frac{1}{2}\sum_{j=1}^M (|v_j|^2 + |w_j|^{2m})
\end{eqal*}
which shown that \autoref{eq:l2nn} is equivalent to the locally adaptive regression spline \autoref{eq:tf} as long as the number of knots in \autoref{eq:tf} is no more than $M$. 
Furthermore, it is easy to check that any spline with knots no more than $M$ can be expressed as a two layer neural network \autoref{eq:l2nn}. 
It suffices to prove that the solution in \autoref{eq:tf} has no more than $n-m$ number of knots. 

 \citet[Proposition 1]{mammen1997locally} showed that there is a solution to \autoref{eq:tf} $\hat f(x)$ such that $\hat f(x)$ is a $m$th order spline with a finite number of knots but did not give a bound. Let the number of knots be $M$, we can represent $\hat f$ using the truncated power basis
\begin{eqal*}
   \hat f(x) = \sum_{j=1}^M a_j (x - t_j)_+^m + c(x)
   := \sum_{j=1}^M a_j \act_j^{(m)}(x) + c(x)
\end{eqal*}
where $t_j$ are the knots, $c(x)$ is a polynomial of order up to $m$, and define $\act_j^{(m)}(x) =(x - t_j)_+^m$. 

\citet{mammen1997locally} however did not give a bound on $M$. \citet{parhi2021banach}'s Theorem~1 implies that $M\leq n-m$. Its proof is quite technical and applies more generally to a higher dimensional generalization of the BV class.

\citet{tetibshirani2022private} communicated to us the following elegant argument to prove the same using elementary convex analysis and linear algebra, which we present below.

Define $\Pi_m(f)$ as the $L^2(P_n)$ projection of $f$ onto polynomials of degree up to $m$, $\Pi_m^{\bot}(f) := f- \Pi_m(f)$.
It is easy to see that 
\begin{eqal*}
   \Pi_m^\bot f(x) = \sum_{j=1}^M a_j \Pi_m^\bot \act_j^{(m)}(x)
\end{eqal*}
% $TV(\hat f) = m!\sum_{j=1}^M |a_j|$.

Denote $f(x_{1:n}) := \{f(x_1), \dots, f(x_n)\} \in \R^n$ as a vector of all the predictions at the sample points.
\begin{eqal*}
   &\Pi_m^\bot \hat f(x_{1:n}) = \sum_{j=1}^M a_j \Pi_m^\bot \act_j^{(m)}(x_{1:n}) \in  \Pi_m^\bot\conv\{\pm \act_j^{(m)}(x_{1:n})\}\cdot\sum_{j=1}^M |a_j| \\
   &\in  \conv\{\pm \Pi_m^\bot\act_j^{(m)}(x_{1:n})\}\cdot\sum_{j=1}^M |a_j|
\end{eqal*}
where $\conv$ denotes the convex hull of a set. 
The convex hull $\conv\{\pm \act_j^{(m)}(x_{1:n})\}\cdot\sum_{j=1}^M |a_j|$ is an $n$-dimensional space, and polynomials of order up to $m$ is an $m+1$ dimensional space, so the set defined above has dimension $n-m-1$.
By Carath\'eodory's theorem, there is a subset of points in this space 
\begin{eqal*}
   \{\Pi_m^\bot\act_{j_k}^{(m)}(x_{1:n})\} \subseteq \{\Pi_m^\bot\act_j^{(m)}(x_{1:n})\}, 1 \leq k \leq n-m
\end{eqal*}
such that 
\begin{eqal*}
   \Pi_m^\bot f(x) = \sum_{k=1}^{n-m} \tilde a_k \Pi_m^\bot \act_{j_k}^{(m)}(x), 
   \sum_{k=1}^{n-m} |a_k| \leq 1
\end{eqal*}
In other word, there exist a subset of knots $\{\tilde t_j, j \in [n-m]\}$ that perfectly recovers $\Pi_m^\bot \hat f(x)$ at all the sample points, and the TV of this function is no larger than $\hat f$.
% Furthermore, 
% \begin{eqal*}
%    TV(\sum_{k=1}^{n-m} \tilde a_k \Pi_m^\bot \act_{j_k}^{(m)}(x)) = \sum_{k=1}^{n-m}
% \end{eqal*}

This shows that
$$
   \tilde f(x) = \sum_{j=1}^{n-m} \tilde a_j (x - t_j)_+^m, 
   s.t. \tilde f(x_i) = f(x_i)
$$
for all $x_i$ in $n$ onbservation points. 

The MSE of locally adaptivity regressive spline \autoref{eq:tf} was studied in \citet[Section 3]{mammen1997locally}, which equals the error rate given in \autoref{thm:2layer}.
\end{proof}

%The rate is optimal for $BV(m)$ up to a factor of constant.
This indicates that the neural network \autoref{eq:nnnstd} is minimax optimal for $BV(m)$. 

Let us explain a few the key observations behind this equivalence. (a) The truncated power functions (together with an $m$th order polynomial) spans the space of an $m$th order spline.
(b) The neural network in \eqref{eq:nnnstd} is equivalent to a free-knot spline with $M$ knots (up to reparameterization).
(c) A solution to \eqref{eq:tf} is a spline with at most $n-m$ knots \citep[Theorem 8]{parhi2021banach}.
(d) Finally, by the AM-GM inequality 
$$
|v_j|^2 + |w_j|^{2m} \geq 2|v_j||w_j|^{m} = 2 |c_j|
$$
where $c_j = v_j|w_j|^{m}$ is the coefficient of the corresponding $j$th truncated power basis. The $m$th order total variation of a spline is equal to $\sum_j|c_j|$. It is not hard to check that the loss function depends only on $c_j$, thus the optimal solution will always take ``$=$'' in the AM-GM inequality.

% \begin{theorem}
%    \citep[Section 5]{tibshirani2014adaptive} Let $\gV$ be an evenly spaced discrete measurement with size $n$, and $L$ be mean square error (MSE) loss $\frac{1}{n} \sum_{i=1}^n (f(x_i) - y_i)^2$ where $y_i = f_0(x_i) + \epsilon_i$ is the observation with iid. zero mean sub-Gaussian noise , $f_0$ is an unknown regression function to be estimated, then with $\lambda = O(n^{1/(2m+3)})$, the solution to problem \autoref{eq:tf} denoted as $\hat f$ satisfy
%    \begin{eqal*}
%       \frac{1}{n} \sum_{i=1}^n (\hat f(x_i) - f_0(x_i))^2 = O(n^{-(2m+2)(2m+3)})
%    \end{eqal*}
%    which is also the minimax rate of for estimation over function class $TV^{(m)}$, i.e. functions with bounded total variation of $m$-th order derivative.
% \end{theorem}

%This statement was proved in \citet[Theorem 1 \& Theorem 8]{parhi2021banach}. 

% One remarkable property of the NN in \eqref{eq:nnnstd} is that it \emph{is} a spline, rather than approximates one. It, however, does not involve representation learning beyond selecting where the knots are for the spline. Moreover, neither the activation function nor the regularization term is commonly used in practice. 
% %\autoref{sec:warmup} is rarely used in practice. 
% %non-standard in the sense of activation function, residual connection and regularizer. 
% %The natural question to ask is whether similar estimation error can be achieved by neural networks with more standard architecture. 
% In the next section, we provide a new result using parallel DNN with only ReLU and weight decay.

%%%%%%%%%%%%%%%%%%%%%%%%%%%%%%%%%%%%%%%%%%%%%%%%%%%%%%%%%%%

\section{Introduction To Common Function Classes}

In the following definition define $\Omega$ be the domain of the function classes, which will be omitted in the definition.

\subsection{Besov Class}
\label{sec:besov}
% Extension of total variation (TV) class to multivariate functions is not simple. In this paper, we focus on Besov class instead.

\begin{definition}
   Modulus of smoothness: For a function $f \in L^p(\Omega)$ for some $1 \leq p \leq \infty$, the $r$-th modulus of smoothness is defined by 
   \begin{equation*}
      w_{r,p}(f, t) = \sup_{h \in \R^d: \|h\|_2\leq t} \|\Delta_h^r(f)\|_p,
   \end{equation*}
   \begin{empheq}[left={\Delta_h^r(f):=\empheqlbrace}, ]{align*}
      &\sum_{j=0}^r (_j^r)(-1)^{r-j}f(x+jh), &\textrm{ if } x\in \Omega, x+rh \in \Omega, \\
      & 0, &\textrm{otherwise}.
   \end{empheq}
\end{definition}

\begin{definition}
   Besov space: For $1 \leq p, q \leq \infty, \alpha > 0, r:=\lceil \alpha \rceil + 1$, define 
   \begin{empheq}[left={|f|_{B^{\alpha}_{p,q}} = \empheqlbrace}]{align*}
      &\Big(\int_{t=0}^\infty(t^{-\alpha}w_{r,p}(f,t))^q \frac{dt}{t}\Big)^\frac{1}{q}, &q < \infty \\
      &\sup_{t>0}t^{-\alpha}w_{r,p}(f,t), & q=\infty,
   \end{empheq} 
   and define the norm of Besov space as:
   \begin{equation*}
      \|f\|_{B^{\alpha}_{p,q}} = \|f\|_p + |f|_{B^{\alpha}_{p,q}}.
   \end{equation*}
   A function $f$ is in the Besov space $B^{\alpha}_{p,q}$ if $\|f\|_{B^{\alpha}_{p,q}}$ is finite.
\end{definition}

% We omit the function domain $\Omega$ in the statements.
% Typically, it is required that $p, q \geq 1$.
Note that the Besov space for $0 < p, q < 1$ is also defined, but in this case it is a quasi-Banach space instead of a Banach space and will not be covered in this paper.

Functions in Besov space can be decomposed using B-spline basis functions. 
Any function $f$ in Besov space $B^\alpha_{p, q}, \alpha > d/p$ can be decomposed using B-spline of order $m, m > \alpha$: let $\vect x \in \R^d$, 
\begin{eqal}
   f(\vect x) = \sum_{k=0}^\infty \sum_{\vect s\in J(k)} c_{k,\vect s}(f) M_{m, k, \vect s}(\vect x)
   \label{eq:besovdecom}
\end{eqal}
where $J(k):=\{2^{-k}\vect s: \vect s \in [-m, 2^k+m]^d \subset \Z^d \}$, $M_{m, k, \vect s} (\vect x):= M_m(2^k (\vect x - \vect s))$, and $M_k(\vect x) = \prod_{i=1}^d M_k(x_i)$ is the cardinal B-spline basis function which can be expressed as a polynomial:
\begin{eqal}
   M_m(x) &= \frac{1}{m!} \sum_{j=1}^{m+1} (-1)^j {m+1 \choose j}(x-j)^m_+ \\
   &= ((m+1)/2)^m\frac{1}{m!} \sum_{j=1}^{m+1} (-1)^j {m+1 \choose j}\Big(\frac{x-j}{(m+1)/2}\Big)^m_+,
   \label{eq:bspline}
\end{eqal}
%\vect s\in \Z^d, -m \leq s_i \leq 2^k+m, 
%i \in [d]\}$.
Furthermore, the norm of Besov space is equivalent to the sequence norm:
\begin{eqal*}
   \|\{c_{k,\vect s}\}\|_{b^\alpha_{p, q}} := \Bigl(\sum_{k=0}^\infty(2^{(\alpha-d/p)k}\|\{c_{k,\vect s}(f)\}_{\vect s}\|_p)^{q}\Bigr)^{1/q} 
   \eqsim \|f\|_{B^\alpha_{p,q}}.
\end{eqal*}
See e.g. \citet[Theorem 2.2]{dung2011optimal} for the proof.

The Besov space is closely connected to other function spaces including the H\"older space ($\gC^\alpha$) and the Sobolev space ($W^\alpha_p$). Specifically, if the domain of the functions is $d$-dimensional \citep{suzuki2018adaptivity, sadhanala2021multivariate},
\begin{itemize}
   \item $\forall \alpha \in \N$, $B^\alpha_{p, 1} \subset W^\alpha_p \subset B^\alpha_{p, \infty}$, and $B^\alpha_{2, 2} = W^\alpha_2$.
   \item For $0 < \alpha < \infty$ and $\alpha \in \gN, \gC^\alpha = B^\alpha_{\infty, \infty}$.
   \item If $\alpha > d/p$, $B^\alpha_{p, q} \subset \gC^0$.
\end{itemize}

\subsection{Other Function Spaces}
\label{sec:tv}

\begin{definition}
   H\"older space: 
   let $m \in \N$, the $m$-th order Holder class is defined as
   \begin{eqal*}
      \gC^m = \left\{f: \max_{|a| = k}\frac{\left| D^a f(x) - D^a f(z)\right|}{\|x-z\|_2} < \infty, \forall x, z \in \Omega\right\}
   \end{eqal*}
   where $D^a$ denotes the weak derivative.
   % let $\alpha \in \N$, define the H\"older norm as 
   % \begin{eqal*}
   %    \|f\|_{\gC^\alpha} := \max_{|a|\leq \beta} \|D^m f\|_\infty + \max_{|a|=m} \sup_{x, y \in \Omega} \frac{|D^a f(x) - D^a f(y)|}{|x-y|^{\alpha-m}},
   % \end{eqal*}
   % where $\Omega$ is the domain of the function class, $m = \ceil \alpha \rceil$, $D^a$ denotes the weak derivative.
   % The H\"older space $\gC^\alpha$ is the set of functions with finite H\"older norm: 
   % \begin{eqal*}
   %    \gC^\alpha = \{f: \|f\|_{\gC^\alpha} < \infty\}.
   % \end{eqal*}
\end{definition}

Note that fraction order of H\"older space can also be defined. For simplicity, we will not cover that case in this paper.

\begin{definition}
   Sobolev space: let $m\in\gN, 1 \leq p \leq \infty$, the Sobolev norm is defined as
   \begin{eqal*}
      \|f\|_{W^m_p} := \left(\sum_{|a|\leq m} \|D^a f\|_p^p\right)^{1/p},
   \end{eqal*}
   the Sobolev space is the set of functions with finite Sobolev norm:
   \begin{eqal*}
      W^m_p :=  \{f: \|f\|_{W^m_p} < \infty\}.
   \end{eqal*}
\end{definition}

\begin{definition}
   Total Variation (TV): The total variation (TV) of a function $f$ on an interval $[a, b]$ is defined as 
   \begin{eqal*}
      TV(f) = \sup_{\gP} \sum_{i=1}^{n_\gP-1} |f(x_{i+1}) - f(x_i)|
   \end{eqal*}
   where the $\gP$ is taken among all the partitions of the interval $[a, b]$.
\end{definition}

% If $f$ is (weakly) differentiable, it has the equivalent definition
% \begin{eqal*}
%    TV(f) = \int_a^b|f'(x)|dx.
% \end{eqal*}
In many applications, functions with stronger smoothness conditions are needed, which can be measured by high order total variation.

\begin{definition}
   High order total variation: the $m$-th order total variation is the total variation of the $(m-1)$-th order derivative
   \begin{eqal*}
      TV^{(m)}(f) = TV(f^{(m-1)})
   \end{eqal*}
\end{definition}

\begin{definition}
   Bounded variation (BV): The $m$-th order bounded variation class is the set of functions whose total variation (TV) is bounded.
%$a \lesssim b$ to denote $a \leq C b$ for some constant $C$ that does not depend on $a$ or $b$, and $a \eqsim b$ to denote $a \lesssim b$ and $ b\lesssim a$. We use $M_m(\cdot)$ to denote $m$-th order Cardinal B-spline basis functions, and define $B_{m, k, \vect s}(\vect x) := B_m(2^k(\vect x - \vect s))$ as the multi-resolution B-spline basis functions.functions that has finite $m$-th order total variation:
   \begin{eqal*}
      BV(m) := \{f: TV(f^{(m)}) < \infty\}.
   \end{eqal*}
\end{definition}

\section{Proof of Estimation Error}

% \subsection{Proof of \autoref{prop:eqmodel}}
\subsection{Equivalence Between Parallel Neural Networks and $p$-norm Penalized Problems}
\label{sec:proofpropeqmodel}
\textbf{\autoref{prop:eqmodel}.}\textit{\propeqmodel}

\begin{proof}
% Using Lagrange's method, one can easily find \autoref{eq:l2} is equivalent to a constrained optimization problem:
% \begin{eqal}
%    % \argmin_{\{\mat W^{(\ell)}_j, \vect b^{(\ell)}_j\}} &\frac{1}{n}\sum_{i=1}^n \ell\Big(\sum_{j=1}^M f_j(\vect x_i), y_i\Big), \\
%    \argmin_{\{\mat W^{(\ell)}_j, \vect b^{(\ell)}_j\}} &\hat L\Big(\sum_{j=1}^M f_j\Big), &
%    s.t. &\sum_{j=1}^M\sum_{\ell=1}^L \big\|\mat W^{(\ell)}_j\big\|_F^2 \leq P
%    \label{eq:l2c}
% \end{eqal}
% for some constant $P$ that depends on $\lambda$ and the dataset $\gD$.

We make use of the property from \autoref{eq:eqv} to minimize the constraint term in \autoref{eq:l2c} while keeping this neural network equivalent to the original one.
Specifically, let $\mat W^{(1)}, \vect b^{(1)}, \dots \mat W^{(L)}, \vect b^{(L)}$ be the parameters of an $L$-layer neural network. 
\begin{equation*}
   f(x) = \mat W^{(L)}\act(\mat W^{(L-1)}\act(\dots \act(\mat W^{(1)} x+\vect b^{(1)}) \dots) +\vect b^{(L-1)}) + \vect b^{(L)},
\end{equation*}
which is equivalent to 
\begin{equation*}
   f(x) = \alpha_L \tilde{\mat W}^{(L)}\act(\alpha_{L-1}\tilde{\mat W}^{(L-1)}\act(\dots \act(\alpha_1 \tilde{\mat W}^{(1)} x+\tilde{\vect b}^{(1)}) \dots) +\tilde{\vect b}^{(L-1)}) + \tilde{\vect b}^{(L)},
\end{equation*}
%for some $\{\vect {\tilde b}^{(\ell)}\}$, 
as long as  $\alpha_\ell > 0, \prod_{\ell=1}^L \alpha^L = \prod_{\ell=1}^L \|\mat W^{(\ell)}\|_F$, where $\tilde {\mat W}^{(\ell)} := \frac{\mat W^{(\ell)}}{\|\mat W^{(\ell)}\|_F}$.
% then all the neural networks that are equivalent to this one by scaling has parameters $a_1 \tilde{\mat W}^{(1)}, \tilde {\vect b}^{(1)}, \dots, a_\ell\tilde{\mat W}^{(L)}, \tilde{\vect b}^{(L)}$ for some $a_\ell > 0$, and $\tilde{\vect b}^{(\ell)}$ depends on $\{a_\tau, \forall \tau\leq \ell\}$ and $\vect b_\ell$. 
% % Here we slightly abuse the notation $a$.
% Furthermore, $\prod_{\ell=1}^L  b_\ell = \prod_{\ell=1}^L \|\mat W^{(\ell)}\|_F$. 
By the AM-GM inequality, the $\ell_2$ regularizer of the latter neural network is 
\begin{eqal*}
   \sum_{\ell=1}^L \|\alpha_\ell \tilde{\mat W}^{(\ell)}\|_F^2 
   = \sum_{\ell=1}^L \alpha_\ell^2 
   \geq L\left( \prod_{\ell=1}^L a_\ell \right)^{2/L} 
   = L\left( \prod_{\ell=1}^L  \|\mat W^{(\ell)}\|_F \right)^{2/L} 
\end{eqal*}
and equality is reached when $\alpha_1 =\alpha_2 =\dots = \alpha_L$.
%$\|\mat W^{(1)}\| = \|\mat W^{(2)}\| =\dots=\|\mat W^{(\ell)}\|$. 
In other word, in the problem \autoref{eq:l2}, it suffices to consider the network that satisfies 
\begin{eqal}
   \|\mat W^{(1)}_j\|_F = \|\mat W^{(2)}_j\|_F = \dots = \|\mat W^{(L)}_j\|_F,
   \forall j \in [M], \ell \in [L].
   \label{eq:suffice}
\end{eqal}
Using \autoref{eq:eqv} again, one can find that the neural network is also equivalent to 
\begin{eqal*}
   f(x) &= \sum_{j=1}^M  a_j \bar{\mat W}^{(L)}\act(\bar{\mat W}_j^{(L-1)}\act(\dots \act(\bar{\mat W}_j^{(1)} x+\bar{\vect b}_j^{(1)}) \dots) +\bar{\vect b}_j^{(L-1)}) + \bar{\vect b}_j^{(L)},
\end{eqal*}
where 
\begin{eqal}
   \|\bar{\mat W}_j^{(\ell)}\|_F \leq \beta^{(\ell)},
   a_j = \frac{\prod_{\ell=1}^L \|\mat W^{(\ell)}_j\|_F}{\prod_{\ell=1}^L \beta^{(\ell)}} = \frac{\|\mat W^{(1)}_j\|_F^L}{\prod_{\ell=1}^L \beta^{(\ell)}} = \frac{(\sum_{\ell=1}^L\|\mat W^{(\ell)}_j\|_F^2/L)^{L/2}}{\prod_{\ell=1}^L \beta^{(\ell)}},
\label{eq:wj}
\end{eqal}
where the last two equality comes from the assumption \autoref{eq:suffice}.
Choosing $\beta^{(\ell)}=c_1\sqrt{w}$ expect $\ell=1$ where $\beta^{(1)} = c_1\sqrt{d}$, and scaling $\vect{\bar b}^{(\ell)}$ accordingly and taking the regularizer in \autoref{eq:l2} into \autoref{eq:wj} finishes the proof.
\end{proof}

% \subsection{Proof of \autoref{thm:pnncv}}
\subsection{Covering Number of Parallel Neural Networks}
\label{sec:proofthmpnncv}
% In the following discussion, we focus on a constrained optimization problem:
% \begin{eqal}
%    % \argmin_{\{\mat W^{(\ell)}_j, \vect b^{(\ell)}_j\}} &\frac{1}{n}\sum_{i=1}^n \ell\Big(\sum_{j=1}^M f_j(\vect x_i), y_i\Big), \\
%    \argmin_{\{\mat W^{(\ell)}_j, \vect b^{(\ell)}_j\}} &\hat L\Big(\sum_{j=1}^M f_j\Big), \\
%    %----------------------
%    s.t.\ & \|\mat {\bar W}_j^{(1)}\|_F \leq c_1 \sqrt{d}, \forall j \in [M]; \\
%    & \|\mat {\bar W}_j^{(\ell)}\|_F \leq c_1  \sqrt{w}, \forall j \in [M],  2 \leq \ell \leq L,
%    \quad \|\{a_j\}\|_{2/L}^{2/L} \leq P'
%    \label{eq:l2c}
% \end{eqal}

\textbf{\autoref{thm:pnncv}.}\textit{\thmpnncv}

The proof relies on the covering number of each subnetwork in a parallel neural network (\autoref{lemma:nncn1}), observing that $|f(x)| \leq 2^{L-1}w^{L-1}\sqrt{d}$ under the condition in \autoref{lemma:nncn1},  and then apply \autoref{lemma:pnorm}. 
We argue that our choice of condition on $\|\vect b^{(\ell)}\|_2$ in \autoref{lemma:nncn1} is sufficient to analyzing the model apart from the bias in the last layer, because it guarantees that $\sqrt{w}\|\mat W^{(\ell)}\gA_{\ell-1}(x)\|_2\leq \|\vect b^{(\ell)}\|_2$. 
This leads to 
$$
\|\mat W^{(\ell)}\gA_{\ell-1}(\vect x)\|_\infty \leq \|\mat W^{(\ell)}\gA_{\ell-1}(\vect x)\|_2 \leq \sqrt{w} \|\vect b^{(\ell)}\|_2 \leq \|\vect b^{(\ell)}\|_\infty
$$
If this condition is not met, $\mat W^{(\ell)}\gA_{\ell-1}(\vect x) + b^{(\ell)}$ is either always positive or always negative for all feasible $\vect x$ along at least one dimension. 
If $(\mat W^{(\ell)}\gA_{\ell-1}(\vect x) + b^{(\ell)})_i$ is always negative, one can replace $b^{(\ell)})_i$ with $-\max_{\vect x} \|\mat W^{(\ell)}\gA_{\ell-1}(\vect x)\|_\infty$ without changing the output of this model for any feasible $\vect x$. 
If $(\mat W^{(\ell)}\gA_{\ell-1}(\vect x) + b^{(\ell)})_i$ is always positive, one can replace $b^{(\ell)})_i$ with $\max_{\vect x} \|\mat W^{(\ell)}\gA_{\ell-1}(\vect x)\|_\infty$, and adjust the bias in the next layer such that the output of this model is not changed for any feasible $\vect x$.
In either cases, one can replace the bias $\vect b^{(\ell)}$ with another one with smaller norm while keeping the model equivalent except the bias in the last layer.

\newcommand\lemmanncn[1][]{
   Let $\gF \subseteq  \{f: R^d \rightarrow \R \}$ denote the set of $L$-layer neural network (or a subnetwork in a parallel neural network) with width $w$ in each hidden layer.
   It has the form
   % \if #1 label
   %    \let\eqenv eqal
   % \else
   %    \let\eqenv eqal*
   % \fi
   \begin{eqal}
      f(x) &= \mat W^{(L)}\act(\mat W^{(L-1)}\act(\dots \act(\mat W^{(1)} x+\vect b^{(1)})
       \dots) +\vect b^{(L-1)}) + \vect b^{(L)}, \\
      %------------------------
      \mat W^{(1)} &\in \R^{w \times d}, \|\mat W^{(1)}\|_F \leq \sqrt{d}, \vect b^{(1)} \in \R^w, \|\vect b^{(1)}\|_2\leq \sqrt{dw},\\
      %-------------------------
      \mat W^{(\ell)} & \in \R^{w\times w} \|\mat W^{(\ell)}\|_F \leq \sqrt{w}, \vect b^{(\ell)} \in \R^w, \|\vect b^{(\ell)}\|_2\leq 2^{\ell-1} w^{\ell-1}\sqrt{dw},
      \quad \forall \ell = 2, \dots L-1, \\
      %-------------------------
      \mat W^{(L)} & \in \R^{1 \times w}, \|\mat W^{(L)}\|_F \leq \sqrt{w}, b^{(L)} = 0
      \ifx\\#1\\
         \nonumber
      \else
         \label{eq:covcond}
      \fi
   \end{eqal}
   and $\act(\cdot)$ is the ReLU activation function, the input satisfy $\|x\|_2 \leq 1$,
   then the supremum norm $\delta$-covering number of $\gF$ obeys
   \begin{eqal*}
      \log \gN(\gF, \delta) \leq c_7 Lw^2 \log (1/\delta) + c_8
    %   \leq c_7 w^2 L^2 \log(w) + c_8 Lw^2 \log (1/\delta) 
   \end{eqal*}
   where $c_7$ is a constant depending only on $d$, and $c_8$ is a constant that depend on $d, w$ and $ L$.
}
\begin{lemma}
   \label{lemma:nncn1}
   \lemmanncn[label]
\end{lemma}

% \yw{In the above, if it is $\epsilon$-convering, then $\delta$ should be $\epsilon$?  Or you can just call it ``covering number'' rather than ``$\epsilon$-covering number''.}

\begin{proof}
First study two neural networks which differ by only one layer. Let $g_\ell, g_\ell'$ be two neural networks satisfying \autoref{eq:covcond} with parameters $\mat W_1, \vect b_1, \dots, \mat W_L, \vect b_L$ and $\mat W'_1, \vect b'_1, \dots, \mat W'_L, \vect b'_L$ respectively. Furthermore, the parameters in these two models are the same except the $\ell$-th layer, which satisfy
\begin{eqal*}
   \|\mat W_\ell - \mat W'_\ell\|_F \leq \epsilon, \|\vect b_\ell - \vect b'_\ell\|_2 \leq \tilde\epsilon. 
\end{eqal*}
Denote the model as
\begin{eqal*}
   g_\ell(x) = \gB_\ell(\mat W_\ell\gA_\ell(\vect x) + \vect b_\ell), 
   g'_\ell(x) = \gB_\ell(\mat W'_\ell\gA_\ell(\vect x) + \vect b'_\ell)
\end{eqal*}
where 
$\gA_\ell(\vect x) = \act(\mat W_{\ell-1}\act(\dots \act(\mat W_1 x + \vect b_1)\dots) + \vect b_{\ell-1})$ 
denotes the first $\ell-1$ layers in the neural network, and $\gA_\ell(x) = \mat W_{L}\act(\dots \act(\mat W_{\ell+1} \act(x) + \vect b_{\ell+1})\dots) + \vect b_{L})$ denotes the last $L-\ell-1$ layers, with definition $\gA_1(\vect x) = \vect x, \gB_L(\vect x)=\vect x$.

Now focus on bounding $\|\gA(\vect x)\|$. Let $\mat W \in \R^{m \times m'}, \|\mat W\|_F \leq \sqrt{m'}, \vect x \in \R^{m'}, \vect b \in \R^m, \|\vect b\|_2 \leq \sqrt{m}$
\begin{eqal*}
   \|\act(\mat Wx + \vect b)\|_2 & \leq \| \mat W \vect x + \vect b\|_2\\
   & \leq \|\mat W\|_2 \|\vect x\|_2 + \|\vect b\|_2\\
   & \leq \|\mat W\|_F \|\vect x\|_2 + \|\vect b\|_2\\
   & \leq \sqrt{m'} \|\vect x\|_2 + \sqrt{m}
\end{eqal*}
where we make use of $\|\cdot\|_2 \leq \|\cdot\|_F$.
% \yw{Are there differences between using $\|\mat W\|_F$ or using $\|\mat W\|_2$? The inequality above is true for the operator norm too.}
Because of that,
\begin{eqal}
   % \|\gA_\ell(\vect x)\|_2 &= \act(\mat W_{\ell-1} \gA_{\ell-1} + \vect b_{\ell-1})\\
   % &\leq \sqrt{w}(\|\gA_{\ell-1}\|_2 + (2\sqrt{w})^{\ell-2})\\
   % &\leq \sqrt{w}(\sqrt{w}(\|\gA_{\ell-2}\|_2+(2\sqrt{w})^{\ell-3})+ (2\sqrt{w})^{\ell-2})\\
   % &\leq \dots\\
   % &\leq w^{(\ell-2)/2} (\|\mat W_1 x + \vect b_1\|_2 + \ell-2) \\
   % &\leq w^{(\ell-2)/2} (\sqrt{d} + \sqrt{w} + \ell-2) \\
   \|\gA_2(\vect x)\|_2 &\leq \sqrt{d} + \sqrt{dw} \leq 2\sqrt{dw}, \\
   \|\gA_3(\vect x)\|_2 &\leq \sqrt{w}\|\gA_2(\vect x)\|_2 + 2w\sqrt{dw} \leq 4w\sqrt{dw}, \\
   & \dots\\
   \|\gA_\ell(\vect x)\|_2 & \leq \sqrt{w}\|\gA_{\ell-1}(\vect x)\|_2 \leq 2\sqrt{dw}(2{w})^{\ell-2}.
   \label{eq:boundA}
\end{eqal} 
% with the obvious exception $\|\gA_1(\vect x)\|_2 \leq 1$.

Then focus on $\gB(\vect x)$. Let $\mat W \in \R^{m \times m'}, \|\mat W\|_F \leq \sqrt{m'}, \vect x, \vect x' \in \R^{m'}, \vect b \in \R^m, \|\vect b\|_2 \leq \sqrt{m}$. Furthermore, $\|\vect x - \vect x'\|_2 \leq \epsilon$,
then 
\begin{eqal*}
   \|\act(\mat W\vect x+\vect b) - \act(\mat W\vect x'+\vect b)\|_2 &\leq\|\mat W(\vect x-\vect x')\|_2
   \leq \|\mat W\|_F \|\vect x-\vect x'\|_2
\end{eqal*}
which indicates that $\|\gB(\vect x) - \gB(\vect x)'\|_2 \leq (\sqrt{w})^{L-\ell} \|\vect x - \vect x'\|_2$

Finally, for any $\mat W, \mat W' \in \R^{m \times m'}, \vect x \in \R^{m'}, \vect b,\vect b' \in \R^m$, 
% \yw{and $\vect b,\vect b'$?}
one have
\begin{eqal*}
   \|(\mat W \vect x + \vect b) - (\mat W'\vect x + \vect b')\|_2 
   &=  \|(\mat W - \mat W') \vect x + (\vect b-\vect b')\|_2 \\
   & \leq \|\mat W - \mat W'\|_2 \|\vect x\|_2 + \|\vect b - \vect b'\|_2. \\
   & \leq \|\mat W - \mat W'\|_F \|\vect x\|_2 + \sqrt{m}\|\vect b - \vect b'\|_\infty.
\end{eqal*}

In summary,
\begin{eqal*}
   |g_\ell(\vect x) - g_\ell'(\vect x)| 
   &= |\gB_\ell(\mat W_\ell\gA_\ell(\vect x) + \vect b_\ell) - \gB_\ell(\mat W'_\ell\gA_\ell(\vect x) + \vect b'_\ell)|\\
   &\leq (\sqrt{w})^{L-\ell}\|(\mat W_\ell\gA_\ell(\vect x) + \vect b_\ell)-(\mat W'_\ell\gA_\ell(\vect x) + \vect b'_\ell)\|_2\\
   &\leq (\sqrt{w})^{L-\ell}(\|\mat W_\ell - \mat W'_\ell\|_F \|\gA_\ell(\vect x)\|_2 + \|\vect b_\ell - \vect b'_\ell\|_2)\\
   % &\leq ^{L-\ell} w^{(L-2)/2}(\sqrt{d} + \sqrt{w} + \ell-2)\epsilon + (\sqrt{w})^{L-\ell} \sqrt{w}\tilde\epsilon
   & \leq 2^{(\ell-1)} w^{(L+\ell-3)/2}d^{1/2} \epsilon + w^{(L-\ell)/2}\bar \epsilon
\end{eqal*}
% with the exception 
% \begin{eqal*}
%    |g_1(x) - g'_1(x)| \leq (\sqrt{w})^{L-1}\epsilon + (\sqrt{w})^{L-1} \sqrt{w}\tilde\epsilon\\
% \end{eqal*}

Let $f(x), f'(x)$ be two neural networks satisfying \autoref{eq:covcond} with parameters $W_1, b_1, \dots, W_L, b_L$ and $W'_1, b'_1, \dots, W'_L, b'_L$ respectively, and $\|W_\ell - W'_\ell\|_F \leq \epsilon_\ell, \|b_\ell - b'_\ell\|_F \leq \tilde\epsilon_\ell$. Further define $f_\ell$ be the neural network with parameters $W_1, b_1, \dots, W_\ell, b_\ell, W'_{\ell+1}, b'_{\ell+1}, \dots, W'_L, b'_L$, then 
\begin{eqal*}
   |f(x) - f'(x)| 
   &\leq |f(x) - f_1(x)| + |f_1(x) - f_2(x)| + \dots + |f_{L-1}(x) - f'(x)|\\
   &\leq \sum_{\ell=1}^{L} 2^{(\ell-2)}d^{1/2} w^{(L+\ell-3)/2} \epsilon + w^{(L-\ell)/2}\bar \epsilon
   %\sum_{\ell=1}^{L} w^{(L-2)/2}(\sqrt{d} + \sqrt{w} + \ell-2)\epsilon_\ell + (\sqrt{w})^{L-\ell} \sqrt{w}\tilde\epsilon_\ell
\end{eqal*}

For any $\delta > 0$, one can choose 
$$
   % \epsilon_1 = \frac{\delta}{2L (\sqrt{w})^{L-1}},
   % \epsilon_\ell = \frac{\delta}{2L w^{(L-2)/2}(\sqrt{d} + \sqrt{w} + \ell-2)},
   % \tilde\epsilon_\ell = \frac{\delta}{2L (\sqrt{w})^{L-\ell} \sqrt{w}}
   \epsilon_\ell = \frac{\delta}{2^\ell w^{(L+\ell-3)/2}d^{1/2}},
   \tilde\epsilon_\ell = \frac{\delta}{2w^{(L-\ell)/2}}
$$
such that $|f(x) - f'(x)| \leq \delta$.

% Because of that, if 
% \begin{eqal}
%    \|W_1-W_1'\|_F &\leq \frac{\epsilon}{Lm^{(L-1)/2}}\\
%    \|W_\ell-W_\ell'\|_F & \leq\frac{\epsilon}{Lm^{(L-2)/2}d^{1/2}}, & 2\leq \ell \leq L-1\\ 
%    \|W_L - W_L'\|_F & \leq\frac{\epsilon}{Lm^{(L-1)/2}d^{1/2}}
%    \label{eq:cneps}
% \end{eqal}
% then for any $x \in \R^{m'}, \|x\|_2 \leq 1$ and any $b_1, \dots, b_L$, 
% \begin{equation*}
%    |f(x; W_1b_1\dots W_Lb_L) - f(x; W'_1b_1\dots W'_Lb_L)| \leq \epsilon.
% \end{equation*}
On the other hand, the $\epsilon$-covering number of $\{\mat W \in \R^{m \times m'}: \|\mat W\|_F\leq \sqrt{m'} \}$ on Frobenius norm is no larger than $(2\sqrt{m'}/ \epsilon+1)^{m\times m'}$, and the $\bar\epsilon$-covering number of $\{\vect b \in \R^{m}: \|\vect b\|_2 \leq 1\}$ on infinity norm is no larger than $(2/\bar\epsilon+1)^m$. The entropy of this neural network can be bounded by 
\begin{eqal*}
   \log \gN(f; \delta) 
   & \leq w^2L\log(2^{L+1}w^{L-1}/\delta+1) 
   + wL\log(2^{L-1}w^{(L-1)/2}d^{1/2}/\delta+1)
   % &\leq dw \log (4dL\sqrt{w} (\sqrt{w})^{L-1}/\delta)
   % +\sum_{\ell=2}^{L} w^2\log (4L  w^{(L+1)/2}(\sqrt{d} + \sqrt{w} + \ell-2)/\delta)\\
   % &+ \sum_{\ell=1}^{L} w\log(4L (\sqrt{w})^{L-\ell} \sqrt{w}/\delta)
\end{eqal*}
\end{proof}

% \begin{proof}[Proof of \autoref{thm:pnncv}]

% Choosing $\|b^{\ell}\| \geq \sqrt{w}\|\gA_{\ell-1}(x)\|_2$ causes $\act(\vect W^{(\ell)}\cdot + b^{(\ell)})$ be a linear function in its feasible region. This also concludes that $\gA_L(\vect x) \leq (4w)^{L/2}$.

% \yw{It is not clear what the above is trying to do. What is a ``nontrivial model''? The conditions on $\|b^{\ell}\|$ did not appear in the theorem statement? We need to make it clearer.}

% Taking this result into \autoref{lemma:pnorm}, and noting that scaling of weights by $c_1$ results that the covering number is multiplied by a constant depending on $c_1$, one can easily finish the proof.
% \end{proof}

% \subsection{Proof of \autoref{lemma:pnorm}}
\subsection{Covering Number of $p$-Norm Constrained Linear Combination}
\label{sec:prooflemmapnorm}
\textbf{\autoref{lemma:pnorm}.}\textit{\lemmapnorm}

% \yw{In the lemma statement, you should specify what $M$ is or just say it is an arbitrary positive integer.}
\begin{proof}
Let $\epsilon$ be a positive constant. 
Without the loss of generality, we can sort the coefficients in descending order in terms of their absolute values. There exists a positive integer $\gM$ (as a function of $\epsilon$), such that $|a_i| \geq \epsilon$ for $i \leq \gM$, and $|a_i| < \epsilon$ for $i > \gM$. 
% \yw{$m$ is undefined. I think you mean ``Without the loss of generality, we can sort the coefficients in a descending order in terms of their absolute values. There exists $m$ (as a function of $\epsilon$), such that ...''.  Did you miss the absolute value or does the need to be symmetric around $0$?  Did you check this symmetry when applying Lemma 6?} 
% Obviously, $m \leq P/\epsilon^p$ and $|a_i| \leq M^{1/p}, \forall i$. \yw{Don't use ``obviously'', explain and prove the statement even if you think it is trivial.} 
% \yw{Maybe use a different symbol for $m$ to avoid the confusion w.r.t. the Besov space parameter $m$.}

By definition, $\gM \epsilon^p \leq \sum_{i=1}^\gM |a_i|^p \leq P$ so $\gM \leq P/\epsilon^p$, and $|a_i|^p \leq P, |a_i| \leq P^{1/p}$ for all $i$.
Furthermore,
\begin{eqal*}
   \sum_{i>m} |a_i| = \sum_{i>\gM} |a_i|^p |a_i|^{1-p} <  \sum_{i>\gM} |a_i|^p \epsilon^{1-p} \leq P \epsilon^{1-p}
\end{eqal*}
Let $\tilde g_i = \argmin_{g \in \tilde\gG}\|g - \frac{a_i}{P^{1/p}} g_i\|_\infty$ where $\tilde\gG$ is the $\delta'$-convering set of $\gG$. 
% \yw{Should it be $g$ and $g_i$ in the definition of $\tilde{g}_i$ instead of $f$ and $f_i$? } Because \yw{``because''?  Seems a broken sentence to me. Explain what you used to obtain the next inequality.}
By definition of the covering set,
\begin{eqal}
   \Bigg\|\sum_{i=1}^M a_i g_i(x) - \sum_{i=1}^\gM P^{1/p}\tilde g_i(x)\Bigg\|_\infty 
   %----------------------------
   &\leq \Bigg\|\sum_{i=1}^\gM (a_i g_i(x) - P^{1/p}\tilde g_i(x))\Bigg\|_\infty 
   + \Bigg\|\sum_{i=\gM+1}^{M} a_i g_i(x)\Bigg\|_\infty\\
   &\leq \gM P^{1/p}\delta' + c_3 P\epsilon^{1-p}.
   \label{eq:nndelta}
\end{eqal}
Choosing
\begin{equation}
\epsilon = ({\delta}/{2c_3P})^{\frac{1}{1-p}}, 
\delta' \eqsim P^{-\frac{1}{p(1-p)}}(\delta/2c_3)^{\frac{1}{1-p}}/2,
\label{eq:nncovpam}
\end{equation}
we have $\gM \leq  P^\frac{1}{1-p}(\delta/2c_3)^{-\frac{p}{1-p}}, \gM P^{1/p}\delta' \leq \delta/2, c_3P\epsilon^{1-p} \leq \delta/2$, so \autoref{eq:nndelta} $\leq \delta$.
% \yw{Explain which parameters you choose for this equation to hold. I think you a fixing $\delta$, then choose $\epsilon$, which specifies an $m$ for each function separately. Then $\delta'$ is chosen appropriately such that the equation is true.  In this logic, should there be a $\delta'$ independent to which function it is? There might be some issues here.} 
One can compute the covering number of $\gF$ by
\begin{equation}
   \log \gN(\gF, \delta) 
   \leq \gM\log \gN(\gG, \delta') 
   \lesssim k\gM\log(1/\delta')
   % -\gM\log\log(1/\delta').
   %\leq c'km \log(M^{1/p}m/\delta').
   % c' \Big(\frac{\delta'}{m}\Big)^{-km}.
   \label{eq:cnbound}
\end{equation} 
% The double logarithmic term is omitted here. \yw{Don't omit terms. You may hide them in a big $\tilde{O}$ notation or define $\lesssim$ to hide log-log factors.} 
% One can to satisfy the requirement.
Taking \autoref{eq:nncovpam} into \autoref{eq:cnbound} finishes the proof.
% \yw{``Taking them''?  Do you mean you would like to plug them into \autoref{eq:cnbound} and check that the inequality   \autoref{eq:cnbound}  is valid? The logic is a bit convoluted here.}
\end{proof}
   
   % The entropy of the set $\{\sum_{i=1}^m a_i g_i(x)\}$ can be bounded by 
   % \begin{equation*}
   %    \log \gN(\{\sum_{i=1}^m a_i g_i\}, \delta') 
   %    \leq m\log \gN(\{a_1 g_1\}, \delta'/m) 
   %    \leq c'km \log(M^{1/p}m/\delta')
   %    % c' \Big(\frac{\delta'}{m}\Big)^{-km}.
   % \end{equation*} 
   % On the other hand,
  
   % Let $f = \sum_{i=1}^K a_i g_i(x)\in \gF$ , Letting $M\epsilon^{1-p} = \delta/2, \delta' = \delta/2$, 

\section{Proof of Approximation Error}

\def\lemmatpappb{
   % There exists a neural network with $d$-dimensional input and one output, with width $w_{d,m} \eqsim dm$ and depth $L \lesssim \log(c_{d,m}/\epsilon)$ for some constant $w_{d,m}, c_{d,m}$ that depends only on $m$ and $d$, denoted as $\tilde{M}_{m,k, \vect s}(\vect x), \vect x \in \R^d$,  approximates the B spline basis function $M_{m, k, \vect s} (\vect x):= M_m(2^k (\vect x - \vect s))$ as defined in \autoref{sec:besov}, and it satisfies
   % \begin{itemize}
   %    \item $|\tilde M_{m, k, \vect s}(\vect x) - M_{m, k, \vect s}(\vect x)| \leq \epsilon$, if $ 0 \leq 2^k (x_i - s_i) \leq m+1, \forall i \in [d]$,
   %    \item $\tilde M_{m, k, \vect s}(\vect x) = 0$, otherwise.
   %    \item The weight in each layer has bounded norm $\|\mat W^{(\ell)}\|_F \lesssim 2^{k/L}\sqrt{w}$ except the first layer where $\|\mat W^{(1)}\|_F \leq 2^{k/L}\sqrt{d}$.
   % \end{itemize}
   {Let $M_{m, k,\vect s}$ be the B-spline of order $m$ with scale $2^{-k}$ in each dimension and position $\vect s \in \R^d$: $M_{m, k, \vect s}(\vect x):=M_m(2^k(\vect x - \vect s))$, $M_m$ is defined in \autoref{eq:bspline}.}
   There exists a neural network with $d$-dimensional input and one output, with width $w_{d,m} = O(dm)$ and depth $L \lesssim \log(c_{d,m}/\epsilon)$ for some constant  $c_{d,m}$ that depends only on $m$ and $d$,  approximates the B spline basis function $M_{m, k, \vect s} (\vect x):= M_m(2^k (\vect x - \vect s))$ as defined in \autoref{sec:besov}. 
   This neural network, denoted as $\tilde{M}_{m,k, \vect s}(\vect x), \vect x \in \R^d$, satisfy

   % There exists a parallel neural network that has the structure and satisfy the constraint in \autoref{prop:eqmodel} for $d$-dimensional input and one output, containing $M = O(m^d)$ subnetworks, each of which has width  $w = O(d)$ and depth $L = O(\log(c(m,d)/\epsilon))$ for some constant $w, c$ that depends only on $m$ and $d$, denoted as $\tilde{M}_{m}(\vect x), \vect x \in \R^d$,  such that 
   \begin{itemize}
      \item $|\tilde M_{m, k, \vect s}(\vect x) - M_{m, k, \vect s}(\vect x)| \leq \epsilon$, if $ 0 \leq 2^k (x_i - s_i) \leq m+1, \forall i \in [d]$,
      \item $\tilde M_{m, k, \vect s}(\vect x) = 0$, otherwise.
      % \item The weights in the last layer satisfy $\|a\|_{2/L}^{2/L} \lesssim 2^{k} m^d e^{2md/L}$.
      \item The weight in each layer has bounded norm $\|\mat W^{(\ell)}\|_F \lesssim 2^{k/L}\sqrt{w}$,
      except the first layer where $\|\mat W^{(1)}\|_F \leq 2^{k/L}\sqrt{d}$.
   \end{itemize}
}
% \subsection{Proof of \autoref{lemma:tpapp2}}
\subsection{Approximation of Neural Networks to  B-spline Basis Functions}
\label{sec:prooflemmatpapp}
\begin{lemma}
   \label{lemma:tpapp2}
   \lemmatpappb
\end{lemma}

% \yw{In the above the $M_{m}$ function is not defined until later in the proof. Define it when you first introduce Besov space and B-spline basis earlier.  Also, why are you calling this a lemma but Proposition 12 a proposition? Seems like the opposite would be better, or calling both propositions would be better.}

% The proof is based on the fact that a deep neural network can approximate multiply function to arbitrary precision, which was proven in \citet[Proposition 3]{yarotsky2017error}, \citet[Section A]{suzuki2018adaptivity}.

Note that the product of the coefficients among all the layers are proportional to $2^{k}$, instead of $2^{km}$ when approximating truncated power basis functions. This is because the transformation from $M_m$ to $M_{m, k, \vect s}$ only scales the domain of the function by $2^k$, while the codomain of the function is not changed. 
To apply the transformation to the neural network, one only need to scale weights in the first layer by $2^k$, which is equivalent to scaling the weights in each layer bt $2^{k/L}$ and adjusting the bias according.

As for the proof, we follow the method developed in \citet{yarotsky2017error, suzuki2018adaptivity}, while putting our attention on bounding the Frobenius norm of the weights.
\begin{lemma}[{\citet[Proposition 3]{yarotsky2017error}}]:
   \label{lemma:mulapp}
   There exists a neural network with two-dimensional input and one output $f_\times(x, y)$, with constant width and depth $O(\log(1/\delta))$, and the weight in each layer is bounded by a global constant $c_1$, such that 
   \begin{itemize}
      \item $|f_\times(x, y) - xy| \leq \delta, \forall\ 0 \leq x, y \leq 1$,
      \item $f_\times(x, y) = 0, \forall\ x=0$ or $y=0$.
   \end{itemize}
\end{lemma}
% \yw{Each lemma needs a proof or a citation.  Add a citation to a specific lemma (in Yarotsky?) for the above.}

We first prove a special case of \autoref{lemma:tpapp2} on the unscaled, unshifted B-spline basis function by fixing $k=0,\vect s=0$:
%We first prove this statement without scaling and shifting by fixing $k=0, \vect s=0$:

\def\proptpapp{
   \modif{
   There exists a neural network with $d$-dimensional input and one output, with width $w =w(d, m) \eqsim dm$ and depth $L \lesssim \log(c(m,d)/\epsilon)$ for some constant $w, c$ that depends only on $m$ and $d$, denoted as $\tilde{M}_{m}(\vect x), \vect x \in \R^d$,}  such that 
   \begin{itemize}
      \item $|\tilde M_m(\vect x) - M_{m}(\vect x)| \leq \epsilon$, if $ 0 \leq x_i \leq m+1, \forall i \in [d]$, while $M_m(\cdot)$ denote $m$-th order B-spline basis function,
      \item $\tilde M_m(\vect x) = 0$, if $ x_i \leq 0$ or $x_i \geq m+1$ for any $i \in [d]$.
      \item The weight in each layer has bounded norm $\|\mat W^{(\ell)}\|_F \lesssim \sqrt{w}$.
      % except the first layer where $\|\mat W^{(1)}\|_F \leq \sqrt{d}$.
   \end{itemize}

   % There exists a parallel neural network that has the structure and satisfy the constraint in \autoref{prop:eqmodel} for $d$-dimensional input and one output, containing $M = \lceil(m+1)/2\rceil^d = O(m^d)$ subnetworks, each of which has width  $w = O(d)$ and depth $L = O(\log(c(m,d)/\epsilon))$ for some constant $w, c$ that depends only on $m$ and $d$, denoted as $\tilde{M}_{m}(\vect x), \vect x \in \R^d$,  such that 
   % \begin{itemize}
   %    \item $|\tilde M_m(\vect x) - M_{m}(\vect x)| \leq \epsilon$, if $ 0 \leq x_i \leq m+1, \forall i \in [d]$, while $M_m(\cdot)$ denote $m$-th order B-spline basis function, and $c$ only depends on $m$ and $d$.
   %    \item $\tilde M_m(\vect x) = 0$, if $ x_i \leq 0$ or $x_i \geq m+1$ for any $i \in [d]$.
   %    \item The weights in the last layer satisfy $\|a\|_{2/L}^{2/L} \lesssim m^d e^{2md/L}$.
   %    % \item The weight in each layer has bounded norm $\|\mat W^{(\ell)}\|_F \lesssim \sqrt{w}$.
   %    % except the first layer where $\|\mat W^{(1)}\|_F \leq \sqrt{d}$.
   % \end{itemize}
} 

\begin{proposition}
   \label{prop:tpapp}
   \proptpapp
\end{proposition}

\begin{proof}
We first show that one can use a neural network with constant width $w_0$, depth $L\eqsim \log(m/\epsilon_1)$ and bounded norm $\|W^{(1)}\|_F \leq O(\sqrt{d}), \|W^{(\ell)}\|_F \leq O(\sqrt{w}), \forall \ell=2, \dots, L$ to approximate truncated power basis function up to accuracy $\epsilon_1$ in the range $[0, 1]$.
Let $m = \sum_{i=0}^{\lceil\log_2 m\rceil}m_i 2^i, m_i \in \{0, 1\}$ be the binary digits of $m$, and define $\bar m_j = \sum_{j=0}^i m_i, \gamma = \lceil\log_2 m\rceil$, then for any $x$ 
\begin{eqal}
   x_+^m &= x_+^{\bar m_\gamma} \times \big(x_+^{2^\gamma}\big)^{m_\gamma}\\
   %-----------------------
   [x_+^{\bar m_\gamma}, x_+^{ 2^\gamma}] &= [x_+^{\bar m_{\gamma-1}} \times \big(x_+^{2^{\gamma-1}}\big)^{m_{\gamma-1}}, x_+^{2^{\gamma-1}}\times x_+^{2^{\gamma-1}}]\\
   &\dots \\
   [x_+^{\bar m_2}, x_+^{4}] &= 
   [x_+^{\bar m_{1}} \times \big(x_+^{2}\big)^{m_1}, x_+^{2}\times x_+^{2}]\\
   [x_+^{\bar m_1}, x_+^2] &= [x_+^{\bar m_{0}} \times x_+^{m_0}, x_+\times x_+]
   \label{eq:recmult}
\end{eqal}
Notice that each line of equation only depends on the line immediately below. Replacing the multiply operator $\times$ with the neural network approximation shown in \autoref{lemma:mulapp} demonstrates the architecture of such neural network approximation.
For any $x, y \in [0, 1]$, let $|f_\times (x, y) - xy| \leq \delta, |x-\tilde x| \leq \delta_1, |y-\delta y| \leq \delta_2$, then $|f_\times (\tilde x, \tilde y) - xy| \leq \delta_1 + \delta_2 + \delta$. Taking this into \autoref{eq:recmult} shows that $\epsilon_1 \eqsim 2^\gamma \delta \eqsim m\delta$, where $\epsilon_1$ is the upper bound on the approximate error to truncated power basis of order $m$ and $\delta$ is the approximation error to a single multiply operator as in \autoref{lemma:mulapp}.

% A similar theorem can be found in \citet[Proposition 1]{suzuki2018adaptivity}. The difference is that \citet[Proposition 1]{suzuki2018adaptivity} provides a bound in the infinity norm of the coefficients while this paper provides a bound in the Frobenius norm.

A univariate B-spline basis can be expressed using truncated power basis, and observing that it is symmetric around $(m+1)/2$:
\begin{eqal*}
   M_m(x) 
   &= \frac{1}{m!} \sum_{j=1}^{m+1} (-1)^j {m+1 \choose j}(x-j)^m_+\\
   &= \frac{1}{m!} \sum_{j=1}^{\lceil(m+1)/2\rceil} (-1)^j {m+1 \choose j}(\min(x, m+1-x)-j)^m_+\\
   & = \frac{((m+1)/2)^m}{m!} \sum_{j=1}^{\lceil(m+1)/2\rceil} (-1)^j {m+1 \choose j}\Big(\frac{\min(x, m+1-x)-j}{(m+1)/2}\Big)^m_+,
\end{eqal*}

A multivariate ($d$-dimensional) B-spline basis function can be expressed as the product of truncated power basis functions and thus can be decomposed as 
\begin{eqal}
   M_m(\vect x)
   &= \prod_{i=1}^d M_m(x_i)\\
   &= \frac{((m+1)/2)^{md}}{(m!)^d} \prod_{i=1}^d \Biggl( \sum_{j=1}^{\lceil(m+1)/2\rceil} (-1)^j {m+1 \choose j}\Big(\frac{\min(x_i, m+1-x)-j}{(m+1)/2}\Big)^m_+ \Biggr)\\
   % &= \frac{((m+1)/2)^{md}}{(m!)^d}  \sum_{j_1, \dots, j_d=1}^{\lceil(m+1)/2\rceil}  \prod_{i=1}^d  (-1)^{j_i} {m+1 \choose j_i}\Big(\frac{\min(x, m+1-x)-j_i}{(m+1)/2}\Big)^m_+ \\
   \label{eq:bsplinepoly}
\end{eqal}
Using \autoref{lemma:mulapp}, one can construct $m + 1$ number of neural networks, and each of them has width
$w_0$ and depth $L = O(\log (m/\epsilon_1)$, such that the $(j + 1)$-th neural network approximates $(\frac{x-j}{(m+1)/2} )^m_+$ with
error no more than $\epsilon_1$ for any $0 \leq x \leq (m + 1)/2$. The weighted summation of these subnetworks can approximate the univariate B-spline basis function with error no more than
% one can construct a parallel neural network containing $M = \lceil(m+1)/2\rceil^d = O(m^d)$ subnetworks, and each subnetwork corresponds to one polynomial term in \autoref{eq:bsplinepoly}. Using the results above, the approximation of this constructed neural network can be bounded by 
\begin{eqal*}
   d((m+1)/2)^m \frac{1}{m!}\sum_{i=1}^{m+1} {m+1\choose j} \epsilon_1 \eqsim \frac{de^{2m}}{\sqrt{m}}\epsilon_1
\end{eqal*}
where we applied Stirling's approximation.

%and we only need to approximate the truncated power basis functions with $j < (m+1)/2$.
% Using this result, one can construct $m+1$ number of neural networks, and each of them has width $w_0$ and depth $L\eqsim \log(m/\epsilon_1)$, such that the $(j+1)$-th neural network approximates $(\frac{x-j}{(m+1)/2})^m_+$ with error no more than $\epsilon_1$ for any $0 \leq x\leq (m+1)/2$.
% The weighted summation of these subnetworks can approximate the univariate B-spline basis function with error no more than 

% where we use Stirling's appropximation to get the last `'$\eqsim$'.
% The approximation error analysis above shows that a neural network with width $w$

A multivariate B-spline basis is the product of univariate B-spline basis along each dimension 
\begin{eqal*}
   M_m(\vect x) = \prod_{i=1}^d M_m(x_i).
\end{eqal*}

We can construct a neural network to approximate this function by parallizing $d$ number of neural networks to approximate each B-spline basis function along each dimension, and use the last $L_1 \eqsim \log(d/\delta)$ layers to approximate their product.
The totol approximation error of this function is bounded by 

% \begin{eqal*}
%    % \Bigg(\sum_{i=1}^{m+1} {m+1\choose j} d(\epsilon_1 + \delta)\Bigg)^d \lesssim \frac{e^{2m}}{\sqrt{m}}d\epsilon_1 + d\delta
%    % \frac{((m+1)/2)^{md}}{(m!)^d}  \sum_{j_1, \dots, j_d=1}^{\lceil(m+1)/2\rceil}  \prod_{i=1}^d  (-1)^{j_i} {m+1 \choose j_i}\epsilon_1 \lesssim e^{md} \epsilon_1
% \end{eqal*}

% % This can be achieved by parallizing $d(m+1)$ number of neural networks with width $w_0$ and depth $L$, and use a neural network to appearance  multiply operator \yw{to ``approximate `multiply' operator''? The reference to Lemma~\ref{lemma:mulapp} is needed? Also, is it easy to see how the approximation error propagates as you chain together multiple operators? Some generic lemmas that says how the approximation errors are preserved upon adding / multiplication, concatenation and chaining would suffice.  btw, the multiplication lemma requires the input to be between $[0,1]$ for this $\delta$-approximation to hold. You may need to check that and carefully work out the parameters when applying that lemma.} after that, resulting in \yw{there seems to be some grammar issues here. break into two sentences so you can write clearly.} a single one with width $w = w(d, m)$ and depth $L'= L(d, m, \epsilon) + O(\log(d/\epsilon)), \|\mat W^{(\ell)}\|_F^2 \lesssim w$. 
% The totol approximation error of this function is bounded by 
\begin{eqal*}
   d\frac{((m+1)/2)^m}{m!}\sum_{j=1}^{m+1} {m+1 \choose j} \epsilon_1 + (d-1)\delta \eqsim \frac{e^{2m}}{\sqrt{m}}d\epsilon_1 + d\delta
\end{eqal*}
where $\delta$ and $\epsilon_1$ has the same definition as above. Choosing $\delta= \frac{\epsilon}{d(e^{2m}\sqrt{m}+1)}$, and recall $\epsilon_1 \eqsim m\delta$ proves the approximation error.

% To bound the norm of the factors $\|a\|_{2/L}^{2/L}$, first observe that 
% \begin{eqal*}
%    |a_{j_1, \dots, j_d}| 
%    &= \frac{((m+1)/2)^{md}}{(m!)^d} \frac{1}{(m+1)/2}\prod_{i=1}^d  {m+1 \choose j_i}\\
%    &\leq \frac{((m+1)/2)^{md}}{(m!)^d} \frac{2^{md}}{(m+1)/2} = O(e^{md})
% \end{eqal*}
% where the first inequality is from $ {m+1 \choose j_i} \leq 2^{m+1}$, the last equality is from Stirling's appropximation. 
% Finally, 
% \begin{eqal*}
%    \|a\|_{2/L}^{2/L} \leq m^d \max_j |a_j|^{2/L} \lesssim m^d e^{2md/L} 
% \end{eqal*}
% which finishes the proof.
\end{proof}

The proof of the \autoref{lemma:tpapp2} for general $k,\vect s$ follows by appending one more layer in the front, as we show below.
\begin{proof}[Proof of \autoref{lemma:tpapp2}]
Using the neural network proposed in \autoref{prop:tpapp}, 
one can construct a neural network for appropximating $M_{m,k,\vect s}$ by adding one layer before the first layer:
\begin{eqal*}
   \act(2^k\mat I_d \vect x - 2^k\vect s)
\end{eqal*}
The unused neurons in the first hidden layer is zero padded.
The Frobenius norm of the weight is $2^k \|\mat I_d\|_F = 2^k\sqrt{d}$.
% by multiplying the first layer in the neural network in \autoref{prop:tpapp} by $2^k$, and adjusting the bias accordingly. 
% The Frobenius norm of the first layer is $O(2^k \sqrt{d})$ instead of $O(\sqrt{d})$. 
Following the proof of \autoref{prop:eqmodel}, rescaling the weight in this layer by $ 2^{-k}$, and the weight matrix 
in the last layer by $2^k$,
% all the other layers by $2^{k/L}$ 
and scaling the bias properly, one can verify that this neural network satisfy the statement.
%minimize the total Frobenius norm while keeping the model equivalent finishes the proof.
% \yw{Don't use ``obviously''. Explain how it works.}
\end{proof}

% \section{Besov Space}
% \subsection{Proof of \autoref{prop:bapp} \yw{``Sparse approximation of Besov functions using B-spline wavelets''}}
\subsection{Sparse approximation of Besov functions using B-spline wavelets}
% \yw{I think it will make it clearer if you add subtitles and refer to these lemmas / propositions by a name. For example, I added one above.}

\textbf{\autoref{prop:bapp}.}\textit{\propbapp}
\label{sec:proofpropbapp}
% \yw{$\bar{M}$ is undefined above. Do you mean $\bar{M}$ is any integer with no constraints?}
% \yw{Should it read ``any function in Besov space $f\in B_{p,q}^\alpha$'' rather than $f_0$, or the $f$ in the statements should change to $f_0$?  also, the statement also reads a bit unclear.  I think you mean the following (my changes are in red). Please fix any thing if you find them to be incorrect due to my edits.
% }
% \begin{proposition}
%       Let $\alpha - d/p > 1, r > 0$. Let $M_{m, k,\vect s}$ be the B-spline of order $m$ with scale $2^{-k}$ in each dimension and position $\vect s \in \R^d$. {\color{red}For any function in Besov space $f_0 \in B^{\alpha}_{p, q}$ and any positive integer $\bar M$, there is an $\bar{M}$-sparse approximation} using B-spline basis of order $m$ satisfying $0 < \alpha < \min(m, m - 1 + 1/p)$: 
%          $
%       \check f_{\bar M} = \sum_{i=1}^{\bar M} a_{k_i,\vect s_i}M_{m, k_i,\vect s_i}
%    $ such that the approximation error
%    $
%       \|\check f_{\color{red}\bar{M}} - {\color{red}f_0}\|_r \lesssim {\bar M}^{-\alpha /d}\|f\|_{B^{\alpha}_{p, q}},
%    $
%    and the coefficients satisfy 
%    $$
%       \|\{2^{k_i} {\color{red}a_{k_i,\vect s_i}}\}_{k_i, \vect s_i}\|_p \lesssim \|f\|_{B^{\alpha}_{p, q}}.
%    $$
% \end{proposition}

\begin{remark}
   The requirement in \autoref{prop:bapp}: $\alpha - d/p>1$ is stronger than the condition typically found in approximation theorem $\alpha - d/p \geq 0$ \citep{dung2011optimal}, so-called ``Boundary of continuity'', or the condition in \citet{suzuki2018adaptivity} $\alpha > d(1/p-1/r)_+$ . This is because although the functions in $B^{\alpha}_{p,q}$ when $0 \leq \alpha - d/p < 1$ can be approximated by B-spline basis, the sum of weighted coefficients may not converge. One simple example is the step function 
   $f_{\textrm{step}}(x) = \mathbf{1}(x \geq 0.5), f_{\textrm{step}} \in B^{1}_{1, \infty}$.
   % \begin{empheq}[left={f(x)=\empheqlbrace}]{align*}
   %    & 0 & x < 1/2\\
   %    & 1 & x \geq 1/2
   % \end{empheq}
   % which is in $B^{1}_{1, \infty}$. 
   Although it can be decomposed using first order B-spline basis as in \autoref{eq:besovdecom}, the summation of the coefficients is infinite. Actually one only needs a ReLU neural network with one hidden layer and two neurons to approximate this function to arbitrary precision, but the weight need to go to infinity.
\end{remark}

\begin{proof}

\citet[Theorem 3.1]{dung2011optimal} \citet[Lemma 2]{suzuki2018adaptivity} proposed an adaptive sampling recovery method that approximates a function in Besov space.
The method is divided into two cases: when $p \geq r$, and when $p < r$.

% \yw{Is there a particular theorem you can cite about the following statement? If the same is used by Suzuki, you can cite Suzuki}. 
When $p \geq r$, there exists a sequence of scalars $\lambda_{\vect j}, \vect j \in P^d(\mu), P_d(\mu) := \{\vect j \in \Z^d: |j_i| \leq \mu, \forall i \in d\}$ for some positive $\mu$, for arbitrary positive integer $\bar k$, the linear operator 
% \yw{It will be a bit confusing to call it a linear estimator or even an estimator.  There is no data involved.}
\begin{eqal*}
   Q_{\bar k}(f,\vect x) &= \sum_{\vect s \in J(\bar k, m, d)}a_{\bar k,\vect s}(f)M_{\bar k,\vect s}(\vect x),&
   a_{\bar k,\vect s}(f) &= \sum_{\vect j \in \Z^d, P^d(\mu)} \lambda_{\vect j} \bar f(\vect s + 2^{-\bar k}\vect j)
\end{eqal*}
has bounded approximation error 
\begin{equation*}
   \|f - Q_{\bar k}(f,x)\|_r \leq C2^{-\alpha \bar k} \|f\|_{B^\alpha_{p,q}},
\end{equation*}
where $\bar f$ is the extrapolation 
% \yw{What is an extension? Is it defined anywhere?} 
of $f$, $J(\bar k, m, d) := \{\vect s: 2^{\bar k}\vect s \in \Z^d, -m/2 \leq 2^{\bar k} s_i \leq 2^{\bar k} + m/2, \forall i \in [d]\}$. See \citet[2.6-2.7]{dung2011optimal} for the detail of the extrapolation as well as references for options of sequence $\lambda_{\vect j}$.

% \yw{There seem to be a lot of clarity issues. What is $\bar{k}$? what is $J(\bar{k})$? What is $\bar{f}$? what is $\vect j$?  is $\lambda_{\vect j}$ a vector?  }

Furthermore, $Q_{\bar k}(f) \in B^\alpha_{p,q}$ 
so it can be decomposed in the form \autoref{eq:besovdecom} with $M=\sum_{k = 0}^{\bar k} (2^{k}+m-1)^d \lesssim 2^{\bar kd}$ components and 
$\|\{\tilde c_{k,\vect s}\}_{k,\vect s}\| \lesssim \|Q_{\bar k}(f)\|_{B^\alpha_{p,q}} \lesssim \|f\|_{B^\alpha_{p,q}}$ 
where $\tilde c_{k, \vect s}$ is the coefficients of the decomposition of $Q_{\bar k}(f)$. Choosing $\bar k \eqsim \log_{2} M /d$ leads to the desired approximation error. 

On the other hand, when $p < r$, \yws{I think you could start by saying that ``Dung has two different approximation for the case when $p\geq r$ and the case when $p<r$'' before you break it down.} there exists a greedy algorithm that constructs 
\begin{eqal*}
   G(f) = Q_{\bar k}(f) + \sum_{k=\bar k + 1}^{k^*} \sum_{j=1}^{n_k} c_{k, \vect s_j}(f) M_{k, \vect s_j}
\end{eqal*}
where $\bar k \eqsim \log_2(M), k^*=[\epsilon^{-1}\log(\lambda M)] + \bar k + 1, n_k = [\lambda M 2^{-\epsilon(k-\bar k)}]$ for some $0 < \epsilon < \alpha/\delta - 1, \delta = d(1/p-1/r), \lambda>0$, such that 
\begin{eqal*}
   \|f - G(f)\|_r \leq {\bar M}^{-\alpha/d} \|f\|_{B^\alpha_{p,q}}
\end{eqal*}
and 
\begin{eqal*}
   \sum_{k=0}^{\bar k}(2^k + m-1)^d + \sum_{k=\bar k + 1}^{k^*} n_k \leq {\bar M}.
\end{eqal*}
See \citet[Theorem 3.1]{dung2011optimal} for the detail.

Finally, since $\alpha - d/p > 1$,
\begin{eqal}
   \|\{2^{k_i} c_{k_i,\vect s_i}\}_{k_i, \vect s_i}\|_p 
   &\leq \sum_{k=0}^{\bar k} 2^{k}\|\{c_{k_i,\vect s_i}\}_{\vect s_i}\|_p \\
   &= \sum_{k=0}^{\bar k} 2^{(1-(\alpha-d/p))k} (2^{(\alpha-d/p)k}\|\{c_{k_i,\vect s_i}\}_{\vect s_i}\|_p) \\
   &\lesssim \sum_{k=0}^{\bar k} 2^{(1-(\alpha-d/p))k} \|f\|_{B^\alpha_{p,q}} \\
   &\eqsim \|f\|_{B^\alpha_{p,q}} 
   \label{eq:seqnorm}
\end{eqal}
where the first line is because for arbitrary vectors $\vect a_i, i \in [n], \|\sum_{i=1}^n \vect a_i\|_p \leq \sum_{i=1}^n \| \vect a_i\|_p $, the third line is because the sequence norm of B-spline decomposition is equivalent to the norm in Besov space (see \autoref{sec:besov}) \yws{Cite a theorem in a paper for this.}.
\end{proof}

Note that \yws{drop ``would'' or drop ``We would note that'' all together} when $\alpha -d/p=1$, \yws{whether it is ``easy'' is irrelevant. how you prove it should be presented, e.g., ``By Lemma XXX'', or ``By substituting X into equation Y''. don't leave it for the readers to guess. } the sequence norm \autoref{eq:seqnorm} is bounded (up to a factor of constant) by $k^*\|f\|_{B^\alpha_{p,q}}$, which can be proven by following \autoref{eq:seqnorm} except the last line. 
This adds a logarithmic term with respect to ${\bar M}$ compared with the result in \autoref{prop:bapp}. 
This will add a logarithmic factor to the MSE. 
We will not focus on this case in this paper of simplicity. 
%This  \yw{``should''? you are writing a proof and you need to be certain.} lead to the same MSE up to a factor of logarithmic \yw{``a logarithmic factor''}.
% \yw{The above paragraph is discussion, rather than part of the proof? If so, move it outside the proof block.}

% \subsection{Proof of \autoref{thm:pnormapp}}
\subsection{Sparse approximation of Besov functions using Parallel Neural Networks}
\label{sec:proofthmpnormapp}
\textbf{\autoref{thm:pnormapp}.}\textit{\thmpnormapp}

% \yw{In the above, say that it holds for any integer $\bar{M}> 0$? Or if there are any assumptions on $\bar{M}$ it should be stated.}

The proof is divided into three steps:
\begin{enumerate}
   \item Bound the 0-norm and the $p$-norm of the coefficients of B-spline basis in order to approximate an arbitrary function in Besov space up to any $\epsilon > 0$.
   \item Bound $p'$-norm of the coefficients of B-spline basis functions where $p' = 2/L, 0 < p' < 1$ using the results above \label{item:errorp}.
   \item Add the approximation of neural network to B-spline basis computed in \autoref{lemma:tpapp2} into Step~\ref{item:errorp}.
\end{enumerate}
  
We first prove the following lemma.
\begin{lemma}
   % For any $a \in \R^{\bar M}, \|a\|_p \leq 1$, for any $0<p'<1$,
   % %  \yw{should it be $r\geq 1$ or $r<1$?}
   % its $p'$-norm is bounded by 
   %  \begin{equation*}
   %     \|a\|_{p'}^{p'} \leq {\bar M}^{1-p'/p}
   %  \end{equation*}
   %  \yw{Modified statement for the simplified proof below.}
    {For any $a \in \R^{\bar M}$, $0<p'<p$, it holds  that:
    $$
    \|a\|_{p'}^{p'} \leq {\bar M}^{1-p'/p} \|a\|_{p}^{p'}.
    $$
    }
    
    \label{lemma:pnorm2}
 \end{lemma} 
 \begin{proof}

{
   % Assume $0<p'<p$:
   \begin{eqal*}
   \sum_{i}|a_i|^{p'} = \langle \mathbf{1}, |\vect a|^{p'} \rangle &\leq \left(\sum_{i} 1\right)^{1-\frac{p'}{p}} \left(\sum_i(|a_i|^{p'})^{\frac{p}{p'}}\right)^{\frac{p'}{p}}
   = \bar{M}^{1-\frac{p'}{p}} \|a\|_p^{p'}
   \end{eqal*}
   The first inequality uses a Holder's inequality with conjugate pair $\frac{p}{p'}$ and $1/(1-\frac{p'}{p})$.
   %  and a decomposition of $\sum_{i}|a_i|^{p'} = \langle \mathbf{1}, |\vect a|^{p'} \rangle $.
}
\end{proof}
\begin{proof}[Proof of \autoref{thm:pnormapp}]
Using \autoref{prop:bapp}, one can construct $\bar M$ number of NN according to \autoref{lemma:tpapp2}, such that each NN represents one B-spline basis function. The weights in the last layer of each NN is scaled to match the coefficients in \autoref{prop:bapp}.
Taking $p'$ in \autoref{lemma:pnorm2} as $2/L$ and combining with \autoref{lemma:tpapp2} finishes the proof.
\end{proof}
\section{Proof of the Main Theorem}
% \subsection{Proof of \autoref{thm:main}}
\label{sec:proofthmmain} 
\textbf{\autoref{thm:main} extended form.}\textit{\thmmain}

\begin{proof}
First recall the relationship between covering number (entropy) and estimation error:
\begin{proposition}
   \label{prop:msecov}
   % \textcolor{red}{TODO:}
   % \citep[Proposition 4]{suzuki2018adaptivity}. 
   Let $\gF\subseteq \{\R^d \rightarrow [-F, F]\}$ be a set of functions. 
   Assume that $\gF$ can be decomposed into two orthogonal spaces $\gF=\gF_\parallel \times \gF_\bot$ where $\gF_\bot$ is an affine space with dimension of N. 
   % In other word, any $f \in \gF$ can be decomposed into two orthogonal components $f = f_\bot + f_\parallel, \sum_{i=1}^n (f_\bot(\vect x_i) f_\parallel(\vect x_i))^2=0$, where $f_\bot \in \gF_\bot, f_\parallel \in \gF_\parallel$. Furthermore, for any $f_\bot \in \gF_\bot, f_\parallel \in \gF_\parallel, f_\bot + f_\parallel \in \gF$.
   % it can be decomposed into two linear function spaces: $\gH\subseteq \gF$ with dimension $N$ and $\gF_\parallel \subseteq \gF$ be its orthogonal complement. 
   Let $f_0 \in \{\R^d \rightarrow [-F, F]\}$ be the target function and $\hat f$ be the least squares estimator in $\gF$:
   \begin{eqal*}
   \hat f = \argmin_{f \in \gF} \sum_{i=1}^n (y_i - f(x_i))^2, 
   y_i = f_0(x_i) + \epsilon_i, \epsilon_i \sim \gN(0, \sigma^2) i.i.d.,
   \end{eqal*}
   % If the covering number $\gN(\gF, \delta, \|\cdot\|_\infty) \geq 3$, 
   then it holds that
   \begin{eqal*}
      % \E_{\gD_n}[\|\hat f - f_0\|^2_{L^2(P_x)}] \leq C\Bigl[\inf_{f \in \gF} \|f - f_0\|^2_{L^2(P_x)} + (F^2 + \sigma^2)\frac{\log \gN(\gF, \delta, \|\cdot\|_\infty)}{n} + \delta(F + \sigma)\Bigr]
      \MSE(\hat f) \leq \tilde O\Big(\argmin_{f \in \gF} \MSE(f) + \frac{N + \log\gN(\gF_\parallel, \delta)+2}{n} + (F+\sigma)\delta \Big) .
   \end{eqal*}
   % where $C$ is a universal constant.
\end{proposition}

The proof of \autoref{prop:msecov} is defered to the section below.
We choose $\gF$ as the set of functions that can be represented by a parallel neural network as stated, the (null) space $\gF_\bot=\{f: f(\vect x) = constant\}$ be the set of functions with constant output, which has dimension 1.
This space captures the bias in the last layer, while the other parameters contributes to the projection in $\gF_\parallel$.
See \autoref{sec:proofthmpnncv} for how we handle the bias in the other layers.
One can find that $\gF_\parallel$ is the set of functions that can be represented by a parallel neural network as stated, and further satisfy 
$
   \sum_{i=1}^n f(\vect x_i) = 0.
$
Because $\gF_\parallel \subseteq \gF$, $\gN(\gF_\parallel, \delta) \leq \gN(\gF, \delta)$ for all $\delta>0$, and the latter is studied in \autoref{thm:pnncv}.

% while the functions defined by the other parameters are studied in \autoref{thm:pnncv}.

In \autoref{thm:main}, the width of each subnetwork is no less than what is required in \autoref{thm:pnormapp}, while the depth and norm constraint are the same, so the approximation error is no more that that in \autoref{thm:pnormapp}.
Choosing $r=2, p = 2/L$, and taking \autoref{thm:pnncv} and \autoref{thm:pnormapp} into this \autoref{prop:msecov}, one gets
\begin{eqal}
   \label{eq:errp}
   % \E_{\gD_n}[\|\hat f - f_0\|^2_{L^2(P_x)}] 
   % \MSE(\hat f)
   % \lesssim {\bar M}^{-2\alpha/d} + \frac{\log n}{n} \rmm{w^2L} {{\bar M}^{\frac{1-2/(pL)}{1-2/L}}\delta^{-\frac{2/L}{1-2/L}} (\log ({\bar M}/\delta)} + 3) + \delta,
   %------------------------------
   \MSE(\hat f)
   &\lesssim \min_{f \in \gF} \MSE(f) + \frac{w^{2+2/(1-2/L)} L^2 \sqrt{d} P'^{\frac{1}{1-2/L}}\delta^{-\frac{2/L}{1-2/L}} \log (wP'/\delta)}{n} + \delta\\
   &\lesssim {\bar M}^{-2\alpha/d} + \frac{w^{2+2/(1-2/L)} L^2}{n} \rmm{w^2L} {{\bar M}^{\frac{1-2/(pL)}{1-2/L}}\delta^{-\frac{2/L}{1-2/L}} (\log ({\bar M}/\delta)} + 3) + \delta,
\end{eqal}
where $\|f\|_{B^\alpha_{p,q}}, m$ and $d$ taken as constants.
By choosing 
% $$
% \delta \eqsim \frac{\rmm{(w^2L)^{1-2/L}}{\bar M}^{1-2/(pL)}}{n^{1-2/L}}, {\bar M} \eqsim \rmm{\Bigl(
%    \frac{n}{w^2L}
% \Bigr)}n^\frac{1-2/L}{2\alpha/d + 1 - 2/(pL)}.
% $$
$$
\delta \eqsim 
\frac{w^{4-4/L}L^{2-4/L}\rmm{(w^2L)^{1-2/L}}{\bar M}^{1-2/(pL)}}{n^{1-2/L}}, 
{\bar M} \eqsim \Big(\frac{n^{1-2/L}}{w^{4-4/L}L^{2-4/L}}\Big)^{\frac{1}{2\alpha/d + 1 - 2/(pL)}},
$$
we get 
\begin{eqal}
   \MSE(\hat f) 
   &\leq \tilde O \Bigg(\Big(\frac{w^{4-4/L}L^{2-4/L}}{n^{1-2/L}}\Big)^\frac{2\alpha/d}{2\alpha/d+1-2/(pL)}  + e^{-c_6 L}\Bigg)
\end{eqal}
where $\MSE(\hat f) $ shows the MSE of the solution to constrained optimization problem \autoref{eq:l2c} by optimally choosing $\bar M$ (or $P'$).
% \yw{Comment 1: What is the difference between $\simeq$ and $\eqsim$?}
% \yws{Comment 2: I think it is better to assume $M \gtrsim \Bigl(
%    \frac{n}{w^2L}
% \Bigr)^\frac{1-2/L}{2\alpha/d + 1 - 2/(pL)}$ directly in the theorem so there is no need to introduce $\bar{M}$, and the associated confusion when first reading the theorem.}

\modif{
Finally, under the assumption in \autoref{lemma:cons2regu}, for any constrained optimization problem, there exists a regularized optimzation problem, whose MSE is not larger than the MSE of the constrained optimization problem up to a factor of a constant. 
This closes the connection between \autoref{eq:l2r} and \autoref{eq:l2c} and finishes the proof.
}

Note that the empirical risk minimizer (ERM) of the parallel nerual network satisfy that 
% there exists a weight decay parameter $\lambda'$ such that 
the $(2/L)$-norm of the coefficients of the parallel neural network satisfy that $\|\{a_j\}\|_{2/L}^{2/L} = \|\{\tilde a_{j,\bar{M}}\}\|_{2/L}^{2/L}$ where 
$\{\tilde a_{j,\bar{M}}\}$ is the coefficient of the particular $\bar{M}$-sparse approximation, {although $\{a_j\}$ is not necessarily $\bar M$ sparse.}
% Here ${\bar M}$ is the number of ``active'' subnetworks that suffices to achieve the claimed approximation error. Here we call a subnetwork ``active'' the weight matrix in it is nonzero. 
Empirically, one only need to guarantee that during initialization, the number of subnetworks 
$M \geq \bar M$ such that the $\bar{M}$-sparse approximation is feasible, thus the approximation error bound from Theorem~\ref{thm:pnormapp} can be applied.
Theorem~\ref{thm:pnormapp} also says that 
$\|\{a_j\}\|_{2/L}^{2/L} = {\|\{\tilde a_{j,\bar{M}}\}\|^{2/L}_{2/L}}\lesssim \bar{M}^{1-2/{pL}},$
thus we can apply the covering number bound from Theorem~\ref{thm:pnncv} with $P' = \bar{M}^{1-2/{pL}}$.
% $\bar M \gtrsim n^\frac{1-2/L}{2\alpha/d + 1 - 2/(pL)}$ 
Finally, if $\lambda$ is optimally chosen, then it achieves a smaller MSE than this particular $\lambda'$, which has been proven to be no more than $O(\bar M^{-\alpha/d})$ and completes the proof.

\end{proof}

\begin{proof}[Proof of \autoref{prop:msecov}]
   For any function $f \in \gF$, define $f_\bot = \argmin_{h \in \gF_\bot} \sum_{i=1}^n (f(\vect x_i) - h(\vect x_i))^2$ be the projection of $f$ to $\gF_\bot$, and define $f_\parallel = f - f_\bot$ be the projection to the orthogonal complement.
   Note that $f_\parallel$ is not necessarily in $\gF_\parallel$. 
   However, if $f \in \gF$, then $f_\parallel \in \gF_\parallel$.
   $y_{i\bot}$ and $y_{i\parallel}$ are defined by creating a function $f_y$ such that $f_y(\vect x_i) = y_i, \forall i$, e.g. via interpolation.
   % Similarly, define $\vect y_\bot = \argmin_{\vect y_\bot \in \gF_\bot} \sum_{i=1}^n (y_i - y_{i\bot})^2, \vect y_\parallel = \vect y - \vect y_\bot$.
   Because $\gF_\parallel$ and $\gF_\bot$ are orthogonal, the empirical loss and population loss can be decomposed in the same way:
   \begin{eqal*}
      L_\parallel(f) &= \frac{1}{n}\sum_{i=1}^n (f_\parallel(\vect x) - f_{0\parallel}(\vect x))^2 + \frac{n-N}{n} \sigma^2, &
      %----------------------
      L_\bot(f) &= \frac{1}{n}\sum_{i=1}^n (f_\bot(\vect x) - f_{0\bot}(\vect x))^2 + \frac{N}{n} \sigma^2,\\
      %----------------------
      \hat L_\parallel(f) &= \frac{1}{n}\sum_{i=1}^n (f_\parallel(\vect x) - y_{i\parallel})^2, &
      %----------------------
      \hat L_\bot(f) &= \frac{1}{n}\sum_{i=1}^n (f_\bot(\vect x) - y_{i\bot}(\vect x))^2,\\
      %---------------------
      MSE_\parallel(f) &= \E_\gD\Big[\frac{1}{n}\sum_{i=1}^n (f_\parallel(\vect x) - f_{0\parallel}(\vect x))^2 \Big],&
      %----------------------
      MSE_\bot(f) &= \E_\gD\Big[\frac{1}{n}\sum_{i=1}^n (f_\bot(\vect x) - f_{0\bot}(\vect x))^2 \Big],\\
   \end{eqal*}   
   such that $L(f)=L_\parallel(f)+L_\bot(f), \hat L(f)=\hat L_\parallel(f)+\hat L_\bot(f)$. This can be verified by decomposing $\hat f, f_0$ and $y$ into two orthogonal components as shown above, and observing that $\sum_{i=1}^n f_{1\bot} (\vect x_i)  f_{2\parallel} (\vect x_i) = 0, \forall f_1, f_2$.

   \textbf{First prove the following claim}
   \begin{claim}
      Assume that $\hat f = \argmin_{f \in \gF} \hat L(f)$ is the empirical risk minimizer. Then $\hat f_\bot = \argmin_{f \in \gF_\bot} \hat L_\bot(f), \hat f_\parallel = \argmin_{f \in \gF_\parallel} \hat L_\parallel(f)$, where $\hat f_\bot$ is the projections of $\hat f$ in $\gF_\bot$, and $\hat f_\parallel = \hat f - \hat f_\bot$ respectively. 
   \end{claim}
   \begin{proof}
   Since $\hat f \in \gF$, by definition $\hat f_\parallel \in \gF_\parallel$. Assume that there exist $\hat f'_\bot, \hat f'_\parallel$, and either $\hat L_\bot(\hat f'_\bot) < \hat L_\bot(\hat f_\bot)$, or $\hat L_\parallel(\hat f'_\parallel) < \hat L_\parallel(\hat f_\parallel)$. Then 
   \begin{eqal*}
      \hat L(\hat f') &= \hat L(\hat f'_\bot + \hat f'_\parallel) 
      = \hat L_\parallel (\hat f'_\bot + \hat f'_\parallel) + \hat L_\bot (\hat f'_\bot + \hat f'_\parallel)
      = \hat L_\parallel (\hat f'_\parallel) + \hat L_\bot (\hat f'_\bot)\\
      &< \hat L_\parallel (\hat f_\parallel) + \hat L_\bot (\hat f_\bot)
      = \hat L_\parallel (\hat f_\bot + \hat f_\parallel) + \hat L_\bot (\hat f_\bot + \hat f_\parallel)
      = \hat L (\hat f) 
   \end{eqal*}
   which shows that $\hat f$ is not the minimizer of $\hat L(f)$ and violates the assumption.

\end{proof}

   \textbf{Then we bound $MSE_\bot(f)$.} We convert this part into a finite dimension least square problem: 
   \begin{eqal*}
      \hat f_\bot &= \argmin_{f \in \gF_\bot} \hat L_\bot(f) \\
      % &= \argmin_{f \in \gF_\bot} \sum_{i=1}^n (f(\vect x_i) - y_{i\bot})^2 \\
      &= \argmin_{f \in \gF_\bot} \frac{1}{n}\sum_{i=1}^n (f(\vect x_i) - f_{0\bot}(\vect x_i) - \epsilon_{i\bot})^2 \\
      &= \argmin_{f \in \gF_\bot} \frac{1}{n}\sum_{i=1}^n (f(\vect x_i) - f_{0\bot}(\vect x_i) - \epsilon_{i\bot})^2  + \epsilon_{i\parallel}^2\\
      &= \argmin_{f \in \gF_\bot} \frac{1}{n}\sum_{i=1}^n (f(\vect x_i) - f_{0\bot}(\vect x_i) - \epsilon_{i\bot} - \epsilon_{i\parallel})^2\\
      &= \argmin_{f \in \gF_\bot} \frac{1}{n}\sum_{i=1}^n (f(\vect x_i) - f_{0\bot}(\vect x_i) - \epsilon_i)^2\\
   \end{eqal*}
   The forth line comes from our assumption that $\gF_\bot$ is orthogonal to $\gF_\parallel$, so $\forall f \in \gF_\bot, f + f_{0\bot} + \epsilon_\bot$ is orthogonal to $\epsilon_\parallel$.

   Let the basis function of $\gF_\bot$ be $h_1, h_2, \dots, h_N$, the above problem can be reparameterized as 
   \begin{eqal*}
      \argmin_{\vect \theta \in \R^N} \frac{1}{n} \| \mat X \vect \theta - \vect y\|^2
   \end{eqal*}
   where $\mat X \in \R^{n\times N}: X_{i} = h_j(\vect x_i), \vect y =  \vect y_{0\bot} + \vect \epsilon, \vect y_{0\bot} = [f_{0\bot}(x_1), \dots, f_{0\bot}(x_n)], \vect \epsilon = [\epsilon_1, \dots, \epsilon_n]$. This problem has a closed-form solution 
   \begin{eqal*}
      \vect \theta = (\mat X^T \mat X)^{-1} \mat X^T \vect y
   \end{eqal*} 
   Observe that $f_{0\bot} \in \gF_\bot$, let $\vect y_{0\bot} = \mat X \vect \theta^*$,The MSE of this problem can be computed by 
   \begin{eqal*}
      L(\hat f_\bot) &= \frac{1}{n}\|\mat X \vect \theta  - \vect y_{0\bot}\|^2 
      = \frac{1}{n}\|\mat X(\mat X^T \mat X)^{-1} \mat X^T (\mat X \vect \theta^* + \vect \epsilon) - \mat X \vect \theta^*\|^2 \\
      &= \frac{1}{n}\|\mat X(\mat X^T \mat X)^{-1} \mat X^T  \vect \epsilon\|^2 \\
   \end{eqal*}

   Observing that $\Pi := \mat X(\mat X^T \mat X)^{-1} \mat X^T$ is an idempotent and independent projection whose rank is $N$, and that $\E[\vect \epsilon \vect \epsilon^T] = \sigma^2 \mat I$, we get 
   \begin{eqal*}
      \MSE_\bot (\hat f_\bot) & = \E[L(\hat f_\bot)] = \frac{1}{n} \|\Pi \vect \epsilon \|^2 = \frac{1}{n}\tr (\Pi \vect \epsilon \vect \epsilon^T) = \frac{\sigma^2}{n} \tr(\Pi)
   \end{eqal*}
   % The MSE of this problem is known \citep[Section 2.7]{liang2016statistical}:
   which concludes that 
   \begin{eqal}
      \MSE_\bot(\hat f) 
      % - \min_{f \in \gH} MSE_\bot(f) 
      = O\Big(\frac{N}{n}\sigma^2\Big) .
      \label{eq:errbot}
   \end{eqal}
   
   See also \citep[Proposition 1]{hsu2011analysis}.
%   See also \citep[Section 2.7]{liang2016statistical}.

   \textbf{Next we study $\MSE_\parallel(\hat f)$.} 
   % For any $f_{j\parallel} \in \tilde \gF_\parallel, f_\bot \in \gH$, and for any fixed $f_0$, denote (with slightly abuse of notation) $f_j = f_{j\parallel} + f_\bot$, and 
   Denote $\tilde\sigma_\parallel^2=\frac{1}{n}\sum_{i=1}^n \epsilon_{i\parallel}^2, E = \max_i |\epsilon_i|$. 
   % From Gaussian tail bound, the distribution of $\epsilon_i$ satisfy
   % \begin{eqal*}
   %    P[|\epsilon_i| \geq t] \leq 2\exp\Big(-\frac{t^2}{2\sigma^2}\Big),
   % \end{eqal*}
   % using the union bound, 
   % \begin{eqal*}
   %    P[E\geq t] = P\Big[\max_{1 \leq i \leq n } |\epsilon_i| \geq t \Big] \leq 2n\exp\Big(-\frac{t^2}{2\sigma^2}\Big),
   % \end{eqal*}
   % The expectation of $E$ can be computed by
   % \begin{eqal*}
   %    \E[E] = \int_{t=0}^\infty P[E\geq t] dt = 
   % \end{eqal*}
   Using Jensen's inequality and union bound, we have 
   \begin{eqal*}
      \exp(t\E[E]) &\leq \E[\exp(tE)] = \E[\max\exp(t|\epsilon_i|)]
      \leq \sum_{i=1}^n \E[\exp(t|\epsilon_i|)] 
      \leq 2n\exp(t^2\sigma^2/2)
   \end{eqal*}
   Taking expectation over both sides, we get 
   \begin{eqal*}
      \E[E] \leq \frac{\log 2n}{t} + \frac{t\sigma^2}{2}
   \end{eqal*}
   maximizing the right hand side over $t$ yields
   \begin{eqal*}
      \E[E] \leq \sigma\sqrt{2\log 2n}.
   \end{eqal*}

   % From \citet[Lemma C.1]{schmidt2020nonparametric}, $\E[E] \leq (3\log n + 1)\sigma$. 
   % which can be bounded by
   % \begin{eqal*}
   %    P(E > t) &\leq n P(|\epsilon_1| > t) 
   %    \leq \frac{1}{\sqrt{2\pi}}\frac{n\sigma}{t}\exp\Bigl(-\frac{t^2}{2\sigma^2}\Bigr)
   % \end{eqal*}
   % so with probability at least $1-\delta_e$,
   % \begin{eqal*}
   %    E \leq \tilde O(\sigma \sqrt{\log (n/\delta_e)}).
   % \end{eqal*}
   Let $\tilde \gF_\parallel$ be the covering set of $\gF_\parallel = \{f_\parallel: f \in \gF\}$.
   For any $\tilde f_\parallel \in \tilde \gF_\parallel$,
   \begin{eqal*}
      L_\parallel(f_j) - \hat L_\parallel(f_j) 
      &= \frac{1}{n} \sum_{i=1}^n (f_{j\parallel}(\vect x_i) - f_{0\parallel}(\vect x_i))^2 - \frac{1}{n} \sum_{i=1}^n (\tilde f_{\parallel}(\vect x_i) - y_{i\parallel})^2 + \frac{n-N}{n} \sigma^2\\
      %-------------------
      &= \frac{1}{n} \sum_{i=1}^n \epsilon_{i\parallel} (2\tilde f_{\parallel}(\vect x_i) - f_{0\parallel}(\vect x_i) - y_{i\parallel}) + \frac{n-N}{n} \sigma^2\\
      %------------------
      &= \frac{1}{n} \sum_{i=1}^n \epsilon_{i} (2\tilde f_{\parallel}(\vect x_i) - f_{0\parallel}(\vect x_i) - y_{i\parallel}) + \frac{n-N}{n} \sigma^2\\
      %-------------------
      &= \frac{1}{n} \sum_{i=1}^n \epsilon_i (2\tilde f_{\parallel}(\vect x_i) - 2f_{0\parallel}(\vect x_i)) 
      + \frac{n-N}{n} \sigma^2 - \tilde\sigma_\parallel^2  \\
   \end{eqal*}
   The first term can be bounded using Bernstein's inequality: let $h_i = \epsilon_i(f_{j\parallel}(\vect x_i) - f_{0\parallel}(\vect x_i))$, by definition $|h_i| \leq 2EF$, 
   \begin{eqal*}
      \Var[h_i] 
      & = \E[\epsilon_i^2(\tilde f_{\parallel}(\vect x_i) - f_{0\parallel}(\vect x_i))^2]\\
      %------------------
      & = (\tilde f_{\parallel}(\vect x_i) - f_{0\parallel}(\vect x_i))^2 \E[\epsilon_i^2]\\
      %------------------
      & = (\tilde f_{\parallel}(\vect x_i) - f_{0\parallel}(\vect x_i))^2 \sigma^2\\
      % & \leq \sigma^2 \Big(L_\parallel(f_j) - \frac{n-N}{n} \sigma^2\Big)
   \end{eqal*}
   using Bernstein's inequality, for any $\tilde f_\parallel \in \tilde \gF_\parallel$, with probably at least $1-\delta_p$,
   \begin{eqal*}
      \frac{1}{n} \sum_{i=1}^n \epsilon_i (2\tilde f_{\parallel}(\vect x_i) - 2f_{0\parallel}(\vect x_i)) 
      &= \frac{2}{n} \sum_{i=1}^n h_i\\
      %--------------------------------
      &\leq \frac{2}{n}\sqrt{2\sum_{i=1}^n \big(\tilde f_{\parallel}(\vect x_i) - f_{0\parallel}(\vect x_i)\big)^2 \sigma^2\log(1/\delta_p)}
      + \frac{8EF\log(1/\delta_p)}{3n} \\
      %-------------------------
      &= 2\sqrt{\Big(L_\parallel(\tilde f_{\parallel}) - \frac{n-N}{n} \sigma^2\Big) \frac{2\sigma^2\log(1/\delta_p)}{n}}
      + \frac{8EF\log(1/\delta_p)}{3n} \\
      %-------------------------
      &\leq \epsilon \Big(L_\parallel(\tilde f_{\parallel}) - \frac{n-N}{n} \sigma^2\Big) + \frac{8\sigma^2\log(1/\delta_p)}{n\epsilon} + \frac{8EF\log(1/\delta_p)}{3n}
   \end{eqal*}
   the last inequality holds true for all $\epsilon > 0$.
   The union bound shows that with probably at least $1-\delta$, for all $\tilde f_\parallel \in \tilde \gF_\parallel$,
   \begin{eqal*}
      % \frac{1}{n} \sum_{i=1}^n \epsilon_i (2f_{j\parallel}(\vect x_i) - 2f_{0\parallel}(\vect x_i)) 
      L_\parallel(\tilde f_{\parallel}) - \hat L_\parallel(\tilde f_{\parallel})
      &\leq \epsilon \Big(L_\parallel(\tilde f_{\parallel}) - \frac{n-N}{n} \sigma^2\Big) 
      + \frac{8\sigma^2\log(\gN(\gF_\parallel, \delta)/\delta_p)}{n\epsilon} + \frac{8EF\log(\gN(\gF_\parallel, \delta)/\delta_p)}{3n}\\
      &\quad+ \frac{n-N}{n} \sigma^2 - \tilde\sigma_\parallel^2.
   \end{eqal*} 
   By rearanging the terms and using the definition of $L(\tilde f_{\parallel})$, we get
   \begin{eqal*}
      (1-\epsilon)\Big(L_\parallel(\tilde f_{\parallel}) - \frac{n-N}{n} \sigma^2\Big)
      \leq \hat L_\parallel(\tilde f_{\parallel})  
      + \frac{8\sigma^2\log(\gN(\gF_\parallel, \delta)/\delta_p)}{n\epsilon} + \frac{8EF\log(\gN(\gF_\parallel, \delta)/\delta_p)}{3n}
      % + \frac{n-N}{n} \sigma^2 
      - \tilde\sigma_\parallel^2.
   \end{eqal*}
   Taking the expectation (over $\gD$) on both sides, and notice that $\E[\tilde \sigma_\parallel^2] = \frac{n-N}{n} \sigma^2$. Furthermore, for any random variable $X, \E[X] = \int_{-\infty}^\infty xdP(X \leq x)$, we get 
   \begin{eqal}
      &\max_{\tilde f_\parallel \in \tilde \gF_\parallel}
      \Big((1-\epsilon) \MSE_\parallel(\tilde f_\parallel) - \E[\hat L_\parallel(\tilde f_{\parallel})] \Big)\\
      %---------------------------
      &\leq \Bigl(\frac{8\sigma^2}{n\epsilon} 
      + \frac{8F\sigma\sqrt{2\log2n}}{3n}\Bigr)\Big(\log\gN(\gF_\parallel, \delta) - \int_{\delta=0}^1\log(\delta_p)d\delta_p\Big) - \frac{n-N}{n} \sigma^2\\
      %------------------------------
      &= \Bigl(\frac{8\sigma^2}{n\epsilon} 
      + \frac{8F\sigma\sqrt{2\log2n}}{3n}\Bigr)(\log\gN(\gF_\parallel, \delta) +1) - \frac{n-N}{n} \sigma^2.
      \label{eq:sb1}
   \end{eqal}
   where the integration can be computed by replacing $\delta$ with $e^x$. Though it is not integrable under Riemann integral, it is integrable under Lebesgue integration. 

   Similarly, let $\check f_\parallel = \argmin_{f \in \gF_\parallel} L_\parallel(f)$,
   \begin{eqal*}
      L_\parallel(\check f_\parallel) - \hat L_\parallel(\check f_\parallel) 
      &= \frac{1}{n} \sum_{i=1}^n \epsilon_i (2\check f_\parallel(\vect x_i) - 2f_{0\parallel}(\vect x_i)) 
      + \frac{n-N}{n} \sigma^2 - \tilde\sigma_\parallel^2  \\
   \end{eqal*}
   with probably at least $1-\delta_q$, for any $\epsilon > 0$,
   \begin{eqal*}
      &-\frac{1}{n} \sum_{i=1}^n \epsilon_i (2\check f_{\parallel}(\vect x_i) - 2 f_{0\parallel}(\vect x_i)) 
      \leq \epsilon \Big(L_\parallel(\check f_{\parallel}) - \frac{n-N}{n} \sigma^2\Big) 
      + \frac{8\sigma^2\log(1/\delta_p)}{n\epsilon} 
      + \frac{8EF\log(1/\delta_p)}{3n}, \\
      %---------------------
      % MSE_\parallel(\tilde f_{\parallel}) 
      % &\geq \frac{1}{1+\epsilon}\Big(\hat L_\parallel(\tilde f_{\parallel})  
      % - \frac{2\log(1/\delta_q)}{n\epsilon} - \frac{8EF\log(1/\delta_q)}{n}
      % - \tilde\sigma_\parallel^2\Big).
      &\hat L_\parallel(\check f_\parallel) 
      \leq (1+\epsilon)\Big(L_\parallel(\check f_\parallel) - \frac{n-N}{n} \sigma^2\Big)
      + \frac{8\sigma^2\log(1/\delta_p)}{n\epsilon} 
      + \frac{8EF\log(1/\delta_q)}{3n}
      + \tilde\sigma_\parallel^2.
      % - \frac{n-N}{n} \sigma^2.
   \end{eqal*}
   Taking the expectation on both sides,
   \begin{eqal}
      \E[\hat L_\parallel(\check f_\parallel)] \leq (1+\epsilon) \MSE_\parallel(\check f_\parallel) 
      + \frac{8\sigma^2}{n\epsilon} 
      + \frac{8F\sigma\sqrt{2\log2n}}{3n}
      + \frac{n-N}{n} \sigma^2.
      \label{eq:sb2}
   \end{eqal}
   Finally, let $\hat f_{*} := \argmin_{f \in \tilde \gF_\parallel} \sum_{i=1}^n (\hat f_\parallel(\vect x_i) - f(\vect x_i))^2$ be the projection of $\hat f_\parallel$ in its $\delta$-covering space,
   % Using the union bound, with probability at least $1-\delta_p - \delta_q$,
   \begin{eqal*}
      % L_\parallel(\hat f) - \hat L_\parallel(\hat f) 
      % &= \frac{1}{n} \sum_{i=1}^n \epsilon_{i\parallel} (2\hat f_{\parallel}(\vect x_i) - f_{0\parallel}(\vect x_i) - y_{i\parallel}) + \frac{n-N}{n} \sigma^2\\
      % %-----------------------
      % & \leq \frac{1}{n} \sum_{i=1}^n \epsilon_{i\parallel} (2 f_{j^*}(\vect x_i) - f_{0\parallel}(\vect x_i) - y_{i\parallel}) + \frac{n-N}{n} \sigma^2 + 2\delta E\\
      % %----------------------
      % &= L_\parallel(f_{j^*}) - \hat L_\parallel(f_{j^*})  + 2\delta E\\
      % %----------------------
      % &= MSE_\parallel(f_{j^*}) - \hat L_\parallel(f_{j^*})  + 2\delta E + \frac{n-N}{n} \sigma^2\\
      \MSE_\parallel(\hat f_\parallel) 
      &= \E\Big[\frac{1}{n}\sum_{i=1}^n (\hat f_\parallel(\vect x_i) - f_{0\parallel}(\vect x_i))^2 \Big]\\
      %--------------------------
      &= \E\Big[\frac{1}{n}\sum_{i=1}^n (\hat f_{*}(\vect x_i) - f_{0\parallel}(\vect x_i))^2 
      + \frac{1}{n}\sum_{i=1}^n (\hat f_\parallel(\vect x_i) - \hat f_{*}(\vect x_i))(\hat f_\parallel(\vect x_i) + \hat f_{*}(\vect x_i) - 2 f_{0\parallel}(\vect x_i))\Big]
      \\
      %-------------------------
      &\leq \E\Big[\frac{1}{n}\sum_{i=1}^n (\hat f_{*}(\vect x_i) - f_{0\parallel}(\vect x_i))^2\Big] + 4F\delta \\
      & = \MSE_\parallel (\hat f_{*}(\vect x_i)) + 4F\delta, \\
   \end{eqal*}
   and similarly 
   \begin{eqal}
      \hat L_\parallel(\hat f_*) \leq \hat L_\parallel(\hat f_\parallel) + (4F+2E)\delta.
      \label{eq:lcov}
   \end{eqal}
   We can conclude that
   \begin{eqal*}
      \MSE_\parallel(\hat f_\parallel) 
      %----------------------------
      &\leq \frac{1}{1-\epsilon}\Big(\E[\hat L_\parallel(\hat f_{*})]  
      +  \Bigl(\frac{8\sigma^2}{n\epsilon} 
      + \frac{8F\sigma\sqrt{2\log2n}}{3n}\Bigr)(\log\gN(\gF_\parallel, \delta)+1)  - \frac{n-N}{n} \sigma^2 \Big)\\
      &\quad + 4F\delta \\
      %----------------------
      &\leq \frac{1}{1-\epsilon}\Big(\E[\hat L_\parallel(\hat f_\parallel)] + (4F + \sigma\sqrt{8\log2n})\delta  \\
      &\quad +  \Bigl(\frac{8\sigma^2}{n\epsilon} 
      + \frac{8F\sigma\sqrt{2\log2n}}{3n}\Bigr)(\log\gN(\gF_\parallel, \delta)+1) - \frac{n-N}{n} \sigma^2 \Big)
      + 4F\delta \\
      %----------------------
      &\leq \frac{1}{1-\epsilon}\Big(\E[\hat L_\parallel(\check f_\parallel)] + (4F + \sigma\sqrt{8\log2n})\delta \\
      &\quad +  \Bigl(\frac{8\sigma^2}{n\epsilon} 
      + \frac{8F\sigma\sqrt{2\log2n}}{3n}\Bigr)(\log\gN(\gF_\parallel, \delta)+1) - \frac{n-N}{n} \sigma^2 \Big)
      + 4F\delta \\
      % &\leq \frac{1}{1-\epsilon}\Big(\hat L_\parallel(\tilde f)  
      % + \frac{2\log(\gN(\gF_\parallel, \delta)/\delta_p)}{n\epsilon} + \frac{8EF\log(\gN(\gF_\parallel, \delta)/\delta_p)}{n}
      % + \frac{n-N}{n} \sigma^2 - \tilde\sigma_\parallel^2\Big)\\
      % & \quad + 4F\delta + \frac{n-N}{n} \sigma^2\\
      %--------------------
      & \leq \frac{1+\epsilon}{1-\epsilon} \MSE_\parallel(\check f_\parallel)
      + \frac{1}{n}\Bigl(\frac{8\sigma^2}{\epsilon} 
      + \frac{8F\sigma\sqrt{2\log2n}}{3}\Bigr)\Big(\frac{\log\gN(\gF_\parallel, \delta)+2}{1-\epsilon}\Big) \\
      &\quad  + \Big(4F + \frac{4F + \sigma\sqrt{8\log 2n}}{1-\epsilon}\Big)\delta,
      %-----------------------
      % L(\hat f) - \hat L(\hat f)
      % &= L_\parallel(\hat f) - \hat L_\parallel(\hat f) 
      % + L_\bot(\hat f) - \hat L_\bot(\hat f)
   \end{eqal*}
   where the first line comes from \autoref{eq:sb1}, and second comes from \autoref{eq:lcov}, the thid line is because $\hat f_\parallel = \argmin_{f \in \gF_\parallel} \hat L_\parallel(f)$, and the last line comes from \autoref{eq:sb2}.
   We also use that fact that $\hat L_\parallel(\hat f) \leq \hat L_\parallel(f), \forall f$.
   Noticing that $\MSE(\hat f) = \MSE_\parallel(\hat f) + \MSE_\bot(\hat f)$, combining this with \autoref{eq:errbot} finishes the proof.
\end{proof}

\modif{
\begin{lemma}
   \label{lemma:cons2regu}
   Assume that these exists $C_1, C_2 >1$ (which may depend on the target function), for all $P'>0$, there exists $\lambda > 0$, such that the soltion to the regularized optimization problem \autoref{eq:l2r}, denoted as $\tilde f$, satisfy 
   \begin{eqal*}
      C_1 P' \leq \|\{\tilde a_j\}\|_{2/L}^{2/L} \leq C_2 P',
   \end{eqal*}
   then the MSE of the regularized optimization problem satisfy 
   \begin{eqal*}
      \MSE(\tilde f) \leq C \MSE(\hat f)
   \end{eqal*}
   where $C$ is a constant that depends on $C_1, C_2$, $\hat f$ is the solution to the constrained optimzation problem \autoref{eq:l2c}, and 
   \begin{eqal*}
      \lambda \lesssim \frac{\MSE(\hat f)}{P'} \lesssim n^{-(1-2/L)}
   \end{eqal*}
\end{lemma}
}
%    Let $\hat f$ be the empirical risk minimize of a constrained problem:
%    \begin{eqal*}
%       \hat f = \argmin_{f} \hat L(f), \textrm{ s.t. } R(f) \leq P
%    \end{eqal*}
%    where both $\hat L(f)$ and $R(f)$ are nonnegative for all $f$, and $R(\hat f) > 0$, and let 
%    \begin{eqal}
%       \label{eq:errn}
%       M(f) := c_1 \min_{f \in \gF} L(f) + c_2 R(f)^{c_3} \log(R(f)) + c_4
%    \end{eqal}
%    be the target function we want to bound, then for all $P>0$, let $\lambda = \frac{\hat L(\hat f)}{R(\hat f)}$, then solution to the empirical risk minimize of regularized problem 
%    \begin{eqal*}
%       \tilde f = \argmin_{f} \hat L(f) + \lambda R(f)
%    \end{eqal*}
%    satisfy 
%    \begin{eqal*}
%       M(\tilde f) \leq 2^{ c_3}M(\hat f) + c_2\log 2.
%    \end{eqal*}

\modif{
% Notice that \autoref{eq:errn} is simplification of \autoref{eq:errorrate} and \autoref{eq:cnbound} by taking $d, w$ and $L$ as constants, and has already been proved, and the trailing $O(1/n)$ term is neglectable as $\MSE_{\textrm{constrained}}(n)$ converges slower than it.
\begin{proof}
   The MSE of the regularized problem can be achieved by taking our assumtion into \autoref{eq:errdecom}. We only need to prove the selection of $\lambda$.
   We apply the decomposition as in \autoref{prop:msecov}, and only need to consider $\gF_\parallel$, as $\gF_\bot$ is not imfluenced by regularization or constrained. From the definition of $\tilde f$ and $\lambda$, we have
   \begin{eqal*}
      % &\hat L( \tilde f) + \lambda \|\{\tilde a_j\}\|_{2/L}^{2/L} \leq \hat L( \hat f) + \lambda \|\{\hat a_j\}\|_{2/L}^{2/L} = 2\hat L( \hat f)
      \hat L(\tilde f) + \lambda \|\{\tilde a_j\}\|_{2/L}^{2/L} &\leq \hat L(\hat f) +  \lambda \|\{\hat a_j\}\|_{2/L}^{2/L}, \\
      %----------------------------------
      \hat L_\parallel(\tilde f) + \lambda \|\{\tilde a_j\}\|_{2/L}^{2/L} &\leq L_\parallel(\hat f) +  \lambda \|\{\hat a_j\}\|_{2/L}^{2/L}
   \end{eqal*}
   From \autoref{prop:msecov}, we get 
   \begin{eqal}
      \label{eq:errwr}
      &(1-\epsilon)\MSE(\tilde f) - O\left(\frac{\log \gN(\|\{\tilde a_j\}\|_{2/L}^{2/L}, \delta)}{n} \right) + \lambda \|\{\tilde a_j\}\|_{2/L}^{2/L} \\
      %--------------------------
      &\leq (1+\epsilon) \MSE(\hat f) + O\left(\frac{\log \gN(\|\{\hat a_j\}\|_{2/L}^{2/L}, \delta)}{n}\right) + \lambda \|\{\hat a_j\}\|_{2/L}^{2/L}
   \end{eqal}
   Observing that $\MSE(\tilde f) \geq 0$, and $\frac{\log \gN(\|\{\hat a_j\}\|_{2/L}^{2/L}, \delta)}{n} \eqsim MSE(\hat f)$ for the optimally chosen $P'$, taking the assumtion into the inequality proves the choice of $\lambda$.
   % where $\{\hat a\}$ are the coefficients in $\hat f$.
   % Because both $L(\tilde f)$ and $R(\tilde f)$ are nonnegative, we have 
   % \begin{eqal*}
   %    % &\hat L( \tilde f) \leq 2\hat L( \hat f), 
   %    % \quad P'' := \|\{\tilde a_j\}\|_{2/L}^{2/L} \leq \frac{2\hat L( \hat f)}{\lambda} = 2P'
   %    R(\tilde f) \leq \frac{1}{\lambda} (L(\hat f) + \lambda R(\hat f)) = 2R(\hat f).
   % \end{eqal*}
   % Because the feasible region of the constrained optimization problem is the subset of the regularized optimization problem,
   % \begin{eqal*}
   %    \min_{f_\textrm{regularized}} L(f) \leq \min_{f_\textrm{constrained}} L(f)
   % \end{eqal*} 
   % Taking them into \autoref{eq:errn} finishes the proof.
\end{proof}
}

\modif{
   \begin{remark}
      Define $R(\lambda) := R(\argmin \hat L(f) + \lambda R(f))$, where $R(f) = \|\{a_j\}\|_{2/L}^{2/L}$ is the regularizer term of a parallel NN ($f$).
      Notice that $R(\lambda)$ is a non-increasing function of $lambda$ (as proved below), the assumption in \autoref{lemma:cons2regu} is equivalent to that if $R(\lambda)$ contains any uncontinuous points, then the uncontinuous points should not be larger than $\frac{C_2}{C_1}$ in ratio.
      On the other hand,
      if $\lambda$ is chosen as $\lambda = O(\frac{\MSE(\hat f)}{P'})$, then from \autoref{eq:errwr}, we get 
      \begin{eqal*}
         \lambda \|\{\tilde a_j\}\|_{2/L}^{2/L} &\leq
         O(\MSE(\hat f)) +  O\left(\frac{\log \gN(\|\{\hat a_j\}\|_{2/L}^{2/L}, \delta)}{n}\right) \\
         & \leq O(\MSE(\hat f)) + \frac{1}{n}\tilde O((\|\{\hat a_j\}\|_{2/L}^{2/L})^{\frac{1}{1-2/L}})
      \end{eqal*}
      If the constant term in $\lambda$ is large enough, the above inequality yields two sets of solutions: 
      $$
         \|\{\tilde a_j\}\|_{2/L}^{2/L} \leq  O\Bigg(\|\{\hat a_j\}\|_{2/L}^{2/L} + \frac{1}{\lambda} \MSE(\hat f) \Bigg) =  O(\|\{\hat a_j\}\|_{2/L}^{2/L}),
      $$ 
      and 
         $$\|\{\tilde a_j\}\|_{2/L}^{2/L} \geq \tilde{O}\big((n\lambda)^{\frac{1-2/L}{2/L}}\big).
      $$
      In the first case, one can easily see from \autoref{eq:errwr} that $\MSE(\tilde f) \leq O(\MSE(\hat f))$, which says that the MSE of the regularized problem is close to the minimax rate; in the later case, the generalization gap of the regularized problem is bounded by $O(n^{\frac{1-2/L}{2/L}}\lambda^{L/2})$, which is much larger than the former case. 
      So a sufficient condition of the above assumption is that the model does not overfit significantly (by orders of magnitude) more than the constrained version.
      In our experiment, we find that the latter case is very difficult to happen, possibly because of the implicit regularization during training, and the connection between $\lambda$ and effective degree of freedom is actually smooth.
      Notably, as $L$ gets larger, in the second case $\|\{\tilde a_j\}\|_{2/L}^{2/L}$ increases exponentially with $L$ (the constant terms depends at most polynomially on $L$), which suggests that the latter case is less likely to happen for deep neural networks.
   \end{remark}
}

\modif{
   \begin{claim}
      For fixed $\gD$, the regularized problem satisfy that $R(\lambda)$ as defined above is strictly non-increasing with $\lambda$. 
   \end{claim}
   \begin{proof}
      We provide a short proof by contradiction: 
      suppose that there exists $lambda_1 <\lambda_2$, and the solution satisfy $(f_1) <R(f_2)$ where $R(f) = \|\{a_j\}\|_{2/L}^{2/L}$ is the regularizer term of a parallel NN, $f_1, f_2$ are the solution to the regularized problem with $\lambda = \lambda_1, \lambda_2$ respectively. 
      Then by definition of $f_1, f_2$, we have $\hat L(f_1)+\lambda_1 R(f_1) \leq  \hat L(f_2) + \lambda_1 R(f_2)$ , so $\lambda_1 \geq \frac{\hat L(f_1) - \hat L(f_2)}{R(f_2) -R(f_1)}$; $\hat L(f_2)+\lambda_2 R(f_2) \leq \hat L(f_1) + \lambda_2 R(f_1)$, so $ \lambda_2 \leq \frac{\hat L(f_1) - \hat L(f_2)}{R(f_2) -R(f_1)}$ which is controversal to our assumption that $\lambda_1 <\lambda_2$. 
   \end{proof}
}

\section{More discussion about the main result}
\label{sec:adddis}
\noindent\textbf{Representation learning and adaptivity.} The results also shed a light on the role of representation learning in DNN's ability to adapt. Specifically, different from the two-layer NN in \citep{parhi2021banach},
% we explained in Section~\ref{sec:warmup}, 
which achieves the minimax rate of $BV(m)$ by choosing appropriate activation functions using each $m$, each subnetwork of a parallel NN can learn to approximate the spline basis of an arbitrary order, which means that if we choose $L$ to be sufficiently large, such Parallel NN with optimally tuned $\lambda$ is simultaneously near optimal for $m = 1,2,3,\dots$. In fact, even if different regions of the space has 
%not only different levels of smoothness within the same order, but also 
different \emph{orders} of smoothness,
%  like in the example we gave in Figure~\ref{fig:illus}, 
the paralle NN will still be able to learn appropriate basis functions in each local region. To the best of our knowledge, this is a property that none of the classical nonparametric regression methods possess.

\noindent\textbf{Synthesis v.s. analysis methods.} Our result could also inspire new ideas in estimator design. There are two families of methods in non-parametric estimation. One called \emph{synthesis} framework which focuses on constructing appropriate basis functions to encode the contemplated structures and regress the data to such basis, e.g., wavelets \citep{donoho1998minimax}. The other is called \emph{analysis} framework which uses analysis regularization on the data directly (see, e.g., RKHS methods \citep{scholkopf2001learning} or trend filtering \citep{tibshirani2014adaptive}).  It appears to us that parallel NN is doing both simultaneously. It has a parametric family capable to synthesizing an $O(n)$ subset of an exponentially large family of basis, then \emph{implicitly} use sparsity-inducing analysis regularization to select the relevant basis functions. In this way the estimator does not actually have to explicitly represent that exponentially large set of basis functions, thus computationally more efficient. 

\noindent\textbf{Random design problem.}
This paper focuses on the fixed design problem such that the results are comparable to that in nonparametic regression. One can easily apply the technique in this paper to achieve the estimation error bound on the random design problem:
\begin{theorem}
   \label{thm:rand}
   Under the same condition as \autoref{thm:main}, the solution $\hat{f}$ parameterized by \autoref{eq:l2} satisfies

   \begin{eqal*}
      &\E_\gD \E_f \MSE(\hat{f}) \leq \tilde O \Bigg(\Big(\frac{w^{4-4/L}L^{2-4/L}}{n^{1-2/L}}\Big)^\frac{2\alpha/d}{2\alpha/d+1-2/(pL)}  + e^{-c_6 L}\Bigg)
   \end{eqal*}
   where $\tilde O$ shows the scale up to a logarithmic factor, and $c_6$ is the constant defined in \autoref{thm:pnormapp}, $\E_\gD$ indicates that the expectation is taken with respect to the training set $\gD$, $\E_f$ indicates
   that the expectation is taken with respect to the domain of $f$.
\end{theorem}

The proof is similar to that of \autoref{thm:main}. The main difference lays in the proof of the estimation error. For $f_\bot$ part, the estimation error can be bounded using VC-dimension, which is 1. For $f_\parallel$ part, the estimation error can be bounded using its covering number, e.g. Lemma 8 in \citet{schmidt2020nonparametric}.

% Finally, if $\lambda$ is optimally chosen, then it achieves a smaller MSE than this particular $\lambda'$, which completes the proof.

%finishes the proof. 

\yws{Comment 3: I am a bit confused by the discussion below. There is no guarantee that when you penalize $p$-norm, you will get an exactly $\bar{M}$-sparse approximation (at least we did not prove it and I think we did not need to prove it). The $\bar{M}$-sparse approximation is used for the approximation only? The actual solution might not be $\bar{M}$-sparse, but will be at least as good even if it is fully dense.}

% have 
% $O(n^\frac{1-2/L}{2\alpha/d + 1 - 2/(pL)})$ 
% $\bar M$
% number of ``active'' subnetworks.

\section{Detailed experimental setup}
\label{sec:expdet}
\subsection{Target Functions}
% The piecewise cubic function used in \autoref{fig:exp}(a)-(c) is 
% \begin{eqal*}
%    f(x) &= 20(x-0.3)^3 -x -80 (x-0.6)_+^3 + 160(x-0.7)_+^3-320(x-0.8)_+^3 \\
%    &+320(x-0.85)_+^3 + 160(x-0.9)_+^3
% \end{eqal*}
% where $x_+ := \max(x, 0)$. 
The doppler function used in \autoref{fig:exp}(d)-(f) is  
\begin{eqal*}
   f(x) = \sin(4/(x+0.01)) + 1.5.
\end{eqal*}
The ``vary'' function used in \autoref{fig:exp}(g)-(i) is 
\begin{eqal*}
    f(x) &= M_1(x/0.01)+M_1((x-0.02)/0.02)+M_1((x-0.06)/0.03)\\
    &+M_1((x-0.12)/0.04)+M_3((x-0.2)/0.02)+M_3((x-0.28)/0.04)\\
    &+M_3((x-0.44)/0.06)+M_3((x-0.68)/0.08),
\end{eqal*}
where $M_1, M_3$ are first and third order Cardinal B-spline bases functions respectively.
% where $x_+ := \max(x, 0)$. 
We uniformly take 256 samples from 0 to 1 in the piecewise cubic function experiment, and uniformly 1000 samples from 0 to 1 in the doppler function and ``vary'' function experiment.
We add zero mean independent (white) Gaussian noise to the observations. The standard derivation of noise is
% 0.05 in the piecewise cubic function experiment,
0.4 in the doppler function experiment and 0.1 in the ``vary'' function experiment.

\subsection{Training/Fitting Method}
% In the experiments, the nerual network has a total depth of 10 and each subnetwork has a width of 10. 
% In the piecewise cubic function experiment, 100 subnetworks are used and in the doppler function experiment, 500 subnetworks are used. 
% In both experiments, the neural network contains 500 subnetworks.
In the piecewise polynomial function (``vary'') experiment, the depth of the PNN $L=10$, the width of each subnetwork $w=10$, and the model contains $M=500$ subnetworks. 
The depth of NN is also 10, and the width is 240 such that the NN and PNN have almost the same number of parameters.
In the doppler function experiment, the depth of the PNN $L=12$, the width of each subnetwork $w=10$, and the model contains $M=2000$ subnetworks, because this problem requires a more complex model to fit.
The depth of NN is 12, and the width is 470.  
We used Adam optimizer with learning rate of $10^{-3}$. 
We first train the neural network layer by layer without weight decay. Specifically, we start with a two-layer neural network with the same number of subnetworks and the same width in each subnetwork, then train a three layer neural network by initializing the first layer using the trained two layer one, until the desired depth is reached. After that, we turn the weight decay parameter and train it until convergence. 
In both trend filtering and smoothing spline experiment, the order is 3, and in wavelet denoising experiment, we use sym4 wavelet with soft thresholding.
We implement the trend filtering problem according to \citet{tibshirani2014adaptive} using CVXPY, and use MOSEK to solve the convex optimization problem. 
We directly call R function $smooth.spline$ to solve smoothing spline.

\subsection{Post Processing}
% The degree of freedom of parallel neural network is defined as the number of subnetworks whose weights are nonzero. Empirically, we compute $\sum_{\ell=1}^L \log \|W^{(\ell)}_j\|_F$ for each $j$, and count a subnetwork as ``zero'' if $\sum_{\ell=1}^L \log \|W^{(\ell)}_j\|_F \leq \max_{j'} \sum_{\ell=1}^L \log \|W^{(\ell)}_{j'}\|_F-50$ or equivalently $\prod_{\ell=1}^L  \|W^{(\ell)}_j\|_F \leq e^{-50} \max_{j'} \prod_{\ell=1}^L  \|W^{(\ell)}_{j'}\|_F$.
The degree of freedom of smoothing spline is returned by the solver in R, which is rounded to the nearest integer when plotting. 
To estimate the degree of freedom of trend filtering, for each choice of $\lambda$, we repeated the experiment for 10 times and compute the average number of nonzero knots as estimated degree of freedom.
For neural networks, we use the definition \citep{tibshirani2015degrees}:
\begin{eqal}
   2\sigma^2 \textrm{df} = \E\|\vect y'-\hat {\vect y}\|_2^2 - \E\|\vect y - \hat {\vect y}\|_2^2
   \label{eq:df}
\end{eqal}
where $\textrm{df}$ denotes the degree of freedom, $\sigma^2$ is the variance of the noise, $\vect y$ are the labels, $\hat {\vect y}$ are the predictions and $\vect y'$ are independent copy of $y$. 
We find that estimating \autoref{eq:df} directly by sampling leads to large error when the degree of freedom is small. 
Instead, we compute   
\begin{eqal}
   2\sigma^2 \hat{\textrm{df}} = \hat\E\|\vect y_0-\hat {\vect y}\|_2^2 - \hat \E\|\vect y - \hat {\vect y}\|_2^2 + \hat \E\|\vect y - \bar y_0\|_2^2 - \|\vect y_0 - \bar y_0\|_2^2
   \label{eq:dfe}
\end{eqal}
where $\hat{\textrm{df}}$ is the estimated degree of freedom, $\E$ denotes the empirical average (sample mean), $\vect y_0$ is the target function and $\bar y_0$ is the mean of the target function in its domain. 
\begin{proposition}
   The expectation of \autoref{eq:dfe} over the dataset $\gD$ equals \autoref{eq:df}.
\end{proposition}
\begin{proof}
\begin{eqal*}
   2\sigma^2 \hat{\textrm{df}} &= \E_\gD[\hat\E\|\vect y_0-\hat {\vect y}\|_2^2 - \hat \E\|\vect y - \hat {\vect y}\|_2^2 + \hat \E\|\vect y - \bar y_0\|_2^2 - \|\vect y_0 - \bar y_0\|_2^2] \\
   &= \E\|\vect y_0-\hat {\vect y}\|_2^2 - \E\|\vect y - \hat {\vect y}\|_2^2 +  \E_\gD[\hat \E[(\vect y - \vect y_0)(\vect y + \vect y_0 - 2\bar y_0)]]\\
   &= \E\|\vect y_0-\hat {\vect y}\|_2^2 - \E\|\vect y - \hat {\vect y}\|_2^2 + \E\Big[\sum_{i=1}^n \epsilon_i(2y_i + \epsilon_i - 2\bar y_0)\Big]\\
   &= \E\|\vect y_0-\hat {\vect y}\|_2^2 - \E\|\vect y - \hat {\vect y}\|_2^2 + n\sigma^2\\
   &= \E\|\vect y'-\hat {\vect y}\|_2^2 - \E\|\vect y - \hat {\vect y}\|_2^2
\end{eqal*}
where $\gD$ denotes the dataset. In the third line, we make use of the fact that $\E[\epsilon_i]=0, \E[\epsilon_i^2] = \sigma^2$, and in the last line, we make use of $\E[\epsilon'_i]=0, \E[{\epsilon'_i}^2] = \sigma^2$, and $\epsilon'_i$ are independent of $y_i$ and $y_{0, i}$
\end{proof}
One can easily check that 
% this estimation is an unbiased estimator to the definition \autoref{eq:df}, and 
a ``zero predictor'' (a predictor that always predict $\bar y_0$, and it always predicts 0 if the target function has zero mean) always has an estimated degree of freedom of 0.

In \autoref{fig:exp}(h)(i), we take the  minimum MSE over different choices of $\lambda$, and plot the average over 10 runs.
Due to optimization issue, sometimes the  neural networks are stuck at bad local minima and the empirical loss is larger than the global minimum by orders of magnitude. 
To deal with this problem, in \autoref{fig:exp}(h)(i), we manually detect these results by removing the experiments where the MSE is larger than 1.5 times the average MSE under the same setting, and remove them before computing the average.

\subsection{More experimental results}

\subsubsection{Regularization weight vs degree-of-freedom}
As we explained in the previous section, the degree of freedom is the exact information-theoretic measure of the generalization gap. A Larger degree-of-freedom implies more overfitting.

\begin{figure}[h]
    \centering
    \subcaptionbox{}{\includegraphics[width=0.4\textwidth]{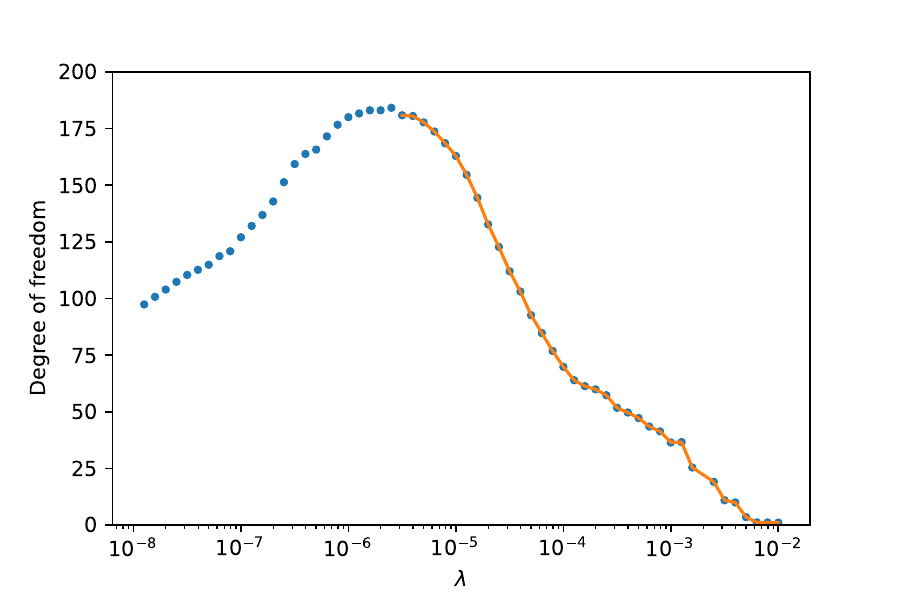}}
    \subcaptionbox{}{\includegraphics[width=0.4\textwidth]{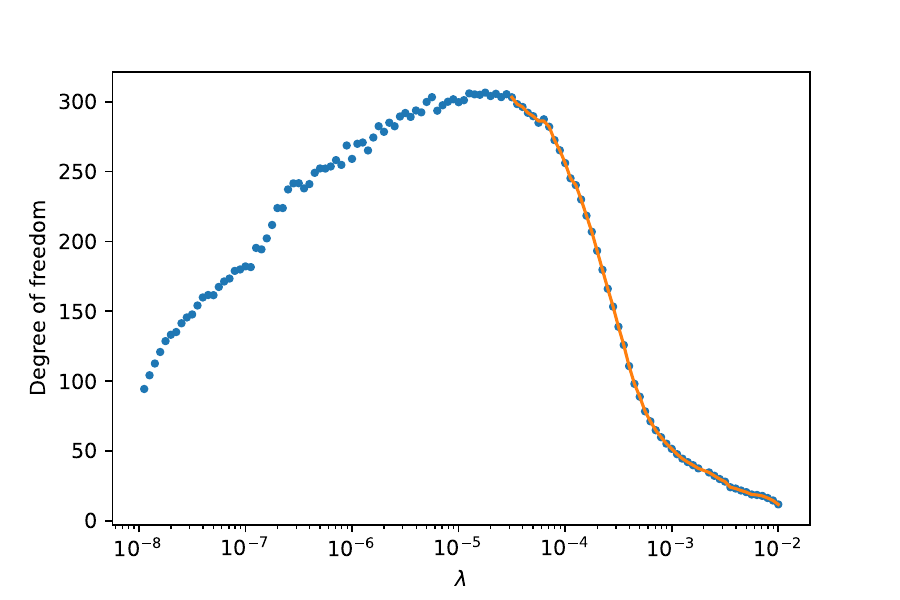}}\\
    \subcaptionbox{}{\includegraphics[width=0.4\textwidth]{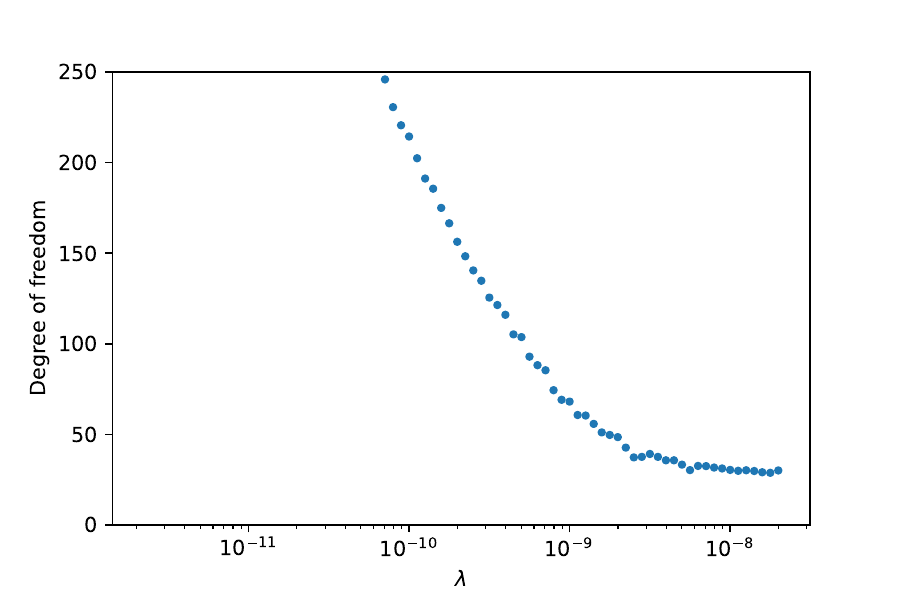}}
    \subcaptionbox{}{\includegraphics[width=0.4\textwidth]{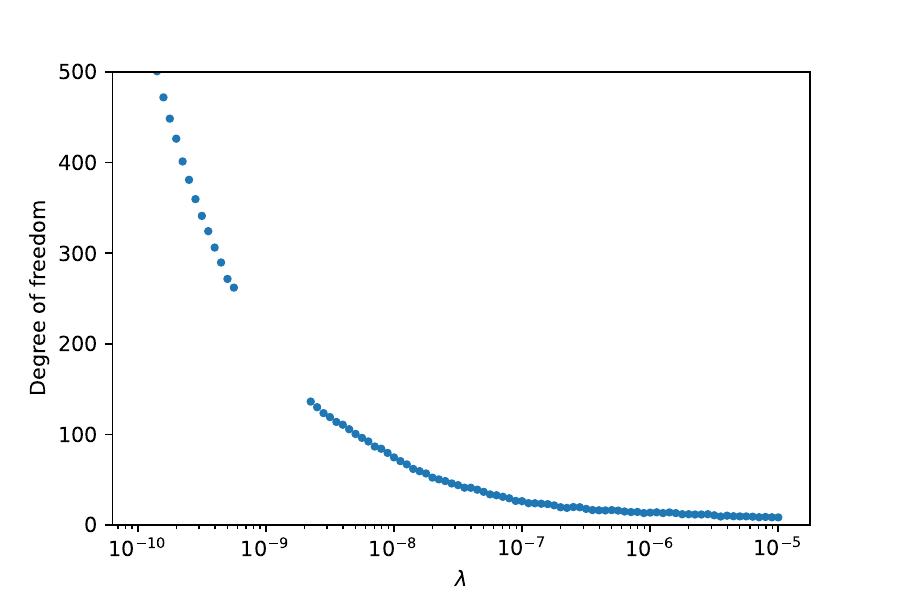}}
    \caption{The relationship between degree of freedom and the  scaling factor of the regularizer $\lambda$. The solid line shows the result after denoising. (a)(b)in a parallel NN. (c)(d) In trend filtering. (a)(c): the ``vary'' function. (b)(d) the doppler function.}
    \label{fig:df}
\end{figure}
In figure \autoref{fig:df}, we show the relationship between the estimated degree of freedom and the scaling factor of the regularizer $\lambda$ in a parallel neural network and in trend filtering.
As is shown in the figure, generally speaking as $\lambda$ decreases towards $0$, the degree of freedom should increase too.
However, for parallel neural networks, if $\lambda$ is very close to $0$, the estimated degree of freedom will not increase although the degree of freedom is much smaller than the number of parameters --- actually even smaller than the number of subnetworks.
Instead, it actually decreases a little. 
This effect has not been observed in other nonparametic regression methods, e.g. trend filtering, which overfits every noisy datapoint perfectly when $\lambda \rightarrow 0$. But for the neural networks, even if we do not regularize at all, the among of overfitting is still relatively mild $30 / 256$ vs $80/1000$. 
In our experiments using neural networks, when $\lambda$ is small, we denoise the estimated  degree of freedom using isotonic regression. 
% dof  at lambda --> 0 vs the total number of data points.   max dof should be n.

We do not know the exact reason of this curious observation. Our hypothesis is that it might be related to issues with optimization, i.e., the optimizer ends up at a local minimum that generalizes better than a global minimum; or it could  be connected to the ``double descent'' behavior of DNN \citep{nakkiran2021deep} under over-parameterization.

\subsubsection{Detailed numerical results}

% \begin{figure}[t]
%    \subcaptionbox{}{\includegraphics[trim=0 20 0 20, clip,width=0.33\textwidth]{figs/hill_nn.pdf}}
%    \subcaptionbox{}{\includegraphics[trim=0 20 0 20, clip,width=0.33\textwidth]{figs/hill_smooth.pdf}}
%    \subcaptionbox{}{\includegraphics[trim=0 20 0 20, clip,width=0.33\textwidth]{figs/hill_pnn.pdf}}\\
%    \subcaptionbox{}{\includegraphics[trim=0 20 0 20, clip,width=0.33\textwidth]{figs/hill_tf.pdf}}
%    \subcaptionbox{}{\includegraphics[trim=0 20 0 20, clip,width=0.33\textwidth]{figs/hill_wavelet.pdf}}
%    \subcaptionbox{}{\includegraphics[trim=0 20 0 20, clip,width=0.33\textwidth]{figs/hill_pnn_subnet.pdf}}
%    \caption{More experiments results of piecewise cubic function.}
%    \label{fig:extrahill}
% \end{figure}

\begin{figure}[t]
   \subcaptionbox{}{\includegraphics[trim=0 20 0 20, clip,width=0.33\textwidth]{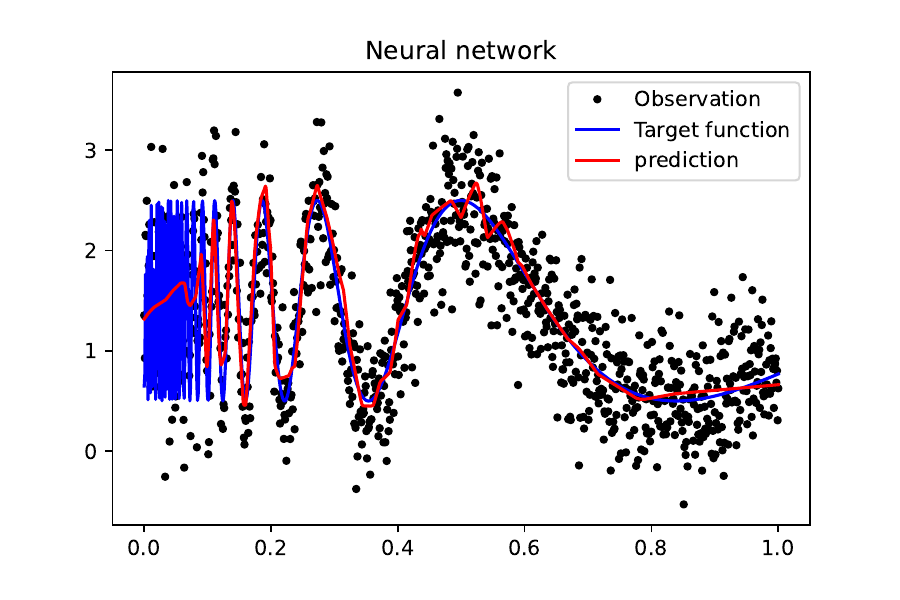}}
   \subcaptionbox{}{\includegraphics[trim=0 20 0 20, clip,width=0.33\textwidth]{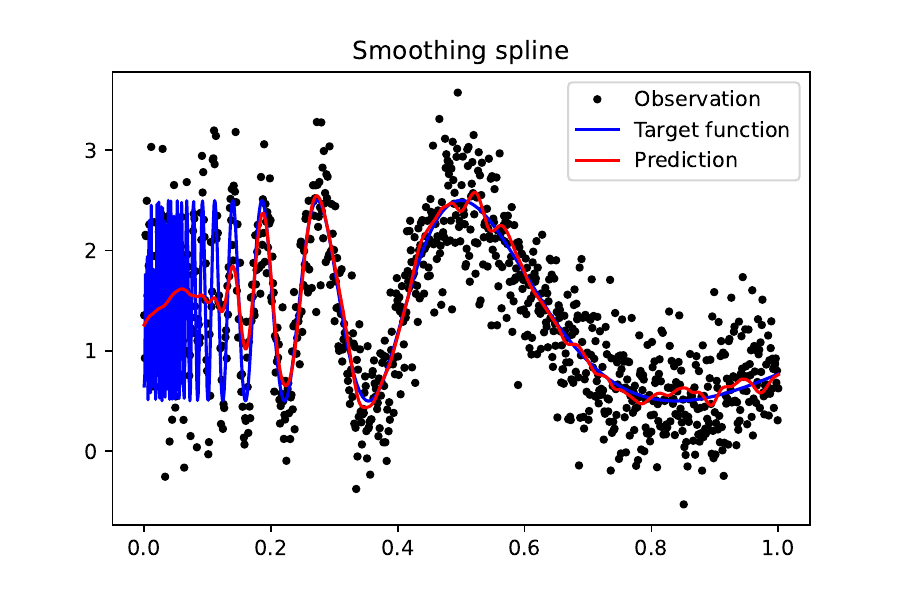}}
   \subcaptionbox{}{\includegraphics[trim=0 20 0 20, clip,width=0.33\textwidth]{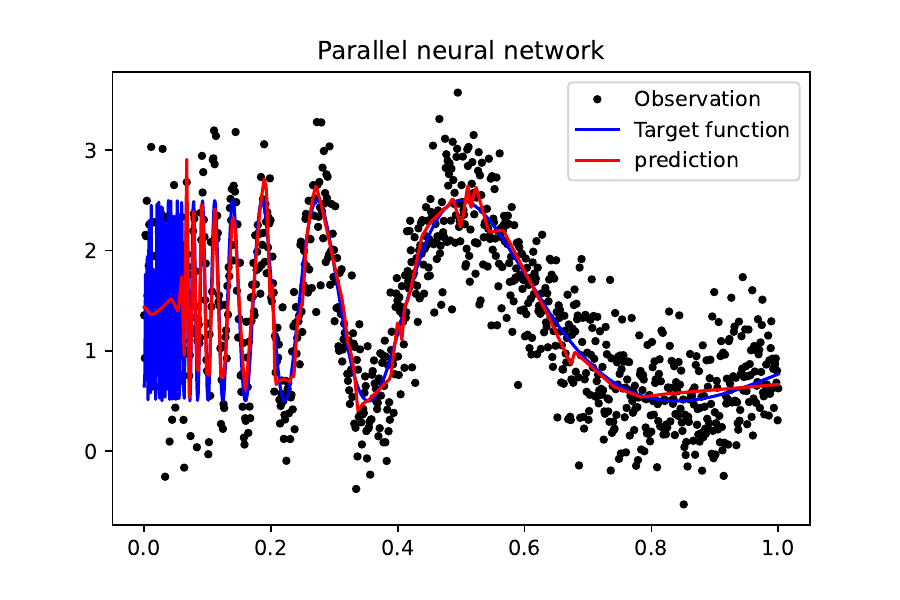}}\\
   \subcaptionbox{}{\includegraphics[trim=0 20 0 20, clip,width=0.33\textwidth]{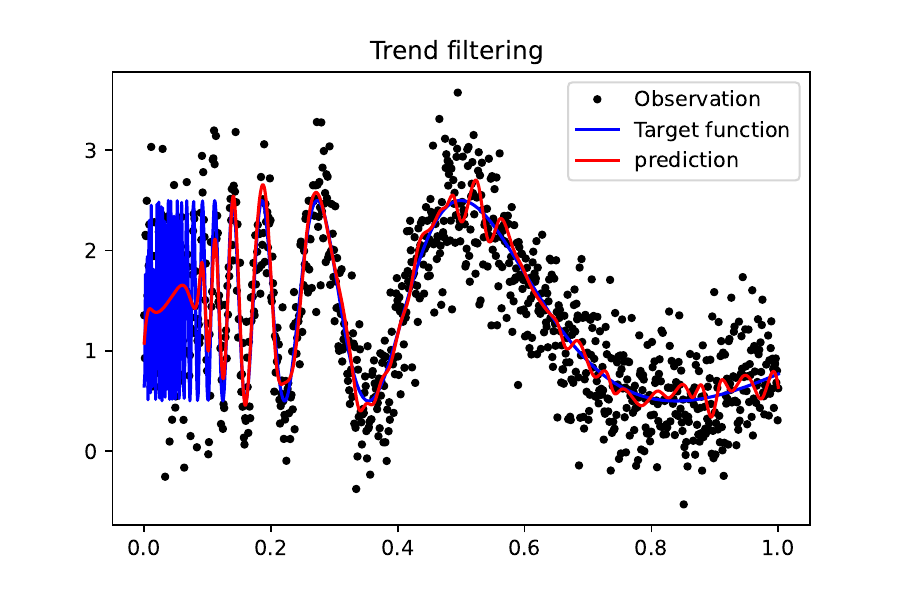}}
   \subcaptionbox{}{\includegraphics[trim=0 20 0 20, clip,width=0.33\textwidth]{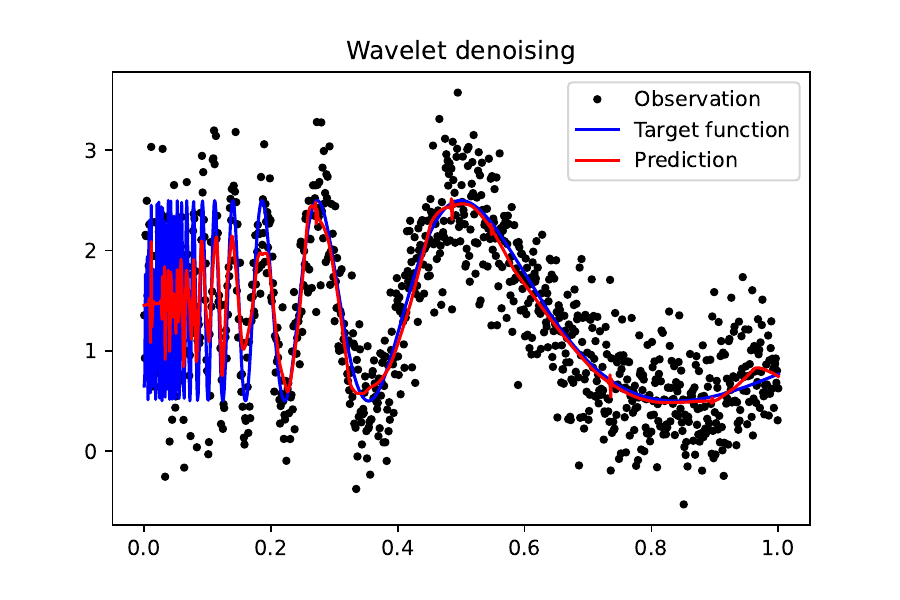}}
   \subcaptionbox{}{\includegraphics[trim=0 20 0 20, clip,width=0.33\textwidth]{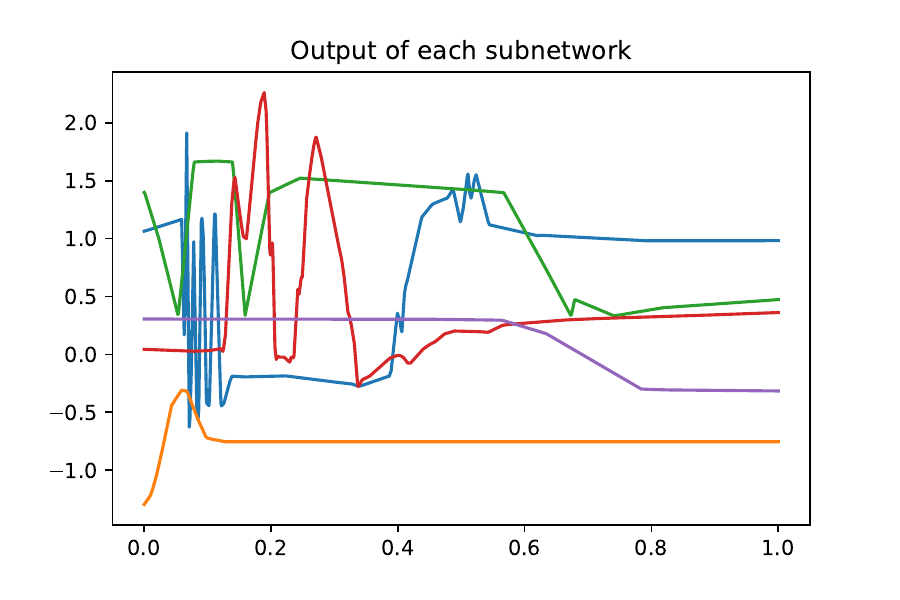}}
   \caption{More experiments results of Doppler function.}
   \label{fig:extradopler}
\end{figure}

\begin{figure}[t]
   \subcaptionbox{}{\includegraphics[trim=0 20 0 20, clip,width=0.33\textwidth]{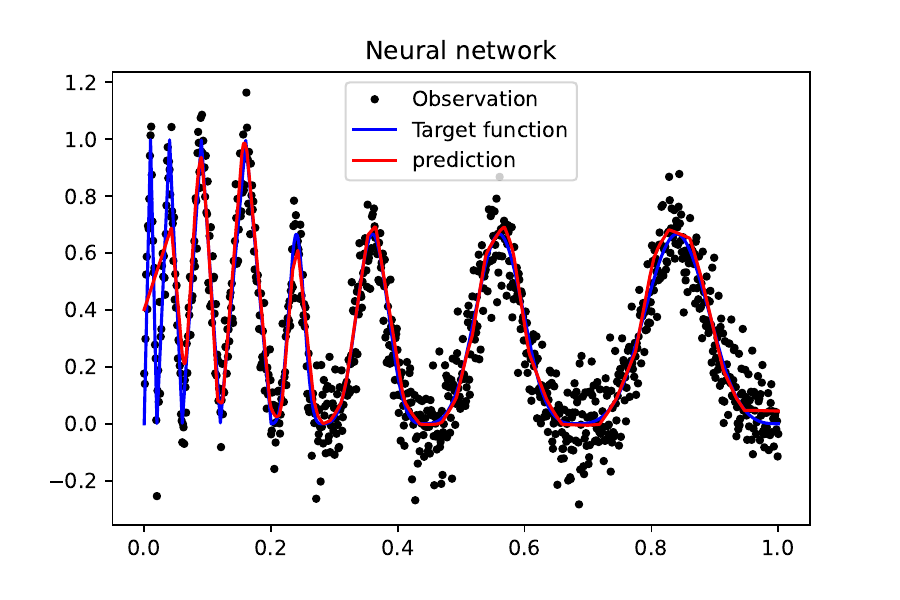}}
   \subcaptionbox{}{\includegraphics[trim=0 20 0 20, clip,width=0.33\textwidth]{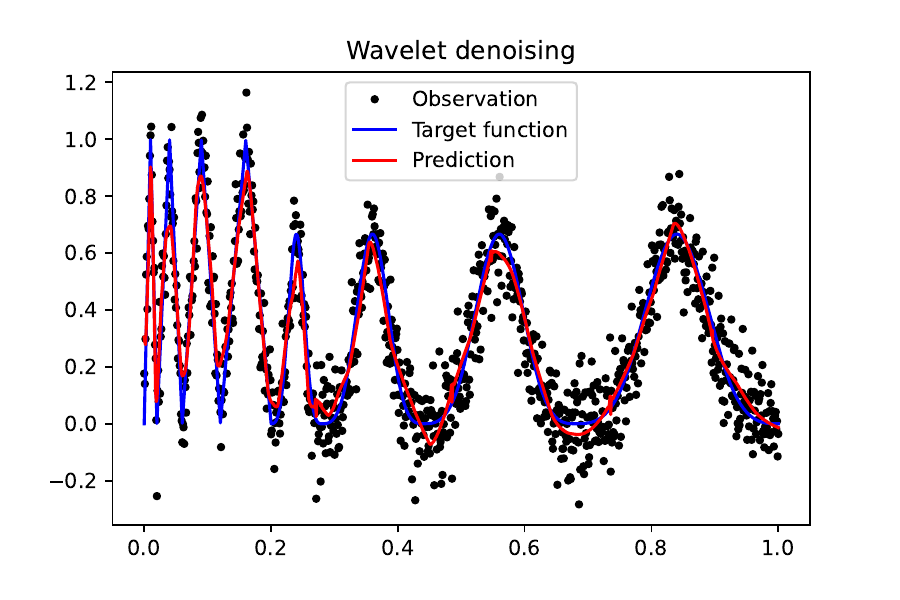}}
   \subcaptionbox{}{\includegraphics[trim=0 20 0 20, clip,width=0.33\textwidth]{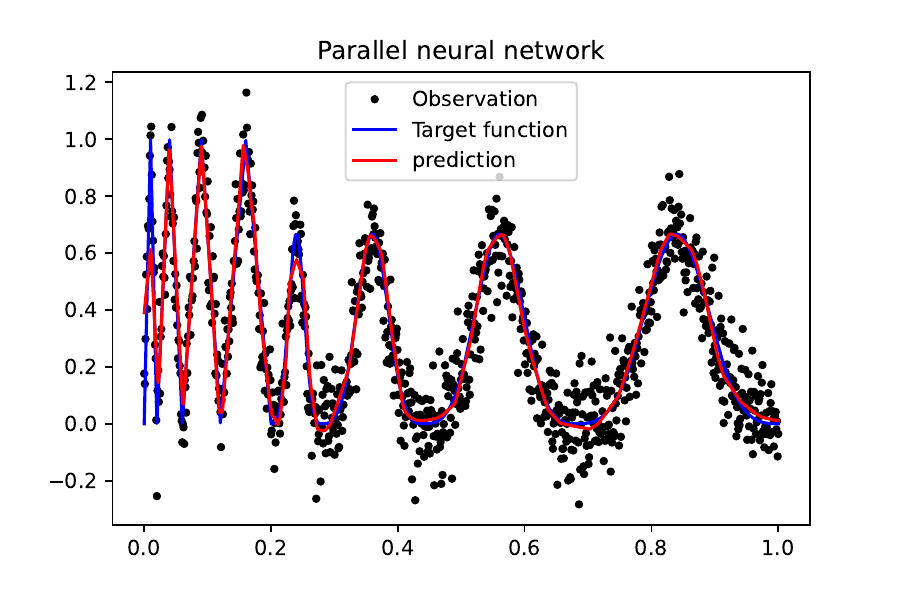}}\\
   \subcaptionbox{}{\includegraphics[trim=0 20 0 20, clip,width=0.33\textwidth]{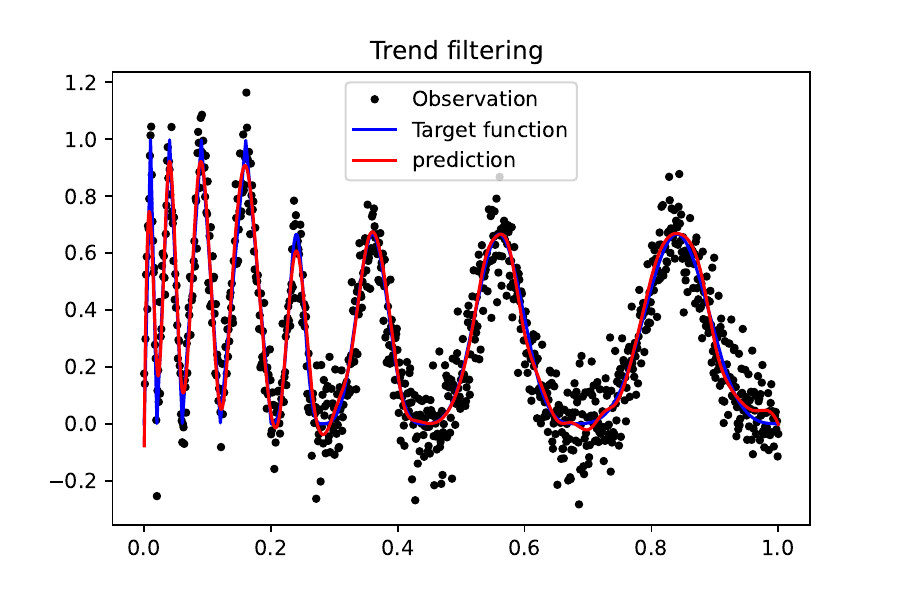}}
   \subcaptionbox{}{\includegraphics[trim=0 20 0 20, clip,width=0.33\textwidth]{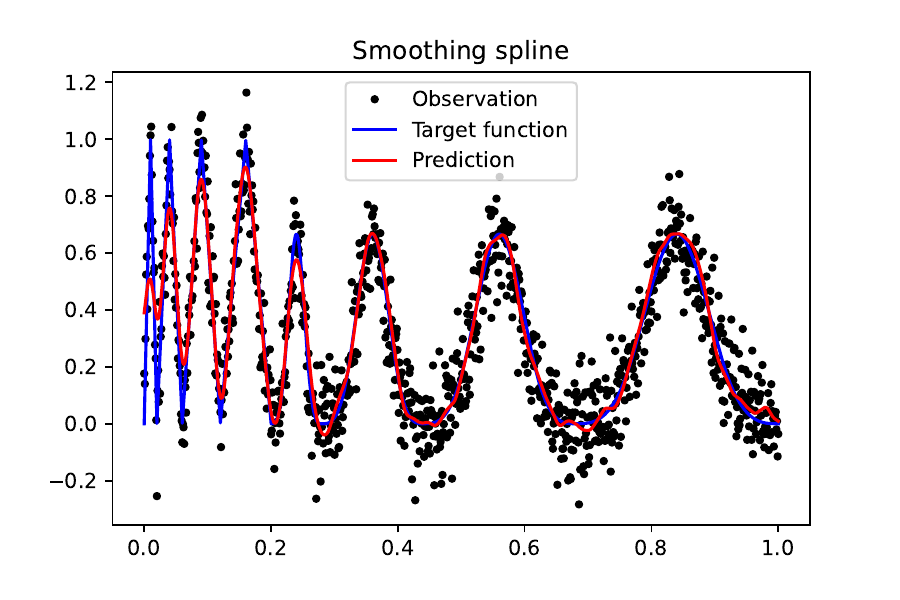}}
   \subcaptionbox{}{\includegraphics[trim=0 20 0 20, clip,width=0.33\textwidth]{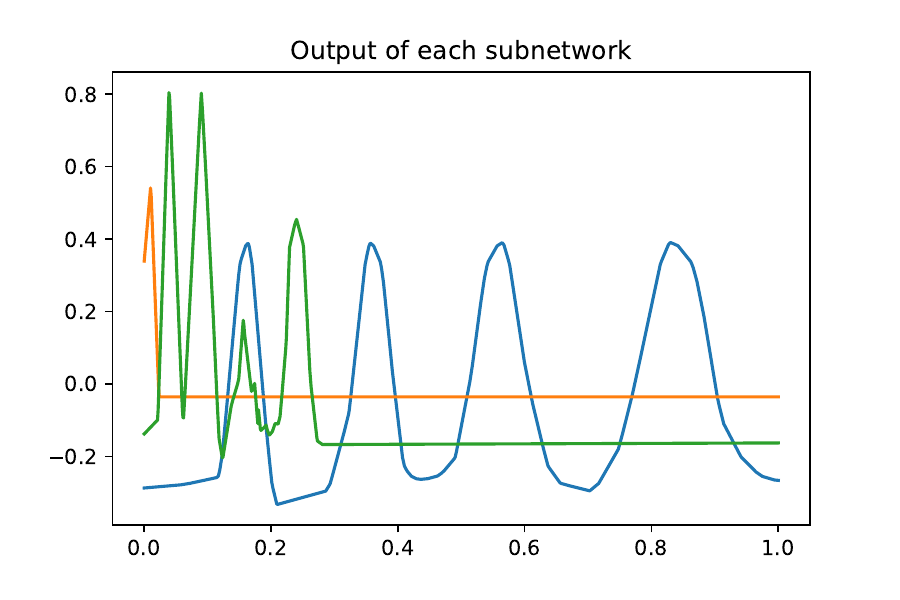}}
   \caption{More experiments results of the ``vary'' function.}
   \label{fig:extravary4}
\end{figure}

In order to allow the readers to view our result in detail,
we plot the numerical experiment results of each method separately in \autoref{fig:extradopler} and \autoref{fig:extravary4}.

\subsubsection{Practical equivalence between the weight-decayed two-layer NN and L1-Trend Filtering}

In this section we investigate the equivalence of two-layer NN and the locally adaptive regression splines from Section~\ref{sec:warmup}. In the special case when $m=1$ the special regularization reduces to weight decay and the non-standard truncated power activation becomes ReLU.  We compare L1 trend filtering \citep{kim2009ell_1} (shown to be equivalent to locally adaptive regression splines by \citet{tibshirani2014adaptive}) and an overparameterized  version of the neural network for all regularization parameter $\lambda >0$, i.e., a regularization path. The results are shown in Figure~\ref{fig:nntf}.  It is clear that as the weight decay increases, it induces sparsity in the number of knots it selects similarly to L1-Trend Filtering, and the regularization path matches up nearly perfectly even though NNs are also learning knots locations.

\begin{figure}[h]
    \centering
    \includegraphics[width=0.4\textwidth]{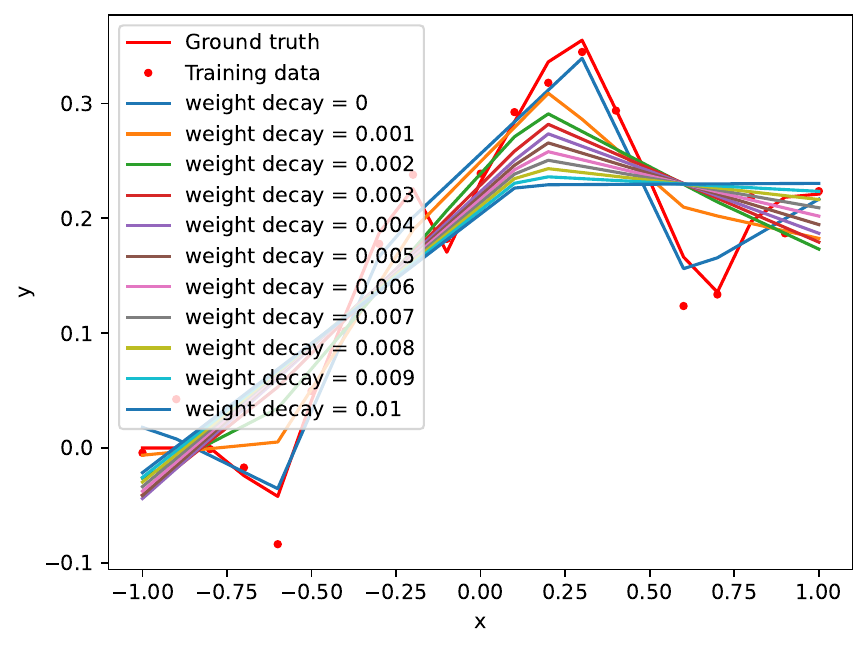}
    \includegraphics[width=0.4\textwidth]{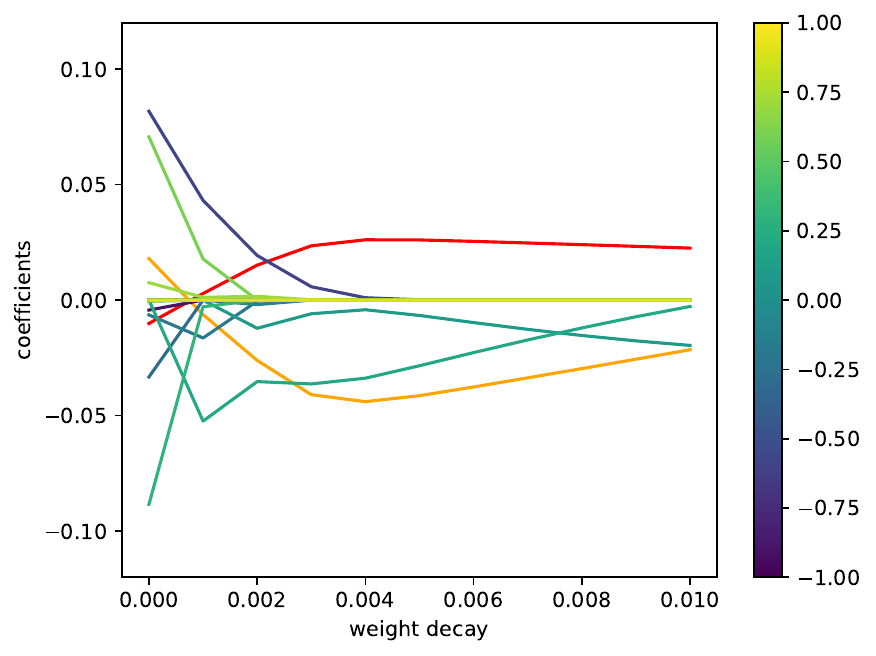}\\
    \includegraphics[width=0.4\textwidth]{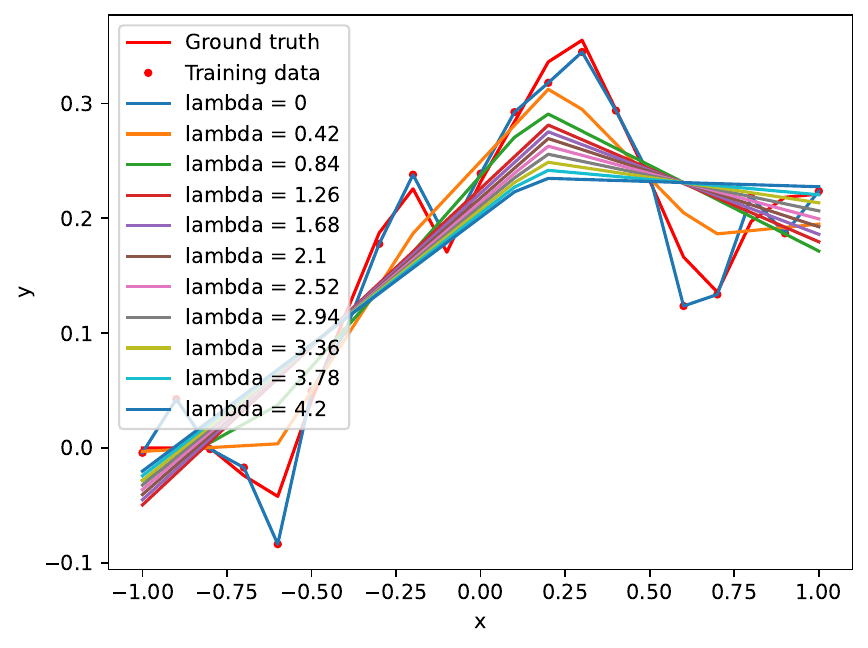}
    \includegraphics[width=0.4\textwidth]{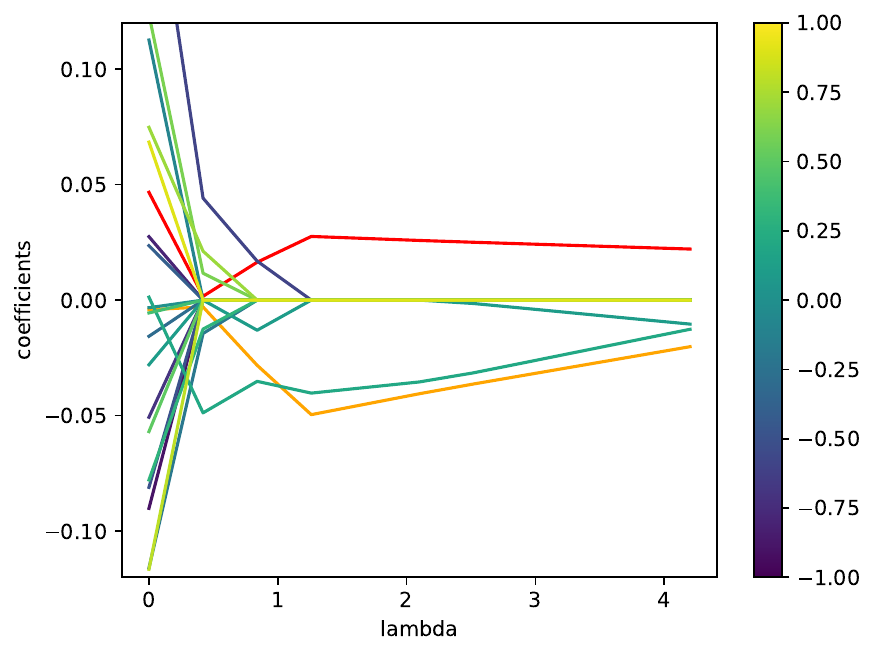}
    \vspace{-12pt}
    \caption{Comparison of the \textbf{weight decayed ReLU neural networks (Top row)} and \textbf{L1 Trend Filtering (Bottom row)} with different  regularization parameters. 
The left column shows the fitted functions and the right column shows the \emph{regularization path} (in the flavor of \cite{friedman2010regularization}) of the coefficients of the truncated power basis at individual data points (the free-knots learned by NN are snapped to the nearest input $x$ to be comparable). }
    \label{fig:nntf}
\end{figure}

\end{document}